\documentclass[twoside]{article}

\usepackage[accepted]{aistats2022}

\usepackage[round]{natbib}

\usepackage[utf8]{inputenc} \usepackage[T1]{fontenc}    \usepackage{hyperref}       \usepackage{url}            \usepackage{booktabs}       \usepackage{amsfonts}       \usepackage{nicefrac}       \usepackage{microtype}      \usepackage{xcolor}         \usepackage{cleveref}

\usepackage{booktabs,xcolor,colortbl}
\usepackage{xparse}
\ExplSyntaxOn
\NewDocumentCommand{\longdash}{ O{2} }
 {
  --\prg_replicate:nn { #1 - 1 } { \negthinspace -- }
 }
\ExplSyntaxOff

\usepackage{xcolor} \usepackage{graphicx} \usepackage{booktabs} \usepackage{array} \usepackage{mathtools} \usepackage{amsthm, amssymb, amscd, amsfonts} \usepackage[ampersand]{easylist} \usepackage{multirow} \usepackage{booktabs} \usepackage{bm} \usepackage{bbm} \usepackage{dcolumn} \usepackage{caption}
\usepackage{subcaption} \usepackage{epstopdf} \usepackage{comment} \usepackage[export]{adjustbox}
\usepackage{hyperref}
\usepackage{wrapfig}
\usepackage{thm-restate} \usepackage{xspace} \usepackage[normalem]{ulem}

\usepackage{enumitem}
\setitemize{noitemsep,topsep=0pt,parsep=0pt,partopsep=0pt}
\setenumerate{noitemsep,topsep=0pt,parsep=0pt,partopsep=0pt}
\usepackage{hyperref}
\usepackage{algorithm} \usepackage{algorithmic} %
 \newtheorem{theorem}{Theorem}

\newtheorem{proposition}[theorem]{Proposition}
\newtheorem{definition}{Definition}

\theoremstyle{definition}

\theoremstyle{definition}

\theoremstyle{definition}

\theoremstyle{definition}

\theoremstyle{definition}

\DeclareMathOperator*{\argmin}{arg\,min\,}
\DeclareMathOperator*{\argmax}{arg\,max\,}
\DeclareMathOperator*{\Id}{id}

\newcommand{\R}{\ensuremath{\mathbb{R}}}

\newcommand{\nobs}{n}

\newcommand{\nvar}{d}
\newcommand{\ndim}{\nvar}

\newcommand{\x}{x}
\newcommand{\xvec}{\bm{\x}}
\newcommand{\xmat}{X}
\newcommand{\xset}{\mathcal{\xmat}}
\newcommand{\y}{y}
\newcommand{\yvec}{\bm{\y}}

\newcommand{\z}{z}
\newcommand{\zvec}{\bm{\z}}

\renewcommand{\u}{u}
\newcommand{\uvec}{\bm{\u}}

\renewcommand{\v}{v}
\newcommand{\vvec}{\bm{\v}}

\newcommand{\KL}{\textnormal{KL}}

\newcommand{\T}{\text{T}} 
 \newcommand{\E}{\mathbb{E}}

\newcommand{\gvec}{\bm{g}}

\newcommand{\dist}{d}

\renewcommand{\dist}{d}

\newcommand{\bary}{\textnormal{bary}}

\renewcommand{\T}{T}
\newcommand{\Tw}{t}

\newcommand{\wrapproof}[1]{
\iftrue #1
\fi }

\newcommand{\weight}{w}
\newcommand{\weightvec}{\bm{\weight}}
\renewcommand{\dist}{d}

\newcommand{\dir}{k}
\newcommand{\ndir}{K}
\newcommand{\class}{m}
\newcommand{\nclass}{M}
\newcommand{\projection}{\bm{\theta}}

\newcommand{\pbary}{P_{X_{\textnormal{bary}}}}
\newcommand{\nlayer}{L}

\newcommand{\divergence}{\phi}
\newcommand{\SW}{\textnormal{SW}}
\newcommand{\maxSW}{\textnormal{max-SW}}
\newcommand{\maxkSW}{\textnormal{max-K-SW}}
\newcommand{\multiW}{\textnormal{Multi-W}}
\newcommand{\multimaxkSW}{\textnormal{Multi-max-K-SW}}

\newcommand{\pmm}[1]{ $\pm $ }

\usepackage{amssymb}\usepackage{pifont}\newcommand{\cmark}{\ding{51}}\newcommand{\xmark}{\ding{55}}

\begin{document}

\twocolumn[

\aistatstitle{Iterative Alignment Flows}

\aistatsauthor{ Zeyu Zhou \And Ziyu Gong \And Pradeep Ravikumar \And David I. Inouye }

\aistatsaddress{ Purdue University \And  Purdue University \And Carnegie Mellon University \And Purdue University } ]

\begin{abstract}
    The unsupervised task of aligning two or more distributions in a shared latent space has many applications including fair representations, batch effect mitigation, and unsupervised domain adaptation.
    Existing flow-based approaches estimate multiple flows independently, which is equivalent to learning multiple full generative models.
    Other approaches require adversarial learning, which can be computationally expensive and challenging to optimize.
    Thus, we aim to jointly align multiple distributions while avoiding adversarial learning.
    Inspired by efficient alignment algorithms from optimal transport (OT) theory for univariate distributions, we develop a simple iterative method to build deep and expressive flows.
    Our method decouples each iteration into two subproblems: 
    1) form a variational approximation of a distribution divergence and 2) minimize this variational approximation via closed-form invertible alignment maps based on known OT results.
    Our empirical results give evidence that this iterative algorithm achieves competitive distribution alignment at low computational cost while being able to naturally handle more than two distributions.
\end{abstract}

\section{INTRODUCTION}

The task of aligning two or more distributions in a shared latent space without any pairing information between data points (i.e., unsupervised) has attracted increasing interest due to its varied applications. These include fair representations \citep{zemel2013learningfair}, batch effect mitigation \citep{haghverdi2018batch}, unsupervised domain adaptation \citep{hu2018duplex}, and generative models \citep{grover2020alignflow}.
For example, \citet{zemel2013learningfair} estimate a shared latent representation of the class-conditional distributions that simultaneously obfuscates any information about protected attributes (e.g., race) while preserving all other information useful for classification.
For genetic data, \citet{haghverdi2018batch} attempt to mitigate batch effects (i.e., irrelevant shifts in the data between batches caused by non-biological factors) by estimating a shared representation among batches;
this enables the integration and analysis of multiple datasets collected at different laboratories.
\citet{hu2018duplex} perform unsupervised domain adaptation by mapping the source and target domains to a shared latent representation.

\begin{table*}[t!]
\caption{Comparison with other methods. *AlignFlow relies on adversarial learning to get good results.} \label{sample-table}
\label{tab:comparison}
\begin{center}
 \resizebox{1.5\columnwidth}{!}{
\begin{tabular}{llllll}
  & \textbf{CycleGAN}& \textbf{AlignFlow}&\textbf{LRMF}&\textbf{SINF}&\textbf{INB} \\ 
\hline \hline \\
\textbf{Distribution alignment}           & \cmark &\cmark&\cmark&\xmark&\cmark\\
\textbf{No adversarial learning}    & \xmark & \xmark* &\cmark&\cmark &\cmark\\
\textbf{Iterative learning}  & \xmark & \xmark &\xmark &\cmark&\cmark\\
\textbf{Multiple distributions}  & \xmark & \xmark &\xmark &\xmark&\cmark\\
\textbf{Shared latent space}  & \xmark & \cmark &\xmark&\xmark&\cmark\\
\hline
\end{tabular}
}
\end{center}
\end{table*}

Prior work on this unsupervised alignment task generally falls into two categories: adversarial and flow-based methods.
\citet{CycleGAN2017} propose CycleGAN for domain translation via Generative Adversarial Networks (GANs) \citep{Goodfellow2014}. Specifically, they jointly train two GANs that attempt to generate one dataset from the other dataset and  add a cycle consistency loss to encourage that translating from domain A to B and back to A will yield the original point---i.e., an approximate invertibility constraint.
\citet{grover2020alignflow} propose AlignFlow that uses normalizing flows \citep{DBLP:conf/icml/RezendeM15,DBLP:journals/corr/DinhKB14} to satisfy cycle consistency by construction and, in contrast to CycleGAN, learns a shared latent representation of the two datasets.
AlignFlow combines both adversarial learning and maximum likelihood estimation (MLE) for training. 
Both CycleGAN and AlignFlow leverage adversarial learning to achieve good results.
However, a fundamental limitation of adversarial learning is that it can be computationally expensive and challenging to optimize \citep[e.g.,][]{lucic2018gans,kurach2019the}.
To avoid adversarial learning, AlignFlow can be set to only use the MLE loss terms.
In this case, the two flow models are estimated independently and they use Gaussian distribution as their latent representation which does not preserve any shared structure (e.g., black pixels in MNIST digits).
Thus, without adversarial learning, a natural consequence is that AlignFlow must essentially estimate two full generative models rather than merely estimating the translation map---which will be simpler if the datasets share some structure.
Hence, their method is likely to require higher sample complexity and computational cost. 
Another flow-based method LRMF \citep{lrmf} directly learns the transformation between two distributions via minimizing the non-adversarial log-likelihood ratio. 
However it is limited to the alignment between two distributions and does not learn a shared latent space.

Inspired by the limitation  of existing methods, we aim at the \emph{joint} estimation of multiple flow models that map to a shared representation \emph{without} adversarial learning. 
While in general this is hard for complex datasets, simple cases can be solved very efficiently using the tools from optimal transport (OT) theory.
Specifically, for 1D distributions, it is easy to compute the invertible maps between each distribution and the barycenter distribution, which naturally serves as a shared latent space.
Thus, we propose a method we call Iterative Na\"ive Barycenter (INB),  
which instead of trying to solve a large global problem directly, iteratively solves simpler subproblems that first estimate a variational divergence and then minimize this variational divergence via OT barycenter maps by leveraging known efficient solutions.
We leverage the development of the maximum K-Sliced Wasserstein distance proposed in Sliced Iterative Normalizing Flows (SINF) \citep{DBLP:conf/icml/DaiS21}.

As we show in the experiments, our INB method can achieve competitive or better alignment performance than baselines within a much shorter time.
Moreover, INB naturally works with multiple distributions in a symmetric way which significantly reduces the  computational cost and improves the alignment performance.

For clarity, we compare INB with prior methods in \autoref{tab:comparison} and summarize our contributions as follows:
\begin{itemize}

    \item We first develop a symmetric Monge map problem and a multi-distribution divergence to enable multi-distribution alignment.We show that the symmetric Monge map problem is equivalent to finding the Monge maps to the barycenter distribution and can be solved in closed-form for 1D distributions.
    
    \item We propose an efficient iterative algorithm for unsupervised distribution alignment by iteratively minimizing the multi-distribution divergence. 
    Our algorithm involves two steps: the first step forms a variational approximation of the divergence around the current iterate and the second step exactly minimizes this variational divergence via known OT solutions for 1D.
    
    \item To the best of our knowledge, our INB approach is the first flow-based distribution alignment approach that can be naturally applied to align multiple distributions.
    \item We demonstrate the benefits of our INB approach on synthetic and real-world datasets.
\end{itemize}

\section{BACKGROUND}
\label{sec:background}

Given samples from $\nclass$ class distributions $(P_{X_1}, P_{X_2}, \cdots, P_{X_\nclass})$, our goal is to find invertible maps $T_1,T_2,\cdots,T_\nclass$  such that the resulting latent distributions are aligned in a shared latent space, i.e., $P_{T_1(X_1)} = P_{T_2(X_2)} = \cdots = P_{T_\nclass(X_\nclass)}$.
Because the maps are invertible, this also enables translation between any two component distributions merely by composing one map and the inverse of the other, i.e., to translate from $\class$ to $\class'$, the following map can be used $T_{\class\to \class'} = T_{\class'}^{-1} \circ T_\class$.
We can formalize our alignment goal as the following optimization problem:
\begin{align}
    \min_{\T_1,\cdots,\T_\nclass} \divergence(P_{T_1(X_1)}, P_{T_2(X_2)},\cdots,P_{T_\nclass(X_\nclass)}) \,,
\end{align}
where $\divergence$ is a multi-distribution statistical divergence (i.e., a functional that is always non-negative and zero if and only if all distributions are equal).
To solve this problem, we need a tractable divergence $\divergence$ and a tractable method for optimizing this problem.
We first review key concepts from optimal transport (OT) that will be needed for deriving our iterative algorithm, particularly closed-form OT solutions to 1D problems.
Then, we will review tractable two-distribution divergences, which we will extend to the $\nclass$ distribution case in later sections.

\subsection{Optimal Transport Fundamentals}

We will first review some standard OT definitions. The following classical Monge map problem \citep[Remark~2.7]{peyre2019computational} can be seen as finding the lowest transportation cost map that aligns the distributions (which is an explicit constraint in the problem).
\begin{definition}[Monge problem]\label{def:monge-problem}
Given two distributions $(P_{X_1}, P_{X_2})$ supported on two spaces $(\xset_1,\xset_2)$ and a cost function $c(\cdot, \cdot)$, the Monge problem is defined as finding the map $T \colon \xset_1 \to \xset_2$ that solves:
$    \arg\min_{T} \E_{P_{X_1}}[c(x, T(x))] \;\; \textnormal{s.t.} \quad P_{T(X_1)} = P_{X_2} \, ,$
where the objective is the \emph{transportation cost} and the constraint is a distribution alignment condition (also known as the \emph{pushforward condition}).
\end{definition}

We next review the definitions of the Kantorovich relaxation \citep[Remark~2.13]{peyre2019computational} and the barycenter distribution \citep[Remark 9.1]{peyre2019computational}, which will be important for our development of multi-distribution divergences.
For this paper, we will assume $c(x,y) = \|x-y\|_2^2$ and that one of the distributions has a density so that the barycenter is unique \citep{agueh2011barycenters}.

\begin{definition}[Kantorovich Relaxation]
\label{def:kantorovich-relaxation}
Given the same variables as \autoref{def:monge-problem}, the Kantorovich problem is defined as:
$    \mathcal{L}_c(P_{X_1}, P_{X_2}) \triangleq \min_{Q \in \mathcal{U}(P_{X_1}, P_{X_2})} \E_{Q}[c(x_1, x_2)] \, ,$
where $Q$ is a joint distribution over $\xset_1$ and $\xset_2$ such that the marginals are equal to $P_{X_1}$ and $P_{X_2}$ respectively (denoted by $\mathcal{U}(P_{X_1}, P_{X_2})$).
\end{definition}

\begin{definition}[Wasserstein Barycenter]
\label{def:wasserstein-barycenter}
Given a set of input distributions $(P_{X_1}, \cdots, P_{X_\nclass})$ defined on some space $\xset$, weights $\weightvec$ such that $\sum_\class \weight_\class = 1$, the barycenter is defined as:
$    \bary(P_{X_1}, P_{X_2}, \cdots, P_{X_\nclass}; \weightvec) \triangleq \argmin_{P_{X_\bary}} \sum_{\class=1}^\nclass \weight_\class \mathcal{L}_c(P_{X_\bary}, P_{X_\class}) \, ,$
where $\mathcal{L}_c$ is defined in \autoref{def:kantorovich-relaxation}.
\end{definition}

Finally, we review the Wasserstein-2 distance between distributions that will be the basis for the tractable sliced Wasserstein distance described next.
\begin{definition}[Wasserstein-2 Distance]
The Wasserstein-2 distance is simply $W_2(P_{X_1},P_{X_2}) = \mathcal{L}_c(P_{X_1},P_{X_2})^{\frac{1}{2}}$, where $\mathcal{L}_c$ is defined as above and $c(x,y) = \|x-y\|_2^2$.
\end{definition}

\subsection{Maximum K-sliced Wasserstein Distance}
While in  general the Wasserstein-2 distance requires solving a complex optimization problem, in 1D, the distance can be computed in closed-form because the Monge map is known in closed-form.
Thus, several works \citep[e.g.,][]{DBLP:journals/jmiv/BonneelRPP15, DBLP:conf/cvpr/KolouriZR16, DBLP:conf/cvpr/DeshpandeZS18} propose to use the sliced Wasserstein distance defined as: $\SW(P_{X_1},P_{X_2}) \triangleq \E_{\theta}[ W_2(P_{\theta^T {X_1}},P_{\theta^T {X_2}})]$ where $\theta$ is distributed as a uniform distribution over all unit norm vectors.
A variant called the maximum sliced Wasserstein distance has also been proposed $\maxSW(P_{X_1}, P_{X_2}) \triangleq \max_{\theta} W_2(P_{\theta^T {X_1}}, P_{\theta^T {X_2}})$, which is computed along the direciton with the largest $W_2$ distance \citep{DBLP:conf/nips/KolouriNSBR19}.
Recently, \citet{DBLP:conf/icml/DaiS21} proposed the maximum K-sliced Wasserstein distance (which they prove is a true metric between distributions) that finds the $K$ orthogonal directions that maximize the $W_2$ distance along each projection, i.e., $\maxkSW(P_{X_1}, P_{X_2})\triangleq \max_{\theta_1, \dots, \theta_\ndir} \sum_{\dir=1}^\ndir W_2(P_{\theta_\dir^T {X_1}}, P_{\theta_\dir^T {X_2}})$ such that $\theta_\dir^T \theta_{\dir'} = 0, \forall \dir \neq \dir'$ and $\|\theta_\dir\|_2=1$.

\section{MULTIPLE DISTRIBUTION ALIGNMENT}
To handle multi-distribution alignment, we first define a symmetric Monge map problem and show that the solution is related to the barycenter problem.
This new multi-distribution problem suggests a natural multi-distribution extension to the maximum K-sliced Wasserstein distance, which will be the divergence we seek to minimize in our iterative algorithm.

\subsection{Symmetric Monge Map Formulation}

The original Monge formulation is asymmetric because the two distributions have distinct roles.
While in theory the role of the distributions does not matter because $T^*_{\class'\to \class} \equiv (T^{*}_{\class\to \class'})^{-1}$, in practice the estimated map $\hat{T}$ will vary depending on which distribution is the source distribution.
Finally and more importantly, the Monge problem in its original formulation only considers two distributions but we want to consider more than two distributions.

\begin{definition}[Symmetric Monge Map (SMM)]
\label{def:symmetric-monge}
Given a set of continuous input distributions $(P_{X_1}, \dots, P_{X_\nclass})$ defined on some continuous space $\xset$, a non-negative weight vector $\weightvec \geq 0$ such that $\sum_\class \weight_\class = 1$, and cost function $c(\cdot, \cdot)$, the \emph{symmetric Monge map problem} is defined as:
\begin{align}
\begin{aligned}
    \argmin_{T_1, T_2, \cdots, T_\nclass}  \quad \sum_{\class=1}^\nclass \weight_\class \E_{P_{X_\class}} \left[ c(x,T_\class(x))\right]\\
    \textnormal{s.t.} \quad 
    P_{T_\class(X_\class)} = P_{T_{\class'}(X_{\class'})} \,\,\, \forall \class \neq \class'  \, .
    \end{aligned}
\end{align}
\end{definition}

When $M$=2, the original Monge problem can be recovered if $T_2 = \Id$ and $\weight_2=0$.  Thus, this problem can be seen as a symmetric relaxation of the Monge problem for two or more distributions.
We prove that our symmetric Monge map problem is equivalent to finding the maps to the barycenter (proof in appendix).
\begin{theorem}[SMM Solution is Monge Maps To Barycenter]
\label{thm:barycenter-equivalence}
For $c(x,y) = \|x-y\|_2^2$ where the distributions have densities, the symmetric Monge map solution (\autoref{def:symmetric-monge}) is the Monge maps between the class distributions and the barycenter distribution (\autoref{def:wasserstein-barycenter}), i.e., $T^*_\class = T^*_{\class \to \bary}$ where $P_{X_{\bary}} = \bary (P_{X_1}, P_{X_2}, \dots, P_{X_\nclass}; \weightvec)$.
\end{theorem}
An important special case where both the barycenter distribution and the Monge maps are known in closed-form is 1D distributions.
Thus, in combination with this theorem, we can solve the SMM problem in closed-form for 1D distributions in our iterative algorithm.

\subsection{Multiple Distribution Divergences}
Just as the Wasserstein distance can be directly derived from the optimum value of the Monge map problem, the optimal value of the SMM problem can provide a natural multi-distribution divergence.
\begin{definition}
\label{def:multi-distribution-wasserstein}
The multi-distribution Wasserstein divergence is defined as $\multiW(P_{X_1}, \cdots, P_{X_\nclass}) \triangleq \min_{T_1, T_2, \cdots, T_\nclass} \sum_{\class=1}^\nclass \weight_\class \E_{P_{X_\class}} \left[ c(x,T_\class(x))\right]$ such that $ P_{T_\class(X_\class)} = P_{T_{\class'}(X_{\class'})} \,\,\, \forall \class \neq \class'$. 
\end{definition}
Similarly, we can define the multi-distribution version of the maximum K-sliced Wasserstein, which we will use to develop our iterative algorithm in the next section.
\begin{definition}
\label{def:multi-distribution-max-k-sw}
The multi-distribution maximum K-sliced Wasserstein divergence is defined as $\multimaxkSW(P_{X_1}, \cdots, P_{X_\nclass}) \triangleq \max_{\theta_1, \dots, \theta_\ndir} \sum_{\dir=1}^\ndir \multiW_2(P_{\theta_\dir^T X_1}, \cdots, P_{\theta_\dir^T X_\nclass})$.
\end{definition}
The proof that these are divergences (i.e., that they are non-negative and have a value of 0 if and only if the distributions are equal) follows easily from the solutions to the SMM problem (see appendix for details).

\section{ITERATIVE DISTRIBUTION ALIGNMENT}\label{sec:align-algorithm}
As a reminder, our ultimate alignment goal is to solve the following problem:
\begin{align}
    \min_{\T_1,\cdots,\T_\nclass} \divergence(P_{T_1(X_1)}, P_{T_2(X_2)},\cdots,P_{T_\nclass(X_\nclass)}) \,,
    \label{eqn:divergence-minimization}
\end{align}
where $\divergence$ is a multi-distribution divergence.
In general a multi-distribution divergence is challenging to even approximate, thus we turn to the sliced Wasserstein versions which are tractable to estimate even for empirical distributions.
In particular, the multi-distribution max-K-SW can be written as a maximization problem over a variational approximation of the divergence denoted by $\tilde{\divergence}$ and parameterized by $\theta = (\theta_1, \cdots, \theta_\ndir)$, i.e., 
\begin{align}
    \divergence(P_{X_1}, \cdots,P_{X_\nclass}) &= \max_{\theta} \tilde{\divergence}(\theta, P_{X_1}, \cdots,P_{X_\nclass})
    \label{eqn:variational-divergence}
\end{align}
where
\begin{align*}
    \tilde{\divergence}(\theta, P_{X_1}, \cdots) &\triangleq \sum_{\dir=1}^\ndir \multiW_2(P_{\theta_\dir^T X_1}, \cdots, P_{\theta_\dir^T X_\nclass}) \,.
\end{align*}
Importantly, note that $\multiW_2(P_{\theta_\dir^T X_1}, \cdots, P_{\theta_\dir^T X_\nclass})$ is tractable to compute in closed-form by sorting the data projected onto each direction.
Combining \autoref{eqn:divergence-minimization} and \autoref{eqn:variational-divergence}, we arrive at the following min-max optimization for alignment:
\begin{align}
    \min_{\T_1,\cdots,\T_\nclass} \max_\theta \tilde{\divergence}(\theta, P_{T_1(X_1)}, P_{T_2(X_2)},\cdots,P_{T_\nclass(X_\nclass)}) \,.
    \label{eqn:min-max-objective}
\end{align}
While this is an adversarial problem, we will not use explicit simultaneous adversarial optimization, which can be challenging as discussed in the introduction.
Rather, we derive a simple alternating iterative approach to this problem which is made possible by the tractable structure of our divergence.
At a high level, we alternate between solving the inner maximization and the outer minimization.
The maximization step forms a variational approximation of the divergence given the current transport maps.
The minimization step adds an invertible layer that \emph{globally minimizes} this variational divergence (i.e., where $\tilde{\divergence} = 0$).
More detailed discussion of the optimization of our algorithms can be found in \Cref{app-sec:inb-discussion}.
The INB algorithm can be seen in \autoref{alg:iterative-algorithm} where for simplicity of exposition, we assume $\ndir = d$ and the more general case is discussed in \Cref{app-sec:m<d}.

For the maximization step, we perform gradient descent on the empirical versions of the multi-distribution maximum K-sliced Wasserstein divergence.
This objective can be written in closed-form as the following problem:
\begin{align}
    \argmax_{\projection: \projection^T \projection = I_\ndir}\; \!\! \sum_{\class=1}^\nclass \!\frac{\weight_\class}{\ndir} \!\sum_{\dir=1}^\ndir \!\frac{1}{\nobs_\class} \!\sum_{i=1}^{\nobs_\class} \!| (\projection_\dir^T\xvec_\class )_{[i]} - \yvec_{[i], \dir}|^2 ,
    \label{eqn:mswd-distance}
\end{align}
where $\xvec_\class \in \mathbb{R}^{d\times n_m} $ is the sample data matrix for the $\class$-th class, $\projection = [\projection_1,\ldots,\projection_\ndir]$, $(\projection_\dir^T \xvec_\class )_{[i]}$ signify the samples from the $\class$-th class distribution projected along the direction $\projection_\dir$ sorted in ascending order, $\yvec_{[i], \dir} \triangleq \sum_{\class=1}^\nclass \weight_\class (\projection_\dir^T \xvec_\class )_{[i]}$ is the empirical barycenter along direction $\projection_{\dir}$, $\ndir \leq \ndim$ is the number of directions, and $I_\ndir \in \R^{\ndir \times \ndir}$ is the identity matrix.
Intuitively, this finds the directions that reveal the largest difference between class distributions along each 1D projection.
We adopt the optimization approach in Sliced Iterative Normalizing Flows (SINF) that optimizes $\theta$ directly on the manifold of orthonormal matrices (also called a Stiefel manifold) using projected gradient descent with backtracking line search (details in \citet{DBLP:conf/icml/DaiS21}).
The algorithm $\multimaxkSW$ can be seen in \autoref{alg:max-k-sw}.

The minimization step can be solved exactly via the SMM solutions for 1D distributions (i.e., Monge maps to the barycenter, which is also known in closed-form) by estimating each of the 1D distributions and then solving for the maps.
Specifically, the solution would be $    T^*_\class = F_{\bary}^{-1} \circ F_\class \, ,$
where $F_\class$ is the CDF function of the $P_{X_\class}$ distribution and $F_{\bary}^{-1}$ is the inverse CDF of the barycenter distribution, which is known to have the following form $F_{\bary}^{-1}(u) = \sum_\class \weight_\class F_\class^{-1}(u)$.
The 1D-Barycenter algorithm can be seen in \autoref{alg:1d-bary}.
These SMM solutions also \emph{locally} minimize the transportation costs to avoid unnecessary distortion from the class distributions.
Therefore, the shared latent distributions will be less distorted than if standard generative normalizing flows were used for each distribution independently (see Experiments).
\begin{algorithm}[h!]
    \caption{1D-Barycenter}
    \label{alg:1d-bary}
\begin{algorithmic}
    \REQUIRE Samples from the $\nclass$ class distributions $ (\zvec_1, \zvec_2, \dots, \zvec_\nclass)$, weight vector $\weightvec$\ENSURE Estimated invertible alignment maps $(t_1, t_2, \cdots, t_\nclass)$
        \FOR{$\class=1,\dots,\nclass$}
        \STATE \COMMENT{Estimate the 1D CDF of $Z_\class$}
        \STATE $F_\class = \textnormal{HistogramDensityEstimation}(\zvec_m)$
        \ENDFOR
        \STATE \COMMENT{Estimate the inverse CDF of barycenter}
        \STATE $F_{\textnormal{bary}}^{-1} = \sum_\class w_\class F_\class^{-1}$
    \STATE $\forall \class, \; t_\class = F_{\textnormal{bary}}^{-1} \circ F_\class $
        \STATE \textbf{return} $(t_1, t_2, \cdots, t_\nclass)$
\end{algorithmic}
\end{algorithm}

\begin{algorithm}[h!]
    \caption{Iterative Na\"ive Barycenter Algorithm}
    \label{alg:iterative-algorithm}
\begin{algorithmic}
    \REQUIRE Samples from the $M$ class distributions $\xvec_1, \xvec_2, \ldots, \xvec_\nclass $, weight vector $\weightvec$, number of directions $K$, number of iterations/layers $L$
    \ENSURE Estimated invertible deep alignment maps $(\T_1, \T_2, \dots, \T_\nclass)$
    \STATE $\T_\class^{(0)} \gets \Id,\quad \forall \class=\{1,\dots,\nclass\}$
    \FOR{$\ell = \{1, 2, \dots, L\}$}
        \STATE  $\forall m, \zvec_\class \gets \T_\class(\xvec_\class)$
        \STATE \COMMENT{Maximization (see \Cref{app-sec:algorithm} for algorithm)} 
\STATE $\projection \gets \multimaxkSW((\zvec_1,\dots,\zvec_\nclass),\weightvec,\ndir)$
        \STATE \COMMENT{Minimization}

        \FOR{$\dir=\{1,\dots,\ndir\}$}
\STATE $\forall \class,\;\zvec'_\class =  \projection_\dir^T \zvec_\class$ \COMMENT{1D projection}
        \STATE $t_{1,\dir},\dots,t_{\nclass,\dir} = \textnormal{1D-Barycenter}(\zvec_1',\dots,\zvec_\nclass')$
        \ENDFOR
\STATE $\forall \class,  t_\class \gets  [t_{\class,1},\dots,t_{\class,\ndir}]  $
        \STATE $\forall \class, \T_\class(\xvec) \gets \projection \Tw_\class (\projection^T \T_\class(\xvec) )$ \ENDFOR 
    \STATE \textbf{return} $(\T_1, \T_2, \cdots, \T_\nclass)$
\end{algorithmic}
\end{algorithm}

\section{RELATED WORK}
\paragraph{Iterative Methods} Iterative Gaussianization  is an iterative density estimation method, that learns invertible flow-based models \citep{Chen2000, Lin2000, Lyu2009, Laparra2011, Balle2016}.
The key idea is to first learn a rotation matrix via ICA \citep{ica1984} or similar method to linearly transform the data, and then Gaussianize each marginal independently. 
\citet{inouye2018deep} extend this by iteratively building normalizing flows from more general ``shallow'' density estimation approaches.
However, these prior iterative approaches are focused on density estimation (i.e., learning a generative model), and in particular, learn a map between a \emph{known} base distribution (e.g. Gaussian) and the unknown data distribution.

Iterative approaches for aligning distributions include Projection Pursuit Monge Map \citep{meng2019ppmm} that iteratively finds interesting directions to project the data onto, and estimates Monge maps for the 1D projected data.
The caveat, however, is that it uses fixed interestingness functions such as variance to find the projection directions.
\citet{kuang2019sample} propose an alternative iterative method for learning optimal maps and the shared representation where each iteration requires the solution of a simpler but \emph{joint} optimal transport problem---rather than solving 1D OT problems as in \citet{meng2019ppmm}.
In practice, \citet{kuang2019sample} use a set of fixed interestingness functions to find the needed structure.
\citet{essid2019adaptive} extend this iterative approach by using an adversarial objective to automatically learn these interestingness functions.
\citet{DBLP:conf/icml/DaiS21} propose SINF as a generative model. 
In theory, the approach could be used to align two distributions but all the experiments in SINF focus on generative models in which one of the distributions is a Gaussian distribution.
They directly solve the optimal transport problem between the source and target distributions. 
In contrast, our approach constructs the map through a shared distribution which preserve the shared structure; thus distorting the original distributions less.
Moreover, thanks to the formulation of barycenter problems, we can naturally deal with multiple distributions, which cannot be done in SINF.
For $\nclass>2$, SINF would need to learn $\binom{\nclass}{2}$ translation maps separately while our model would jointly learn $\nclass$ maps.
As we show in the \Cref{sec:experiment}, our model achieves better alignment performance even for $\nclass=2$ experiments.

\paragraph{Adversarial Methods}
CycleGAN \citep{CycleGAN2017} minimizes the objective function:
\begin{align*}
    \begin{aligned}
    &\argmin_{G,F} \dist_{\textnormal{adv}}(P_{G(X_1)}, P_{X_2}) + \dist_{\textnormal{adv}}(P_{F(X_2)}, P_{X_1}) 
    \! \\
    &+ \!\lambda \!\Big(\E_{P_{X_1}}[\|F(G(x))\! -\! x\|_1] \!+\! \E_{P_{X_2}}[\|G(F(x)) \!-\! x\|_1]\!\Big) \,,
\end{aligned}
\end{align*}
where the distance $\dist_{\textnormal{adv}}$ approximates the distance between distributions via adversarial learning (i.e., minimax learning) and the cycle consistency terms (after $\lambda$) can be seen as a relaxation of an invertibility constraint.
StarGAN \citep{choi2018stargan} generalize CycleGAN to more than two domains.
However, these approaches cannot guarantee invertibility and require expensive and challenging adversarial learning \citep{lucic2018gans,kurach2019the}.

\paragraph{Flow Methods}
AlignFlow \citep{grover2020alignflow} extends CycleGAN by using invertible models so that the cycle consistency constraint is satisfied by construction:
\begin{align}
\begin{aligned}
    \argmin_{T_1, T_2} &\dist_{\textnormal{adv}}(P_{T^{-1}_{2} \! \circ \, T_{1}(X_1)}, P_{X_2}\!) \!+ \! \dist_{\textnormal{adv}}( P_{T^{-1}_{1} \!\circ \, T_{2}(X_2)}, P_{X_1}\!) \\ & + \lambda \Big(\KL(P_{X_1}, P_{T_{1}^{-1}(\alpha)} + \KL(P_{X_2}, P_{T_{2}^{-1}(\alpha)}\Big) \,,
     \end{aligned}
\end{align}
where the first two distance terms (equivalent to CycleGAN) are implemented using adversarial learning, $\alpha$ is a Gaussian prior distribution, and the KL terms are implemented via maximum likelihood.
Unlike our formulation, AlignFlow ignores transportation costs entirely and pushes the shared latent representation towards the assumed prior distribution $\alpha$ rather than the more natural shared latent distribution.
Also, for $\nclass>2$, AlignFlow would require $\binom{\nclass}{2}$ adversarial terms where each term adds significant complexity to training the model.

\paragraph{Wasserstein Barycenter Methods} 
Most existing methods compute Waaserstein Barycenter of discrete distributions.
For example, \citet{DBLP:wb2014} propose a fast algorithm with entropic regularization.
Those methods typically scale poorly with the number of dimensions and are not suitable for many modern machine learning problems.
More recently, several efficient and scalable methods have been proposed to estimate Wasserstein Barycenter over continuous spaces \citep{DBLP:conf/nips/LiGYS20,DBLP:wbfan,korotin}.
\citet{DBLP:wbfan} utilize input convex neural networks (ICNN) \citep{dblp:icnn} to estimate both a generator for barycenter distribution and the transportation map between mariginals and barycenter.
A typical difference between our method and those models is that even though we use the solution to 1D barycenter, our main goal is to efficiently learn \textit{invertible flow models} between marginal distributions rather than estimate the global barycenter.

\paragraph{Domain Adaptation Methods}
Domain Adaptation has become more and more popular recently and there have been several works that leverage the tools from OT for it.
\citet{DBLP:journals/pami/CourtyFTR17} propose to find the discrete optimal transportation map between source and target domains with class regularization.
JDOT \citep{DBLP:jdot} improves it by directly aligning the joint distribution of the marginals and the class conditional distributions.
DeepJDOT \citep{DBLP:deepjdot} futher improves JDOT by learning a shared space in a Convolutional Neural Network for classification.
All these methods are based on solving discrete OT problems and JDOT and DeepJDOT focus more on finding a classifier for domain adaptation instead of finding invertible alignment maps directly.

\section{EXPERIMENTS}\label{sec:experiment}

We explore our iterative alignment method both qualitatively and quantitatively using both 2D simulated data and ``permuted'' MNIST \citep{lecun-mnisthandwrittendigit-2010}---permuted means that our methods do not leverage the image structure of MNIST but merely treat each image as 784-dimensional vector.\footnote{AlignFlow and the FID we use for evaluation do use the image structure but our iterative methods do not.}  Additional experiments, implementation details and results can be found in the appendix, including experiments on FashionMNIST \citep{xiao2017/online}.

\paragraph{Metrics}
We use standard distribution distances to compare the alignment performance across methods.
We first note that the alignment condition can equivalently be written as $P_{X_m} =P_{T_{\class}^{-1} \circ T_{\class'}(X_{\class'})},\quad \forall \class \neq \class'$.
Thus, for every class distribution $P_{X_m}$, we can sample $\nclass-1$ ``fake'' distributions using our invertible transformations $\hat{P}_{X_{\class' \to \class}}=P_{\hat{T}_{\class}^{-1} \circ \hat{T}_{\class'}(X_{\class'})}$.
We merely average the empirical Wasserstein distance between all pairs of real samples and ``fake'' samples, i.e., $\text{WD} = \frac{1}{\nclass^2 - \nclass} \sum_{\class\neq \class'} \hat{W}(P_{X_m}, \hat{P}_{X_{\class' \to \class}})$, where $\hat{W}$ is the Wasserstein distance estimated using samples via the Sinkhorn algorithm \citep{NIPS2013_af21d0c9} with $\epsilon=10^{-4}$ and maximum iterations set to 100.
For higher dimensional data (e.g., MNIST), the Wasserstein distance between samples could be a poor estimator of the true Wasserstein distance \citep{genevay2019sample}.
Thus, we also compute the Frechet Inception Distance score (FID) \citep{DBLP:conf/nips/HeuselRUNH17} for a more fair evaluation, and we similarly compute the average between every pair of real and fake samples, i.e., $\text{FID} = \frac{1}{\nclass^2 - \nclass} \sum_{\class\neq \class'} \hat{\text{FID}}(P_{X_m}, \hat{P}_{X_{\class' \to \class}})$.
We also compute transportation cost to highlight that our algorithm distorts the distributions less and finds a shared latent distribution that is closer to the original distributions because we use the SMM solution for the minimization subproblem, which can be seen to locally minimize the transportation cost.
We estimate the transportation cost by an average over the test set, i.e., $\text{TC}=\sum_{\class=1}^\nclass \frac{\weight_\class}{n_\class} \sum_{x\in X_\class} \|x- \hat{T}_\class(x)\|^2$, where $X_\class$ is the test dataset for the $\class$-th class and lower is better.
We compute the mean and standard deviation over 5 runs of each method.
We also track approximate wall-clock training time for MNIST (all models are trained on a CPU except AlignFlow which is trained on a single GTX 1080 Ti and SINF-Align which is trained on Tesla P100).
More details are in the appendix.
\paragraph{Baseline Methods}

Because prior iterative method focuses on generative models rather than distribution alignment, we adapt prior generative methods to produce alignment approaches.
First, we adapt the iterative density destructors method (DD) \citep{inouye2018deep} by learning \emph{independent} normalizing flows from each class distribution to a fixed uniform distribution, which is the same for all class distributions and thus serves as a fixed shared latent space.
We also adapt SINF \citep{DBLP:conf/icml/DaiS21} to the alignment task (SINF-Align) where we directly find the map between two distributions without any shared representation.
While the SINF paper mentioned that SINF could be used to align any two distributions, the experiments in the paper assumed that one of the distributions was a standard normal distribution---i.e., only generative experiments were performed.
Given that SINF is not symmetric (a point emphasized in the SINF paper), we train two SINFs: one from distribution $X_0$ to $X_1$ and the other in the reverse direction.
We notice a significant difference of the performance of the forward and inverse of SINF maps.  Specifically, the forward map performs well but the inverse map performs poorly (detailed results given in the appendix).
These results suggest that the direction of learning is critical and that a symmetric formulation is more stable.
For MNIST, as a non-iterative baseline, we compare to the invertible AlignFlow \citep{grover2020alignflow}, which explicitly maps both distributions to an assumed prior distribution.

\paragraph{Our Methods}
For our methods, except for INB, as a comparison, we also report the results with the single-layer independent (na\"ive) barycenter (NB) (assume all features are independent of each other and learn alignment maps directly without any projection) and multi-layer random rotation followed by NB (Rand-INB).
Number of layers and other parameters are in the appendix.

\textit{Computational complexity of INB.}
The complexity of the maximization is $\mathcal{O}(J(\nobs \nclass \ndir(d+\log \nobs)+\ndir^2d+\ndir^3))$, where $J, \nclass, \nobs, d, \ndir$ are the number of iterations ($J_{\textnormal{max}}$ in \autoref{alg:max-k-sw}), classes, samples per class, dimensions, and latent dimensions, respectively. The terms come from projecting down to $\ndir$ dimensions, computing SWD via sorting, and updating the projection matrix.
The complexity of the inner minimization is $\mathcal{O}(\nobs \nclass \ndir)$ since each latent dimension can be computed independently and primarily estimates histograms, which have piecewise linear CDFs and inverse CDFs.

\begin{table}[h!]
    \begin{center}
    \caption{
    Transportation cost (TC), sample-based Wasserstein distance (WD, lower is better) for 2D data.
    More 2D datasets in appendix.
    }
        \label{tab:2d-k2-barycenter-flow}
     \resizebox{\columnwidth}{!}{
    \begin{tabular}{l|ll} 
    \hline
    \centering
\textbf{Model}   & \textbf{WD} & \textbf{TC} \\
    \hline
    NB &0.0788 $\pm$ 0.0000 &\textbf{0.4013 $\pm$ 0.000}
     \\ 
    Rand-NB & 0.0047 $\pm$ 0.0011 & 0.4903 $\pm$ 0.0205 \\ 
    INB &\textbf{0.0025 $\pm$ 0.0005} & 0.4832 $\pm$ 0.0282 \\ 
    \hline \hline
    DD &0.0085 $\pm$ 0.0000 &1.2564 $\pm$ 0.0000\\
    SINF-Align(0$\Rightarrow$1)& \textbf{0.0024 $\pm$ 0.0002}&  \longdash\\
    SINF-Align(1$\Rightarrow$0)& \textbf{0.0026 $\pm$ 0.0003}& \longdash\\
    \hline
    \end{tabular}
    }

    \end{center}
\end{table}

\begin{table*}[h!]
    \begin{center}
    \caption{
    Transportation cost (TC), sample-based Wasserstein distance (WD, lower is better), FID (lower is better) and time for MNIST($\nclass=2$). For a fair comparison, the $\ndir$ used for INB ($\nlayer=20$) is adjusted to be the same as SINF which is $56$.
    }
    \label{tab:mnist-k2}
     \resizebox{1.5\columnwidth}{!}{
    \begin{tabular}{l|llll} 
    \hline
\textbf{Model} &   \textbf{WD} & \textbf{FID} & \textbf{TC} & \textbf{Time(s)}\\
    \hline
    NB & 60.010 $\pm$ 0.000 & 229.551 $\pm$ 0.000   &\textbf{28.115 $\pm$ 0.000} & \textbf{25} \\ 
INB ($\nlayer=20$) &23.481 $\pm$ 0.161 & 43.196 $\pm$ 0.588 & 31.671 $\pm$ 0.056& 430 \\
    INB ($\nlayer=250$) & \textbf{23.183 $\pm$ 0.095} & \textbf{37.480 $\pm$ 0.008} & 32.841 $\pm$ 0.097 & 2200\\ 
    \hline \hline
    DD & 39.079 $\pm$ 0.000 & 166.320 $\pm$ 0.000 &235.164 $\pm$ 0.000 & 360 \\
SINF-Align($0\Rightarrow 1$)& 50.151 $\pm$ 0.950 & 247.142 $\pm$ 0.972 &  \longdash & 50\\
    SINF-Align($1\Rightarrow 0$) & 42.658 $\pm$ 1.253 & 202.058$\pm$ 1.716 &   \longdash& 50\\
    AlignFlow($\lambda=$1e-4)  &  56.386  & 158.654 & 392.578 & 220000 \\
    AlignFlow($\lambda=$1e-5)  &  60.452 & 191.983 & 412.531 & 220000 \\    
    \hline
    \end{tabular}}
    \end{center}
\end{table*}

\begin{table*}[h!]
    \begin{center}
    \caption{
    Transportation cost (TC), sample-based Wasserstein distance (WD, lower is better), FID (lower is better) and time for MNIST($\nclass=10$).
    }
    \label{tab:mnist-k10}
    \resizebox{1.5\columnwidth}{!}{
    \begin{tabular}{l|llll} 
    \hline
\textbf{Model} & \textbf{WD} &  \textbf{FID}& \textbf{TC} & \textbf{Time(s)}\\
    \hline
    NB & 65.674 $\pm$ 0.000 & 190.920 $\pm$ 0.000 & \textbf{25.907 $\pm$ 0.000} & \textbf{90} \\ 
    INB & \textbf{41.044 $\pm$ 0.076} &\textbf{86.264 $\pm$ 0.550} &28.934 $\pm$ 0.140 &5000 \\ 
    \hline \hline
    DD &53.587 $\pm$ 0.000 &187.475 $\pm$ 0.000 & 227.171 $\pm$ 0.000&1700\\
    \hline
    \end{tabular}}
    \end{center}
\end{table*}
\begin{figure*}[!t]
\newcommand{\smallfigwidth}{0.12\textwidth}
     \centering
     \begin{subfigure}[b]{0.16\textwidth}
\includegraphics[width=\textwidth]{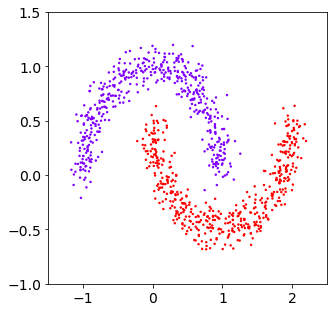}
         \caption{Original Data}
         \label{fig:2d-ori}
     \end{subfigure}
     \begin{subfigure}[t]{\smallfigwidth}
         \centering
         \includegraphics[width=\textwidth]{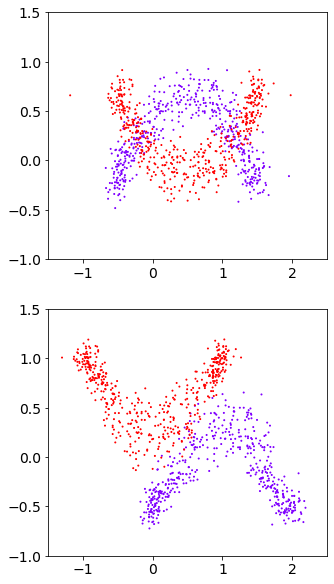}
         \caption{NB}
     \end{subfigure}
     \begin{subfigure}[t]{\smallfigwidth}
         \centering
         \includegraphics[width=\textwidth]{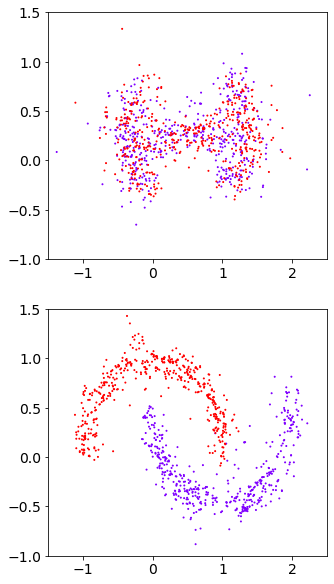}
         \caption{INB}
     \end{subfigure}
\begin{subfigure}[t]{\smallfigwidth}
         \centering
         \includegraphics[width=\textwidth]{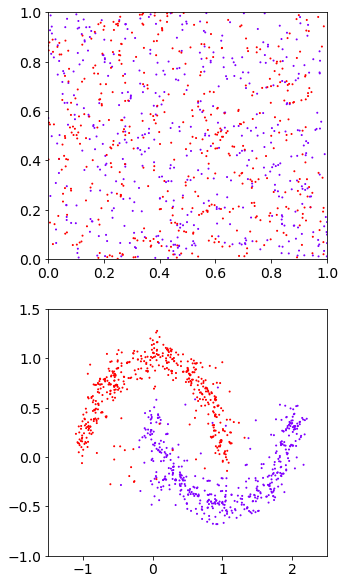}
         \caption{DD}
     \end{subfigure}
          \begin{subfigure}[t]{0.3\textwidth}
         \includegraphics[width=\textwidth]{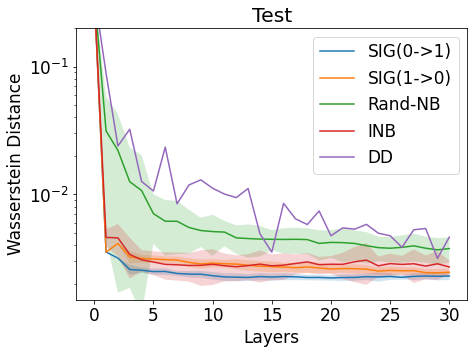}
         \caption{WD over 30 Layers}
        
      \label{fig:2d-convergence}
     \end{subfigure}
        \caption{The purple and red moons in \autoref{fig:2d-ori} represent distributions $P_{X_1}$ and $P_{X_2}$. The goal is to flip them (i.e. find $P_{X_1^{'}} =P_{T^*_{2\to1}(X_2)}$ and $P_{X_2^{'}} =P_{T^*_{1\to2}(X_1)}$). The shared representations (top row) for each method show that our iterative methods (INB) find low transportation cost shared latent spaces whereas DD ignores transportation cost and merely projects both distributions to the uniform distribution.
        The bottom row shows test samples that were flipped to the other class distribution (ideally these ``fake'' samples would look like the original data). \autoref{fig:2d-convergence} shows that INB converge faster than Rand-NB because we optimize for the directions.
        SIG represents SINF-Align with SIG setup (details in the appendix).
}
        \label{fig:2d-k2-qualitative}
\end{figure*}

\begin{figure*}[!t]
     \centering
         \begin{subfigure}[t]{0.24\textwidth}
         \centering
         \includegraphics[width=\textwidth]{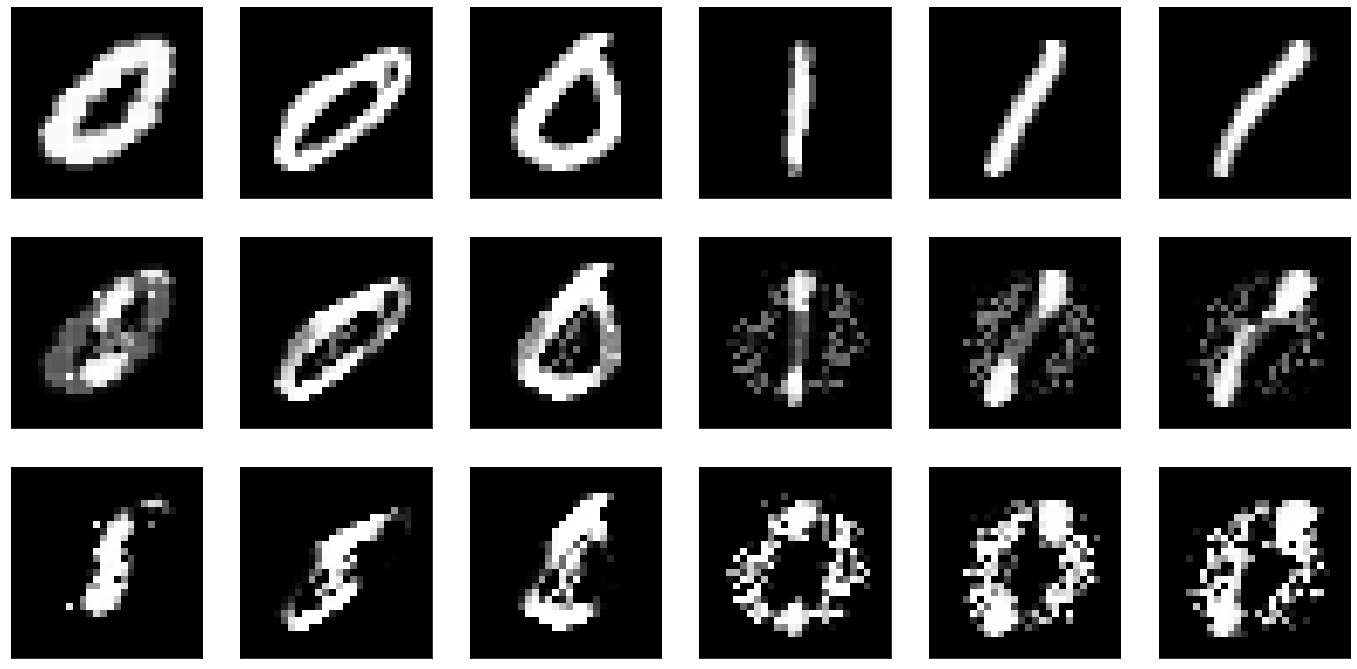}
         \caption{NB}
     \end{subfigure}
     \begin{subfigure}[t]{0.24\textwidth}
         \centering
         \includegraphics[width=\textwidth]{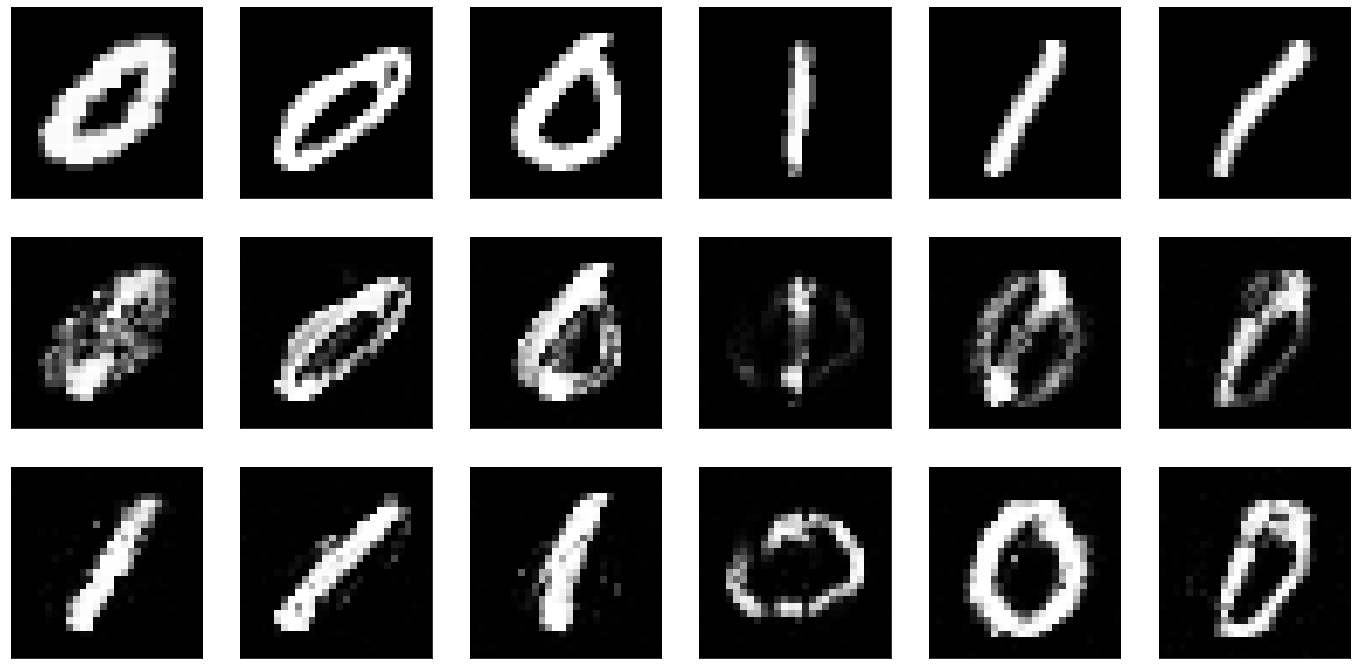}
         \caption{INB}
     \end{subfigure}
     \begin{subfigure}[t]{0.24\textwidth}
         \centering
         \includegraphics[width=\textwidth]{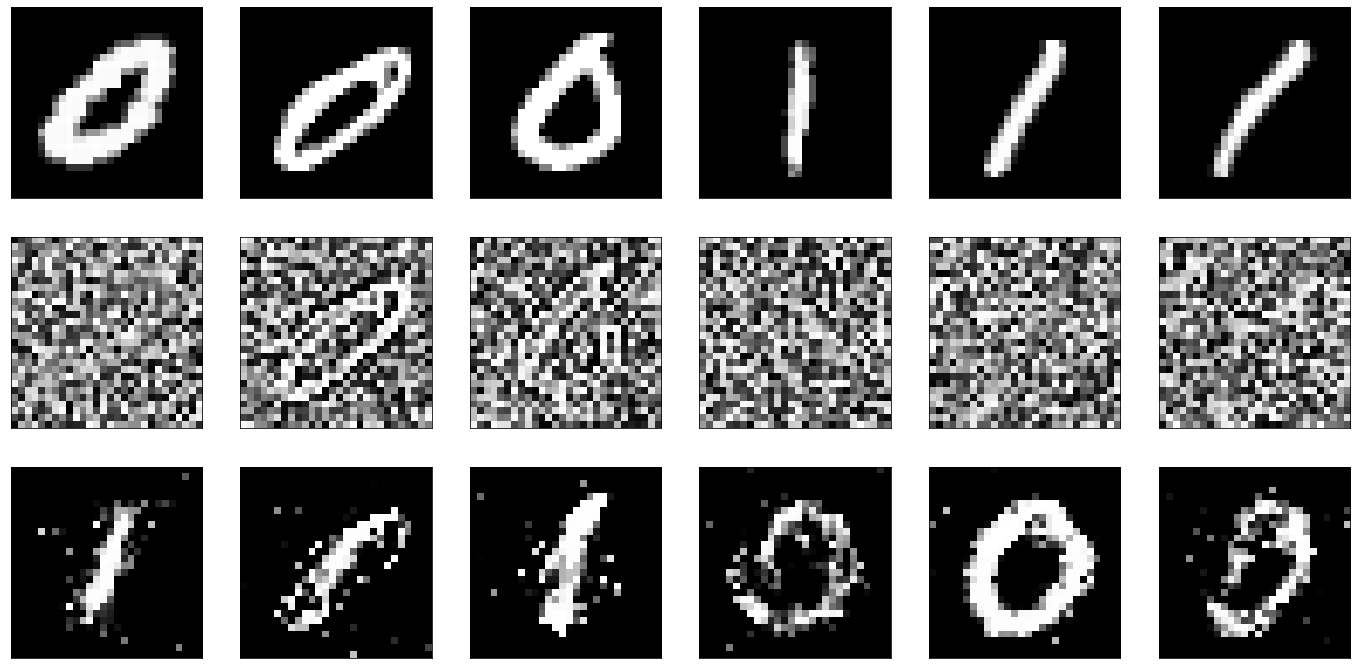}
         \caption{DD}
     \end{subfigure}
        \begin{subfigure}[t]{0.24\textwidth}
         \centering
         \includegraphics[width=\textwidth]{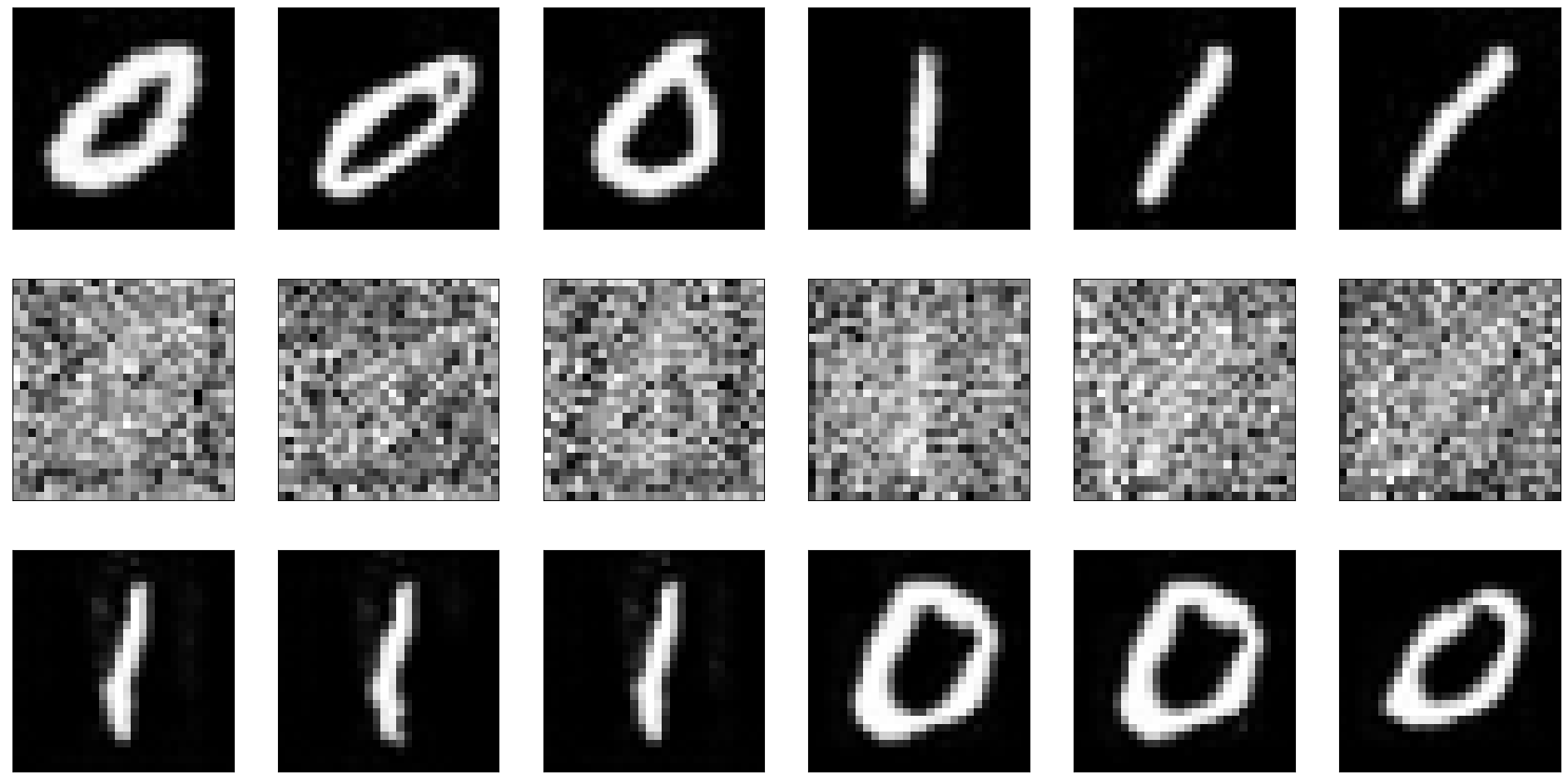}
         \caption{AlignFlow}
     \end{subfigure}
        \caption{These examples demonstrate that our methods find a more natural shared latent representation that preserves structural similarities (e.g., black pixels) between the two digits while DD and AlignFlow do not. The rows from top to bottom are the original MNIST digits, their shared latent representation, and their projection to the space of the other digit (i.e., flipped).}
        \label{fig:mnist-k2-samples}
        \vspace{1em}
\end{figure*}

\paragraph{2D Experiments} 
The qualitative results in \autoref{fig:2d-k2-qualitative} illustrate that our method (INB) finds shared latent space where the transportation cost is low (i.e., where the map distorts the original distributions less), whereas density destructors (DD) ignores transportation costs and projects both distributions to the uniform distribution.
The results for the 2D datasets with $\nclass=2$ in \autoref{tab:2d-k2-barycenter-flow} demonstrate that our iterative flows perform comparable or better than the baseline methods (DD, SINF-Align) in terms of alignment, which is measured by the empirical Wasserstein-2 distance on test data (WD), while having significantly lower transportation cost (TC) on test data.
\autoref{fig:2d-convergence} shows that INB converges much faster than Rand-NB.
Additional experiments and results for $\nclass>2$ are in appendix. 
\begin{table*}[t!]
\centering
\caption{
   Multi-distribution ($\nclass=10$) results for MNIST with INB.
        The labels of the rows represent the class of real samples and the labels of the columns represent the class of flipped samples e.g. the number in the row "2" and column "4" represents the WD between the real "4" samples and the fake "4" flipped from "2" samples. 
    } 
    \label{tab:mnist-k10-matrix}
     \resizebox{1.35\columnwidth}{!}{
\begin{tabular}{|r|r|r|r|r|r|r|r|r|r|r|r|}
\hline & \textbf{0} & \textbf{1} & \textbf{2} & \textbf{3} & \textbf{4} & \textbf{5} & \textbf{6} & \textbf{7} & \textbf{8} & \textbf{9} \\
\hline \textbf{0} & 0.10 & 13.37 & 48.57 & 42.77 & 35.85 & 41.80 & 36.13 & 30.86 & 45.42 & 32.10 \\
\hline \textbf{1} & 39.86 & 0.49 & 47.85 & 41.98 & 34.91 & 41.60 & 34.78 & 29.59 & 44.08 & 31.01 \\
\hline \textbf{2} & 41.02 & 13.86 & 0.05 & 43.88 & 36.75 & 43.93 & 37.43 & 31.24 & 46.24 & 33.47 \\
\hline \textbf{3} & 40.80 & 13.55 & 48.66 & 0.04 & 36.94 & 43.38 & 36.44 & 31.32 & 45.81 & 33.22 \\
\hline \textbf{4} & 40.63 & 14.12 & 48.89 & 43.24 & 0.09 & 43.37 & 37.22 & 31.68 & 46.12 & 32.89 \\
\hline \textbf{5} & 40.61 & 13.46 & 49.03 & 42.95 & 36.39 & 0.08 & 36.20 & 31.42 & 45.66 & 32.88 \\
\hline \textbf{6} & 40.31 & 13.84 & 48.87 & 42.56 & 36.26 & 41.57 & 0.11 & 30.19 & 44.88 & 32.18 \\
\hline \textbf{7} & 40.08 & 13.44 & 48.40 & 42.42 & 36.58 & 42.43 & 35.57 & 0.14 & 45.69 & 32.91 \\
\hline \textbf{8} & 40.59 & 13.67 & 49.48 & 44.32 & 37.94 & 43.41 & 36.47 & 32.01 & 0.08 & 34.12 \\
\hline \textbf{9} & 40.16 & 13.88 & 48.44 & 43.36 & 35.54 & 42.42 & 36.26 & 30.97 & 45.51 & 0.14\\
\hline
\end{tabular}
}
\end{table*}
\paragraph{``Permuted'' MNIST}
Qualitative samples from the latent space and after flipping between the two digits (\autoref{fig:mnist-k2-samples}) highlight that our methods retain shared latent structure such as the black pixels, whereas the generative baselines (DD, AlignFlow) move the shared latent distribution to the assumed prior (uniform or Gaussian, respectively) so that shared structure is also removed.
Quantitative results in \autoref{tab:mnist-k2} demonstrate that  INB has superior performance in terms of both WD and FID. 
Regarding SINF, because the original paper does not test their model on alignment task, we attempt to use their best model to be fair.
We report the result of the best SINF-Align models where the number of layers is chosen based on the best test WD.
Note that SINF-Align usually achieves the best WD after a few layers and that is why the time we report is quite short.
The results demonstrate that our methods perform well in terms of the alignment condition (measured by WD and FID where lower is better) than the iterative baseline (DD, SINF-Align) and end-to-end baseline AlignFlow.
Also, the computational cost is much lower for the iterative methods ($<1$ hour on CPU when $\nclass=2$), whereas AlignFlow trained for 200 epochs on a GPU took approximately 60 hours (thus, we only estimate one model and cannot compute standard deviations for AlignFlow).

While we use $\nclass=2$ to fairly compare to prior methods, our method focuses on multi-distribution alignment for $\nclass>2$.
Therefore, we present quantitative results for $\nclass=10$ in \autoref{tab:mnist-k10} and \autoref{tab:mnist-k10-matrix} (results for $\nclass=3$ in the appendix).
Because no prior methods consider the multi-class case, we only show DD as a baseline method which learns $\nclass$ independent flows to the uniform.
This multi-class situation (i.e. $\nclass>2$) is much more difficult for AlignFlow (which did not implement $\nclass>2$) and would na\"ively require $(\nclass^2 - \nclass)/2$ pairwise adversarial loss terms.
SINF does not provide any natural way to handle the $\nclass>2$ case as well.
Qualitative examples of transforming between every digit and every other digit (i.e., $\nclass=10$) for MNIST are shown in \autoref{fig:mnist-10-samples}.
Notice that even for this multiclass case, almost all transformed digits are recognizable.
We observe that the distributions are not fully aligned in the shared space (e.g., some digit structure remains).
On one hand, this could explain why there are artifacts in the flipped samples in some cases (e.g., "6").
On the other hand, we hypothesize that this indicates the smoothness of our method and explains why it has better alignment compared to SINF-Align.
We leave further investigation to future works.

\begin{figure}[h!]\centering
\includegraphics[width=0.99\linewidth]{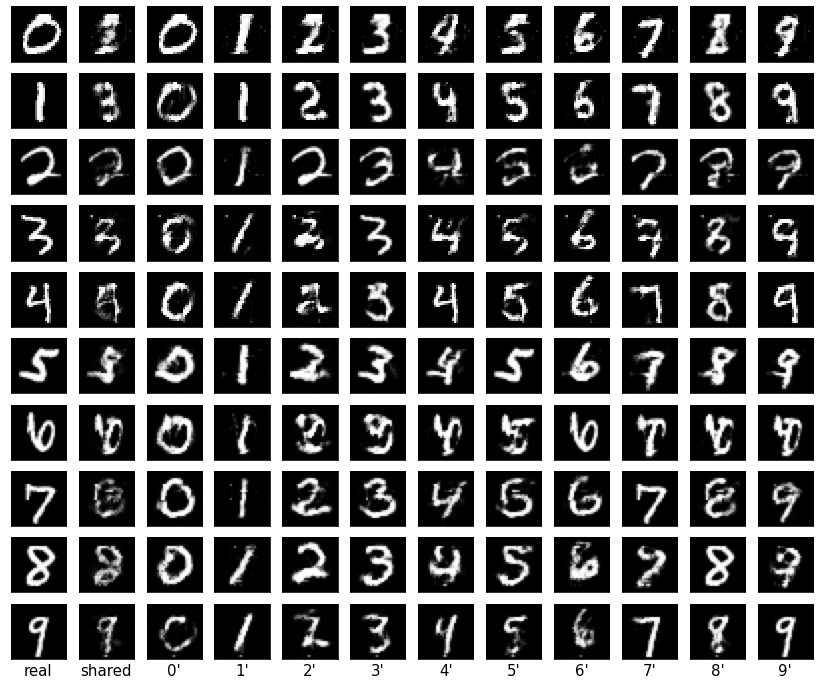}
        \caption{Multi-distribution ($\nclass=10$) results for MNIST with INB.
        The first column shows the real samples and the second column shows their shared latent representations. The following columns show the mappings of the real samples to the distribution of the other digits e.g. all flipped samples in the first row are flipped from the real 0 in the first column.}
        \label{fig:mnist-10-samples}
\end{figure}

\section{DISCUSSION AND CONCLUSION}\label{discussion}

We seek to iteratively align multiple distributions without adversarial learning.
We leverage insights from OT theory to construct an iterative estimation algorithm that alternates between estimation of a tractable divergence via maximization and exact minimization of this variational divergence.
Unlike prior approaches, our formulation does not require a fixed latent distribution and can be symmetrically applied to any number of distributions.
Unlike prior approaches based on deep normalizing flows, our approach is significantly faster.
Despite many advantages of our approach, however, there are also many open challenges.
For example, our current algorithm is greedy.
Though its greedy nature makes it easy to implement, we cannot guarantee that it finds the globally optimal alignment solution.
We leave exploring the non-greedy algorithm to future work.
We believe our work is a first step towards non-adversarial distribution alignment that can open up novel perspectives on distribution alignment.

\subsubsection*{Acknowledgements}
D.I., Z.Z., and Z.G. acknowledge support from the Army Research Lab through contract number W911NF-2020-221. P. R. acknowledges the support of NSF via IIS-1909816.
{
\bibliographystyle{plainnat}
\bibliography{references,Mendeley-fixed}
}

\clearpage
\appendix

\thispagestyle{empty}

\onecolumn \makesupplementtitle

\section{OVERVIEW}
We have organized our appendix as follows:
\begin{itemize}
    \item \Cref{app-sec:proofs} includes the proofs (and key OT results needed for the proofs).
        \item \Cref{app-sec:algorithm} includes the algorithm of multi-distribution max-K-SW and the discussion of the optimization of our algorithm.
    \item \Cref{app-sec:tree-sliced} describes an alternative to max-K-SW using tree-sliced Wasserstein divergence instead that could be used within our algorithmic general framework.
    \item \Cref{app-sec:failure-sinf} describes our investigation on directly using SINF for alignment task.
    \item \Cref{app-sec:additional-experiments} describes additional experiments including additional FashionMNIST experiments and includes quantitative result tables for these experiments (qualitative figures are included in the final appendix section).
    \item \Cref{app-sec:experimental-details} provides more details on our experimental setup including dataset preparation, models, and metric details.
    \item \Cref{app-sec:additional-and-expanded-figures} provides both expanded figures from the main paper and new result figures for the additional experiments.
\end{itemize}

\section{PROOFS}
\label{app-sec:proofs}

\subsection{Symmetric Monge Map Solution Proofs}

\begin{proof}[Proof of \autoref{thm:barycenter-equivalence}]
First, let us denote $\pbary \triangleq P_{T^*_\class(X_\class)}$ for any $\class$ since they are all equal because of the pushforward condition (at this point we do not assume anything about $\pbary$).
We can prove that $T_\class^*$ is the optimal Monge map (which is unique for quadratic cost) from $P_{X_\class}$ to $\pbary$ for all $\class$, i.e., $T_\class^* = T^*_{\class\to\bary}$, via contradiction.
Suppose $T^* \neq T^*_{\class\to\bary}$, then $T^*$ could be replaced by the optimal Monge map and the minimum value could be reduced---which is a contradiction to the optimality of $T^*$.
Given this fact and Brenier's theorem \citep[Theorem 2.1]{peyre2019computational} on the equivalence between the Kantorovich and the Monge map problems, we can now transform our original objective at the optimum $T_\class^*$ to the Kantorovich barycenter objective from \autoref{def:wasserstein-barycenter} at its optimum:
\begin{align}
    & \quad \sum_{\class=1}^\nclass \weight_\class \E_{P_{X_\class}} \left[ c(x,T^*_\class(x))\right]\\
& =\sum_{\class=1}^\nclass \weight_\class \E_{Q^*_\class}[c(x_m, x_\bary)] \\
    &=\sum_{\class=1}^\nclass \weight_\class \min_{Q_\class \in \mathcal{U}(P_{X_\class}, \pbary)} \E_{Q_\class}[c(x_m, x_\bary)] \\
     &=\sum_{\class=1}^\nclass \weight_\class \mathcal{L}_c(P_{X_\class}, P_{X_\bary}) \, ,
\end{align}
where the first equality is by Brenier's theorem and $Q_m^*$ is the optimal Kantorovich joint distribution, the second equality is by the definition of the Kantorovich problem, and the third equality is by the definition of $\mathcal{L}_c$.
Thus, our objective can be equivalently written as optimizing over $\pbary$ for the objective above, which is exactly the definition of a barycenter in \autoref{def:wasserstein-barycenter}.
Thus, $\pbary = \bary(\mu_1, \mu_2, \cdots \mu_\nclass; \weightvec)$.
\end{proof}

\begin{proposition}[{Univariate Barycenter \citep[Remark 9.6]{peyre2019computational}}]
\label{thm:univariate-barycenter}
Given a weight vector $\weightvec$ with cost $c(x,y) = \|x-y\|^2$, the inverse CDF of the barycenter is the weighted average inverse CDF of the class distributions, i.e.,
\begin{align}
    \forall u \in [0,1], \quad F^{-1}_{\bary}(u) = \sum_{\class=1}^\nclass \weight_\class F^{-1}_\class(u) \, ,
\end{align}
where $F^{-1}_\class$ is the inverse CDF of the $\class$-th class distribution.
\end{proposition}
\begin{proposition}[{Univariate Optimal Transport Map \citep[Remark 2.30]{peyre2019computational}}]
\label{thm:univariate-map}
The optimal map between univariate distributions $P_{X_1}$ and $P_{X_2}$ is the composition of the CDF of $P_{X_1}$ with the inverse CDF of $P_{X_2}$, i.e.,
\begin{align}
    T_{1 \to 2}^* = F^{-1}_{2} \circ F_{1} \, .
\end{align}
\end{proposition}

\begin{theorem}[Optimal 1D Symmetric Monge Maps]
\label{thm:univariate-symmetric-monge}
The optimal univariate symmetric Monge maps are:
$    T^*_\class = F_{\bary}^{-1} \circ F_\class \, ,$
where $F_\class$ is the CDF function of the $P_{X_\class}$ distribution and $F_{\bary}^{-1}$ is the inverse CDF of the barycenter distribution, which is known to have the following form $F_{\bary}^{-1}(u) = \sum_\class \weight_\class F_\class^{-1}(u)$.
\end{theorem}

\begin{proof}From \autoref{thm:barycenter-equivalence}, we know that the solution to the symmetric Monge problem is the Monge map between the class distribution and the barycenter distribution.
From \autoref{thm:univariate-barycenter}, we can form the univariate barycenter distribution given the class distributions.
We can then combine this result with \autoref{thm:univariate-map} to solve for the optimal map between the univariate class distribution and the univariate barycenter distribution. 
\end{proof}

\subsection{Divergence Proofs}

\begin{proposition}
\label{thm:multi-w-is-divergence}
$\multiW(P_{X_1}, \cdots, P_{X_\nclass}) \triangleq \min_{T_1, T_2, \cdots, T_\nclass} \sum_{\class=1}^\nclass \weight_\class \E_{P_{X_\class}} \left[ c(x,T_\class(x))\right]$ such that $ P_{T_\class(X_\class)} = P_{T_{\class'}(X_{\class'})} \,\,\, \forall \class \neq \class'$ as defined in \autoref{def:multi-distribution-wasserstein} is a divergence.
\end{proposition}
\begin{proof}
We need to prove two properties: 1) $\multiW(P_{X_1}, \cdots, P_{X_\nclass}) \geq 0$, and 2) $\multiW(P_{X_1}, \cdots, P_{X_\nclass})=0$ if and only if $P_{X_\nclass} = P_{X_{\nclass'}}, \forall m \neq m'$.
The first property is obvious by inspection of the objective function which is always non-negative.

If $P_{X_\nclass} = P_{X_{\nclass'}}, \forall m \neq m'$, then we can use the trivial solution of all maps being the identity, i.e., $\forall \class, T_\class(\xvec)=\xvec$.  By construction, the constraint is satisfied and the cost will be 0, which is the global optimum of the minimization.

If $\multiW(P_{X_1}, \cdots, P_{X_\nclass})=0$, then we know that $\forall \xvec$ and $\forall \class$, $c(\xvec, T_\class(\xvec)) = 0$ (by contradiction if one of them was $>0$ then it would violate the assumption that the sum was 0). The only function that satisfies this property would be the identity functions for all $T_\class$.
By the constraint of the optimization, we know that $P_{T_\class(X_\class)} = P_{T_{\class'}(X_{\class'})} \,\,\, \forall \class \neq \class'$ and thus since these must be the identity, then we know that $P_{X_\class} = P_{X_{\class'}}, \forall \class \neq \class'$.
\end{proof}

\begin{proposition}
$\multimaxkSW(P_{X_1}, \cdots, P_{X_\nclass}) \triangleq \max_{\theta_1, \dots, \theta_\ndir} \sum_{\dir=1}^\ndir \multiW_2(P_{\theta_\dir^T X_1}, \cdots, P_{\theta_\dir^T X_\nclass})$ as defined in \autoref{def:multi-distribution-max-k-sw} is a divergence.
\end{proposition}
\begin{proof}
The non-negativity property follows directly from the fact that $\multiW_2$ is a divergence which is non-negative.  We now prove that 
$\multimaxkSW(P_{X_1}, \cdots, P_{X_\nclass}) = 0$ if and only if $P_{X_1}=P_{X_2}=\cdots=P_{X_\nclass}$.

If $\multimaxkSW(P_{X_1}, \cdots, P_{X_\nclass}) = 0$, then we can prove that $\forall \theta \in \{\theta \in \R^\ndim: \|\theta\|_2=1\}, \multiW_2(P_{\theta^T X_1}, \cdots, P_{\theta^T X_\nclass})=0$. (The proof for this statement is by contradiction.  Suppose $\exists \theta$ such that
$\multiW_2(P_{\theta^T X_1}, \cdots, P_{\theta^T X_\nclass})>0$.  Then, we could set $\theta_1 = \theta$ in the maximization problem and $\multimaxkSW(P_{X_1}, \cdots, P_{X_\nclass}) > 0$.  Yet this is a contradiction to our assumption that $\multimaxkSW(P_{X_1}, \cdots, P_{X_\nclass}) = 0$.)
Thus, by \autoref{thm:multi-w-is-divergence}, we know that $\forall \theta \in \{\theta \in \R^\ndim: \|\theta\|_2=1\}, \forall \class \neq \class', P_{\theta^T X_\class} = P_{\theta^T X_{\class'}}$. From this we can conclude that $\forall \class \neq \class', P_{X_\class} = P_{X_{\class'}}$ because two joint distributions are equal if and only if the marginals along every direction are equal.

If $P_{X_1}=P_{X_2}=\cdots=P_{X_\nclass}$, then we know that the marginals along all directions must be equal, i.e., $\forall \theta \in \{\theta \in \R^\ndim: \|\theta\|_2=1\}, \forall \class \neq \class', P_{\theta^T X_\class} = P_{\theta^T X_{\class'}}$.  Thus, $\forall \theta, \multiW_2(P_{\theta^T X_1}, \cdots, P_{\theta^T X_\nclass})=0$ and the maximal value of $\max_{\theta_1,\cdots,\theta_\ndir} \sum_{\dir=1}^\ndir \multiW_2(P_{\theta_\dir^T X_1}, \cdots, P_{\theta_\dir^T X_\nclass})$ must also be 0 for any $\theta_1,\cdots,\theta_\ndir$.  Thus, $\multimaxkSW(P_{X_1}, \cdots, P_{X_\nclass}) = 0$.
\end{proof}

\section{ALGORITHMS}\label{app-sec:algorithm}

\subsection{Multi-distribution Maximum K-Sliced Wasserstein Distance}
\begin{algorithm}[h!]
    \newcommand{\dvec}{\bm{d}}
\caption{\multimaxkSW} 
    \label{alg:max-k-sw}
\begin{algorithmic}
    \REQUIRE Samples from the $\nclass$ class distributions $ (\xvec_1, \xvec_2, \ldots, \xvec_\nclass)$, weight vector $\weightvec$, number of directions $\ndir$, max number of iterations $J_{\max}$ 
    \ENSURE Estimated projection matrix $\projection$
    \STATE Randomly initialize $\projection\in \mathbb{R}^{d\times \ndir}$ satisfying $\projection^T \projection = I_\ndir$, 
    $ \projection =  [\projection_1,\dots,\projection_\ndir] $ 
    \FOR{$ j = \{1, 2, \dots, J_{\max}\}$}
        \STATE $\dvec =  \sum_{\class=1}^\nclass \!\frac{\weight_\class}{\ndir} \!\sum_{\dir=1}^\ndir \!\frac{1}{\nobs_\class} \!\sum_{i=1}^{\nobs_\class} \!| (\projection_\dir^T \xvec_\class )_{[i]} - \yvec_{[i], \dir}|^2$
        \STATE $\gvec=[-\frac{\partial \dvec}{\partial \projection_{i,j}}]$, $\uvec = [\gvec,\projection]$, $\vvec=[\projection,-\gvec]$
        \STATE $\projection = \projection - \tau \uvec(I_{2\ndir} + \frac{\tau}{2} \vvec^T \uvec)^{-1} \vvec^T \projection$,   learning rate $\tau$ determined by backtracking line search 
        \IF{$\projection$ converge}
        \STATE Stop
        \ENDIF
    \ENDFOR 
    \STATE \textbf{return} $\projection$
\end{algorithmic}
\end{algorithm}

\subsection{Discussion of INB}\label{app-sec:inb-discussion}

Expanding \autoref{eqn:min-max-objective} and simplifying to a single slice, i.e., $\theta \in \mathbb{R}^d$
\begin{align}
    &\min_{\T_1,\cdots,\T_\nclass} \max_\theta \tilde{\divergence}(\theta, P_{T_1(X_1)}, P_{T_2(X_2)},\cdots,P_{T_\nclass(X_\nclass)}) \nonumber \\
    =&\min_{\T_1,\cdots,\T_\nclass} \max_\theta \multiW(P_{\theta^T T_1(X_1)}, P_{\theta^T T_2(X_2)},\cdots,P_{\theta^T T_\nclass(X_\nclass)}) \nonumber \\
    =&\min_{\T_1,\cdots,\T_\nclass} \max_\theta \left(\begin{aligned}
        \min_{f_1, \cdots, f_\nclass} &\sum_{\class=1}^\nclass \weight_\class \E_{z\sim P_{\theta^T T_\class(X_\class)}}[c(z, f_\class(z)] \\
        \textnormal{s.t.}\quad 
        & P_{f_\class(\theta^T T_\class(X_\class))} = P_{f_{\class'}(\theta^T T_{\class'}(X_{\class'}))} \,\,\, \forall \class \neq \class'
    \end{aligned}\right)
\end{align}
At each iteration, for the maximization problem, we use \autoref{alg:max-k-sw} to find $\theta$.
To solve the outer minimization of $T_\class$, we update our transformation with one layer to achieve the global minimum (given the current $\theta$), i.e., $T'_\class = t_\class \circ T_\class$ will be the optimal solution where we construct $t_\class$ based on solutions to the inner optimization.
Specifically, we solve the inner 1D problems (given a fixed $\theta$ and $T_\class$) denoted by $f^*_1, \cdots, f^*_\nclass$
by estimating the 1D CDF function for each class ($k$ in total) using the whole dataset and finding the \emph{local 1D} barycenter map.
\begin{align}
    f_1^*, \cdots, f_\nclass^* &= 
        \argmin_{f_1, \cdots, f_\nclass} \sum_{\class=1}^\nclass \weight_\class \E_{z\sim P_{\theta^T Z_\class}}[c(z, f_\class(z)] \\
        \textnormal{s.t.}\quad 
        & P_{f_\class(\theta^T Z_\class)} = P_{f_{\class'}(\theta^T Z_{\class'})} \,\,\, \forall \class \neq \class' \nonumber
\end{align}
 where $Z_\class\triangleq  T_\class(X_\class)$.
Then, we can construct $\forall \class, t_{\class}^{*}(\boldsymbol{z})=\theta f_{\class}^{*}\left(\theta^{T} \boldsymbol{z}\right)+\left(\boldsymbol{z}-\theta \theta^{T} \boldsymbol{z}\right)$ as discussed in \Cref{app-sec:m<d}.
Given a fixed $\theta$, the updated $T'_{\class} = t_{\class}^{*} \circ T_{\class}$ is the optimal solution to the outer minimization problem (even when the inner minimization is unconstrained).

\emph{Proof of optimality.}
First, note that the new random variable projected along the slice is equal to the transformed 1D slice distribution, i.e.,
\begin{align*}
    \theta^{T} T_{\class}'\left(X_{\class}\right)
    &=\theta^{T} t_{\class}^{*} \circ T_{\class}\left(X_{\class}\right)\\
    &=\theta^{T} t_{\class}^{*}\left(Z_{\class}\right)\\
     &= \theta^{T}\left(\theta f_{\class}^{*}\left(\theta^{T} Z_{\class}\right)+\left(Z_{\class}-\theta \theta^{T} Z_{\class}\right)\right)\\
    &=f_{\class}^{*}\left(\theta^{T} Z_{\class}\right)+\theta^{T} Z_{\class}-\theta^{T} Z_{\class} \\
    &= f_{\class}^{*}\left(\theta^{T} Z_{\class}\right)
\end{align*}
Now $P_{f_{\class}^*\left(\theta^{T} Z_{\class}\right)}=P_{f_{\class^{\prime}}^*\left(\theta^{T} Z_{\class^{\prime}}\right)}, \forall \class \neq \class^{\prime}$ by the alignment constraint on the $f^*_\class$'s, and thus, combining with above, we have that $P_{\theta^{T} T'_{\class}\left(X_{\class}\right)}=P_{\theta^{T} T'_{\class^{\prime}}(X_{\class^{\prime}})},\forall \class \neq \class^{\prime}$.
Thus, the $W_2$ distance along all slices is 0, and $T'_{\class}$ is the globally optimal solution per the property of $W_2$.

\section{TREE-SLICED WASSERSTEIN DIVERGENCE}
\label{app-sec:tree-sliced}
We note that our general variational algorithm could work for other variational tractable divergences such as the tree-sliced Wasserstein (tree-SW) distance \citep{le2019tree-sliced}.Because the tree-SW can be seen as a generalization of the SW distance, we could similarly define a max-tree-SW distance and a multi-distribution max-tree-SW divergence.
The maximization would be over the tree split structure rather than orthogonal directions as in the multi-distribution maximum K-sliced Wasserstein divergence.
Note that the optimal Monge maps for tree-SW are known in closed form similar to the 1D case \citep{le2019tree-sliced}.
Additionally, the barycenter is also known in closed-form \citep{le2019tree-barycenter}.Thus, the inner maximization problem over tree structures could use a decision tree algorithm to approximately solve the inner maximization problem.
The outer minimization could be solved by first finding the barycenter in closed-form and then computing the optimal maps to this barycenter in closed-form.
For this last step, the tree-SW only provides the amount of mass to move between nodes but does not explicitly define the continuous invertible function to do so.  For simplicity, we can assume the distribution has support on the unit hypercube (if it does not, then we can use CDF functions of the appropriate marginal distributions so that it does satisfy this constraint). For each movement, we could merely define a piecewise linear function defined over the unit interval to move mass across the split.
This could be defined in a top down fashion where at the root node, we use a piecewise linear function to move mass from the left to the right of the split and then recursively apply this idea the nodes below.
This would create a piecewise linear invertible function over continuous space that would match the optimal tree Monge maps.

\section{FAILURE OF SINF FOR ALIGNMENT TASK}
\label{app-sec:failure-sinf}
In \citet{DBLP:conf/icml/DaiS21}, they propose Sliced Iterative Normalizing Flows (SINF). 
SINF first projects the data into lower dimensional space using orthogonal projection found by max-\ndir-SWD. 
Then it aligns the distribution along each direction using know solution to 1D OT problems.
Though they state that this could be used to find the transformation between any two distributions, in the paper, they fix one of the distribution to be standard normal distribution.
In specific, they propose Sliced Iterative Generator (SIG) and Gaussianizing Iterative Slicing (GIS) and they report that the two models perform better for generative modeling and density estimation separately.

In this paper, we report the results of SINF-Align($0\Rightarrow1$) and SINF-Align($1\Rightarrow0$). 
When reporting WD and FID, since SINF is an invertible model, we use the inverse of SINF for inverse transformation and then compute the average.
Since they don't provide any result of applying their model for alignment task, we try our best to compare fairly - we use the same $\ndir=56$ as what they set as default value for MNIST and FashionMNIST and we don't include any hierarchical structure.
And we report the test results at the layer where SINF achieves best test WD.

We observe that for both WD and FID, SINF performs well in the test for task in the same direction of training.
In most cases, it converges quite fast and is relatively stable.
However, when we use it for inverse task, the result is very bad.
In most cases, the WD and FID would keep increasing as we add more layers.
See Figure~\ref{fig:sinf_mnist} and Figure~\ref{fig:sinf_fmnist} for qualitative results.
We want to emphasize the possible failure of directly using SINF for inverse task.
In contrast, our model is trained based on a symmetric objective which naturally avoids this problem.

\begin{figure*}[!h]
     \centering
         \begin{subfigure}{0.4\textwidth}
         \centering
         \includegraphics[width=\textwidth]{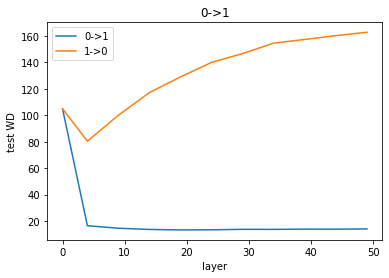}
         \caption{WD: SINF($0\Rightarrow 1$)}
     \end{subfigure}
    \begin{subfigure}{0.4\textwidth}
         \centering
         \includegraphics[width=\textwidth]{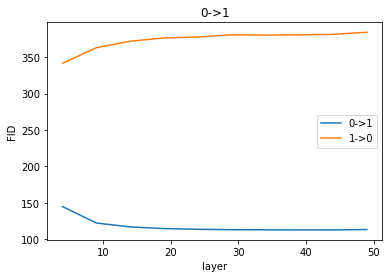}
         \caption{FID: SINF($0\Rightarrow 1$)}
     \end{subfigure}
        \begin{subfigure}{0.4\textwidth}
         \centering
         \includegraphics[width=\textwidth]{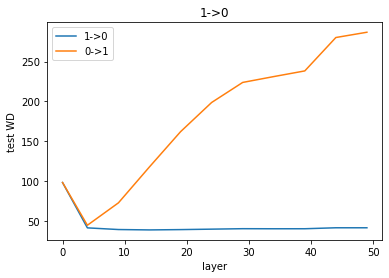}
         \caption{WD: SINF($1\Rightarrow 0$)}
     \end{subfigure}
        \begin{subfigure}{0.4\textwidth}
         \centering
         \includegraphics[width=\textwidth]{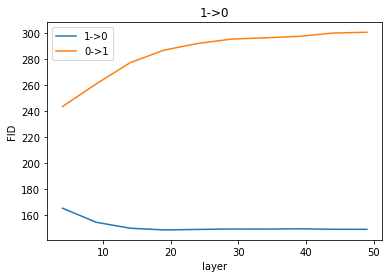}     \caption{FID: SINF($1\Rightarrow 0$)}
     \end{subfigure}
        \caption{Results of SINF-Align for MNIST($\nclass=2$). The results are recorded after each 5 layers. The label of the curve represents which task it is used for.}
    \label{fig:sinf_mnist}
\end{figure*}

\begin{figure*}[!h]
     \centering
         \begin{subfigure}{0.4\textwidth}
         \centering
         \includegraphics[width=\textwidth]{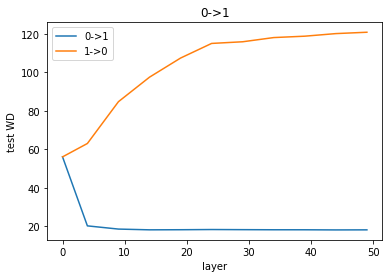}
         \caption{WD: SINF-Align($0\Rightarrow 1$)}
     \end{subfigure}
    \begin{subfigure}{0.4\textwidth}
         \centering
         \includegraphics[width=\textwidth]{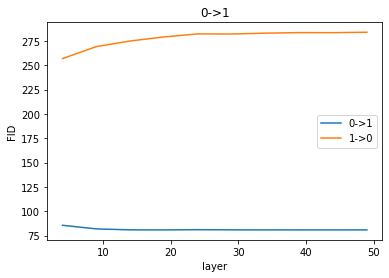}
         \caption{FID: SINF-Align($0\Rightarrow 1$)}
     \end{subfigure}
        \begin{subfigure}{0.4\textwidth}
         \centering
         \includegraphics[width=\textwidth]{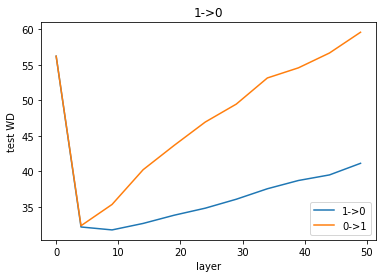}
         \caption{WD: SINF-Align($1\Rightarrow 0$)}
     \end{subfigure}
        \begin{subfigure}{0.4\textwidth}
         \centering
         \includegraphics[width=\textwidth]{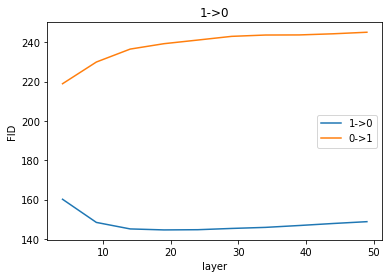}     \caption{FID: SINF-Align($1\Rightarrow 0$)}
     \end{subfigure}
        \caption{Results of SINF-Align for FashionMNIST($\nclass=2$). The results are recorded after each 5 layers. The label of the curve represents which task it is used for.}
    \label{fig:sinf_fmnist}
\end{figure*}

\section{ADDITIONAL EXPERIMENTS}
\label{app-sec:additional-experiments}

In this section, we include the quantitative results for all experiments in addition to those presented in the main paper. 
In the following subsections, brief introductions of each experiment are provided.
More experiment details are provided in the next section.

\subsection{2D Experiment for All Datasets}
For the 2D datasets with $\nclass=2$, we investigate the performance of our iterative methods along with the baselines DD and SINF-Align.
For $\nclass>2$, we only compare to DD since SINF does not have a natural extension for multiple distributions.
See Table~\ref{tab:2d-k2-barycenter-flow-full} and Table~\ref{tab:2d-k4-barycenter-flow} for quantitative results. 
See Figure~\ref{fig:rand_k4}, Figure~\ref{fig:2d_k2_full}, Figure~\ref{fig:random_k4_full} and Figure~\ref{fig:gaussian_k3_full} for expanded figures of the latent representation and translations between distributions. 
In both $\nclass=2$ and $\nclass>2$ cases, INB successfully translates the distributions to look similar to the original data (i.e., the fake distributions by translating from one class to another are similar to the original distributions).

\subsection{FashionMNIST with $\nclass=2$ Class Distributions}
We redo the experiment for MNIST with $\nclass=2$ in the main paper for FashionMNIST with $\nclass=2$. 
See Table~\ref{tab:MNIST-FMNIST-k2} for quantitative result.
See Figure~\ref{fig:mnist_k2_full} and Figure~\ref{fig:fmnist_k2_full} for expanded figures of MNIST and FashionMNIST. 
For fairness, we simply pick the first three samples in test set here. 
These examples demonstrate that our methods find a more parsimonious shared latent representation that
preserves structural similarities (e.g., black pixels) between the two digits while DD and AlignFlow do not preserve this shared structure.

\subsection{MNIST and FashionMNIST with $\nclass=3$ Class Distributions}
We investigate the performance of our models together with DD for more than two class distributions.
See Table~\ref{tab:MNIST-FMNIST-k3} for quantitative result.
See Figure~\ref{fig:mnist_k3} and Figure~\ref{fig:fmnist_k3} for mapping performance. For fairness, we pick the first sample in test set here. 
The latent representation of our models keeps more features of original samples while DD just projects to uniform distribution.

\subsection{MNIST and FashionMNIST with $\nclass=10$ Class Distributions}
We investigate the performance of our models together with DD for ten class distributions. 
See Table~\ref{tab:mnist-fmnist-k10} for quantitative result.
See Table~\ref{tab:mnist-k10-matrix-inb}, Table~\ref{tab:mnist-k10-matrix-nb}, Table~\ref{tab:mnist-k10-matrix-dd}, Table~\ref{tab:fmnist-k10-matrix-inb}, Table~\ref{tab:fmnist-k10-matrix-nb}, Table~\ref{tab:fmnist-k10-matrix-dd} for WD table for each digit with different models.
See Figure~\ref{fig:mnist-10-samples-inb}, Figure~\ref{fig:mnist-10-samples-nb}, Figure~\ref{fig:mnist-10-samples-dd}, Figure~\ref{fig:fmnist-10-samples-inb}, Figure~\ref{fig:fmnist-10-samples-nb}, Figure~\ref{fig:fmnist-10-samples-dd}, 
for expanded figures of mapping performance with different models.
For fairness, we pick the first sample in test set here. 
We can observe that with INB, most mappings seem good though the model struggles to translate in certain cases such as from 6 to 8.

\begin{table*}[!ht]
    \centering
    \caption{Transportation cost (TC), sample-based Wasserstein distance (WD, lower is better) for 2D data.
    The best methods (within one standard deviation of the top method) are bolded.
    }
    \label{tab:2d-k2-barycenter-flow-full}
     \resizebox{\columnwidth}{!}{
        \begin{tabular}{l|ll|ll|ll} 
    \hline
    \centering
     & \multicolumn{2}{c}{\textbf{Moon}} & \multicolumn{2}{c}{\textbf{Random Pattern}} & \multicolumn{2}{c}{\textbf{Circles}}   \\
    \textbf{Model}   & \textbf{WD} & \textbf{TC} & \textbf{WD} & \textbf{TC} & \textbf{WD} & \textbf{TC} \\
    \hline
    NB &0.0788 $\pm$ 0.0000 &\textbf{0.4013 $\pm$ 0.000} &0.3173 $\pm$ 0.0000& \textbf{0.9537 $\pm$ 0.0000} &0.0042 $\pm$ 0.0000 & \textbf{0.0602 $\pm$ 0.0000}
     \\ 
    Rand-NB & 0.0047 $\pm$ 0.0011 & 0.4903 $\pm$ 0.0205 &0.0620 $\pm$ 0.0188 & 1.0234 $\pm$ 0.0583 & 0.0043 $\pm$ 0.0015 & 0.0830 $\pm$ 0.0065\\ 
    INB &\textbf{0.0025 $\pm$ 0.0005} & 0.4832 $\pm$ 0.0282 & 0.0458 $\pm$ 0.0260 & 1.0207 $\pm$ 0.0270 & 0.0033 $\pm$ 0.0005 & 0.0834 $\pm$ 0.0090 \\ 
    \hline \hline
    DD &0.0085 $\pm$ 0.0000 &1.2564 $\pm$ 0.0000& 0.0469 $\pm$ 0.0000 & 3.7005 $\pm$ 0.0000 & \textbf{0.0029 $\pm$ 0.0000} & 1.2580 $\pm$ 0.0000\\
    SINF-Align(0$\Rightarrow$1)& \textbf{0.0024 $\pm$ 0.0002}&  \longdash & \textbf{0.0340 $\pm$ 0.0083} & \longdash & \textbf{0.0028 $\pm$ 0.0002} &\longdash\\
    SINF-Align(1$\Rightarrow$0)& \textbf{0.0026 $\pm$ 0.0003}& \longdash & 0.0637 $\pm$ 0.0105 & \longdash & \textbf{0.0029 $\pm$ 0.0002} & \longdash\\
    \hline
    \end{tabular}
    }
\end{table*}

\begin{table}[!ht]
    \centering
    \caption{The results for the 2D random pattern dataset with $\nclass=4$ and 2D Gaussian with $\nclass=3$  demonstrate that our methods still perform well for $\nclass>2$ in terms of the pushforward constraint, which is measured by the empirical Wasserstein-2 distance on test data (WD). 
    The best methods (within one standard deviation of the top method) are bolded.
    }
    \label{tab:2d-k4-barycenter-flow}
         \resizebox{0.6\columnwidth}{!}{
        \begin{tabular}{l|ll|ll} 
    \hline
    \centering
     & \multicolumn{2}{c}{\textbf{Random Pattern ($\nclass=4$)}} & \multicolumn{2}{c}{\textbf{Gaussian ($\nclass=3$)}}  \\
    \textbf{Model}   & \textbf{WD} & \textbf{TC} & \textbf{WD} & \textbf{TC}  \\
    \hline
    NB &0.488 $\pm$ 0.000& \textbf{9.084 $\pm$ 0.000} & 0.692 $\pm$ 0.000 &\textbf{7.027 $\pm$ 0.000} 
     \\ 
    Rand-NB &\textbf{0.155 $\pm$ 0.023} & 9.652 $\pm$ 0.094 & \textbf{0.067 $\pm$ 0.001}  &7.469 $\pm$ 0.018\\ 
    INB & \textbf{0.153 $\pm$ 0.023} & 9.532 $\pm$ 0.062 & \textbf{0.065 $\pm$ 0.002} & 7.461 $\pm$ 0.006 \\ 
    \hline \hline
    DD & \textbf{0.154 $\pm$ 0.000} & 9.434 $\pm$ 0.000 & 0.096 $\pm$ 0.000 & 7.851 $\pm$ 0.000\\
    \hline
    \end{tabular}
    }
\end{table}

\begin{table*}[!ht]
    \centering
    \caption{Results for FashionMNIST with $\nclass=2$.
    The best methods (within one standard deviation of the top method) are bolded.}
    \label{tab:MNIST-FMNIST-k2}
         \resizebox{0.75\columnwidth}{!}{
    \begin{tabular}{l|llll} 
    \hline
\textbf{Model} &   \textbf{WD} & \textbf{FID} & \textbf{TC} & \textbf{Time(s)}\\
    \hline
    NB &44.038 $\pm$ 0.000  & 118.285 $\pm$ 0.000 & \textbf{20.522 $\pm$ 0.000}&\textbf{40}\\ 
INB ($\nlayer=20$) &24.976 $\pm$ 0.092&84.802 $\pm$ 0.744&25.964 $\pm$ 0.122&430\\
    INB ($\nlayer=250$) &\textbf{24.553 $\pm$ 0.129}&\textbf{79.829 $\pm$ 0.928}&26.989 $\pm$ 0.060&2800\\ 
    \hline \hline
    DD &27.913 $\pm$ 0.000 &90.546 $\pm$ 0.000&181.401 $\pm$ 0.000&300\\
SINF-Align($0\Rightarrow 1$)& 41.111 $\pm$ 0.800& 169.722 $\pm$ 1.452 &  \longdash & 50\\
    SINF-Align($1\Rightarrow 0$) &  31.897 $\pm$ 0.184& 187.153 $\pm$ 0.670&   \longdash& 50\\
    \hline
    \end{tabular}}
\end{table*}

\begin{table*}[!ht]
    \centering
    \caption{Results for MNIST and FashionMNIST with $\nclass=3$.
    It shows that our method enables a natural extension beyond the two class case without requiring a significant increase in computational complexity.
    The best methods (within one standard deviation of the top method) are bolded. INB used for FashionMNIST is set to be $\nlayer=100$ and $\ndir=10$.}
    \label{tab:MNIST-FMNIST-k3}
    \resizebox{\columnwidth}{!}{
    \begin{tabular}{l|llll|llll} 
    \hline
    \textbf{Dataset} & \multicolumn{4}{c}{\textbf{MNIST($\nclass=3$)}} &\multicolumn{4}{c}{\textbf{FashionMNIST($\nclass=3$)}}\\
    \textbf{Model} & \textbf{WD} &  \textbf{FID}& \textbf{TC} & \textbf{Time(s)}& \textbf{WD} &  \textbf{FID}& \textbf{TC} & \textbf{Time(s)}\\
    \hline
    NB & 84.408 $\pm$ 0.000 &229.778 $\pm$ 0.000 &\textbf{28.958 $\pm$ 0.000} &\textbf{25} & 71.341 $\pm$ 0.000 & 166.114 $\pm$ 0.000 & \textbf{28.233 $\pm$ 0.000} & \textbf{25} \\ 
    INB & \textbf{40.116 $\pm$ 0.115} & \textbf{158.940 $\pm$ 0.695} & 34.062 $\pm$ 0.090 & 3700& \textbf{41.820 $\pm$ 0.142} & \textbf{116.871 $\pm$ 1.615} & 34.374 $\pm$ 0.077 & 1400\\ 
    \hline \hline
    DD & 60.226 $\pm$ 0.000 &220.308 $\pm$ 0.000 &233.354 $\pm$ 0.000 &  320& 44.975 $\pm$ 0.000 & 131.043 $\pm$ 0.000 & 171.150 $\pm$ 0.000 & 470\\
    \hline
    \end{tabular}}
\end{table*}

\begin{table*}[h!]
    \begin{center}
    \caption{
    Transportation cost (TC), sample-based Wasserstein distance (WD, lower is better), FID (lower is better) and time for MNIST and FashionMNIST($\nclass=10$).
    }
    \label{tab:mnist-fmnist-k10}
    \resizebox{1\columnwidth}{!}{
    \begin{tabular}{l|llll|llll} 
    \hline
    \textbf{Dataset} & \multicolumn{4}{c}{\textbf{MNIST($\nclass=10$)}} & \multicolumn{4}{c}{\textbf{FashionMNIST($\nclass=10$)}}\\
    \textbf{Model} & \textbf{WD} &  \textbf{FID}& \textbf{TC} & \textbf{Time(s)}& \textbf{WD} &  \textbf{FID}& \textbf{TC} & \textbf{Time(s)}\\
    \hline
    NB & 65.674 $\pm$ 0.000 & 190.920 $\pm$ 0.000 & \textbf{25.907 $\pm$ 0.000} & \textbf{90} & 60.288 $\pm$ 0.000 & 172.690 $\pm$ 0.000&\textbf{47.272 $\pm$ 0.000}& \textbf{90}\\ 
    INB & \textbf{41.044 $\pm$ 0.076} &\textbf{86.264 $\pm$ 0.550} &28.934 $\pm$ 0.140 &5000 & \textbf{36.439 $\pm$ 0.042} & \textbf{122.619 $\pm$ 0.714} & 55.128 $\pm$ 0.043 & 5000 \\ 
    \hline \hline
    DD &53.587 $\pm$ 0.000 &187.475 $\pm$ 0.000 & 227.171 $\pm$ 0.000&1700 & 40.788 $\pm$ 0.000 & 126.625 $\pm$ 0.000 & 127.099 $\pm$ 0.000&1560\\
    \hline
    \end{tabular}}
    \end{center}
\end{table*}

\begin{table}[h!]\centering
\caption{Multi-distribution ($\nclass=10$) results for MNIST with INB.
        The labels of the rows represent the class of real samples and the labels of the columns represent the class of flipped samples e.g. the number in the row "2" and column "4" represents the WD between the real "4" samples and the fake "4" samples flipped from "2" samples.  }
       \label{tab:mnist-k10-matrix-inb}
\resizebox{0.7\columnwidth}{!}{
\begin{tabular}{|r|r|r|r|r|r|r|r|r|r|r|r|}
\hline & \textbf{0} & \textbf{1} & \textbf{2} & \textbf{3} & \textbf{4} & \textbf{5} & \textbf{6} & \textbf{7} & \textbf{8} & \textbf{9} \\
\hline \textbf{0} & 0.10 & 13.37 & 48.57 & 42.77 & 35.85 & 41.80 & 36.13 & 30.86 & 45.42 & 32.10 \\
\hline \textbf{1} & 39.86 & 0.49 & 47.85 & 41.98 & 34.91 & 41.60 & 34.78 & 29.59 & 44.08 & 31.01 \\
\hline \textbf{2} & 41.02 & 13.86 & 0.05 & 43.88 & 36.75 & 43.93 & 37.43 & 31.24 & 46.24 & 33.47 \\
\hline \textbf{3} & 40.80 & 13.55 & 48.66 & 0.04 & 36.94 & 43.38 & 36.44 & 31.32 & 45.81 & 33.22 \\
\hline \textbf{4} & 40.63 & 14.12 & 48.89 & 43.24 & 0.09 & 43.37 & 37.22 & 31.68 & 46.12 & 32.89 \\
\hline \textbf{5} & 40.61 & 13.46 & 49.03 & 42.95 & 36.39 & 0.08 & 36.20 & 31.42 & 45.66 & 32.88 \\
\hline \textbf{6} & 40.31 & 13.84 & 48.87 & 42.56 & 36.26 & 41.57 & 0.11 & 30.19 & 44.88 & 32.18 \\
\hline \textbf{7} & 40.08 & 13.44 & 48.40 & 42.42 & 36.58 & 42.43 & 35.57 & 0.14 & 45.69 & 32.91 \\
\hline \textbf{8} & 40.59 & 13.67 & 49.48 & 44.32 & 37.94 & 43.41 & 36.47 & 32.01 & 0.08 & 34.12 \\
\hline \textbf{9} & 40.16 & 13.88 & 48.44 & 43.36 & 35.54 & 42.42 & 36.26 & 30.97 & 45.51 & 0.14\\
\hline
\end{tabular}
}
\end{table}

\begin{table}[h!]\centering
\caption{Multi-distribution ($\nclass=10$) results for MNIST with NB.  }
       \label{tab:mnist-k10-matrix-nb}
       \resizebox{0.7\columnwidth}{!}{
\begin{tabular}{|r|r|r|r|r|r|r|r|r|r|r|}
\hline & \textbf{0} & \textbf{1} & \textbf{2} & \textbf{3} & \textbf{4} & \textbf{5} & \textbf{6} & \textbf{7} & \textbf{8} & \textbf{9} \\
\hline \textbf{0} & 0.05 & 36.62 & 68.93 & 57.06 & 58.73 & 57.39 & 55.07 & 57.05 & 60.75 & 53.41 \\
\hline \textbf{1} & 84.16 & 0.19 & 79.58 & 69.92 & 61.35 & 76.99 & 69.73 & 58.62 & 71.06 & 62.15 \\
\hline \textbf{2} & 69.26 & 34.74 & 0.03 & 57.43 & 55.72 & 69.63 & 64.54 & 50.98 & 61.80 & 54.20 \\
\hline \textbf{3} & 66.71 & 34.93 & 66.25 & 0.02 & 60.91 & 56.30 & 62.10 & 56.63 & 60.02 & 56.26 \\
\hline \textbf{4} & 72.76 & 30.54 & 71.53 & 66.01 & 0.03 & 68.37 & 61.77 & 45.46 & 62.11 & 39.27 \\
\hline \textbf{5} & 62.54 & 35.27 & 73.39 & 51.94 & 58.95 & 0.01 & 59.55 & 55.85 & 57.19 & 52.40 \\
\hline \textbf{6} & 63.69 & 33.64 & 72.76 & 66.62 & 55.87 & 65.76 & 0.06 & 59.08 & 64.91 & 55.18 \\
\hline \textbf{7} & 78.07 & 32.02 & 72.60 & 68.39 & 52.64 & 71.43 & 69.12 & 0.05 & 66.52 & 45.02 \\
\hline \textbf{8} & 67.04 & 35.10 & 68.20 & 57.73 & 55.24 & 56.15 & 61.91 & 53.14 & 0.03 & 50.75 \\
\hline \textbf{9} & 71.66 & 33.81 & 71.99 & 65.97 & 42.25 & 65.93 & 65.09 & 41.49 & 60.93 & 0.04 \\
\hline
\end{tabular}}
\end{table}

\begin{table}[h!]\centering

\caption{Multi-distribution ($\nclass=10$) results for MNIST with DD.  }
       \label{tab:mnist-k10-matrix-dd}
\resizebox{0.7\columnwidth}{!}{\begin{tabular}{|r|r|r|r|r|r|r|r|r|r|r|}
\hline & \textbf{0} & \textbf{1} & \textbf{2} & \textbf{3} & \textbf{4} & \textbf{5} & \textbf{6} & \textbf{7} & \textbf{8} & \textbf{9} \\
\hline \textbf{0} & 0.03 & 23.55 & 60.27 & 50.89 & 47.45 & 52.91 & 48.72 & 42.76 & 55.22 & 44.07 \\
\hline \textbf{1} & 54.44 & 0.01 & 62.29 & 52.96 & 48.32 & 56.41 & 49.56 & 42.67 & 57.63 & 44.82 \\
\hline \textbf{2} & 53.24 & 23.62 & 0.02 & 51.56 & 47.54 & 56.63 & 49.92 & 42.09 & 56.10 & 44.33 \\
\hline \textbf{3} & 53.00 & 24.22 & 59.51 & 0.02 & 48.86 & 52.72 & 49.47 & 43.34 & 56.60 & 45.37 \\
\hline \textbf{4} & 52.82 & 23.33 & 60.51 & 53.61 & 0.01 & 56.03 & 48.97 & 40.75 & 55.78 & 39.24 \\
\hline \textbf{5} & 51.55 & 23.49 & 60.56 & 50.33 & 48.01 & 0.01 & 48.29 & 43.18 & 53.75 & 44.84 \\
\hline \textbf{6} & 52.91 & 24.27 & 60.67 & 53.00 & 47.38 & 55.15 & 0.02 & 42.67 & 55.76 & 44.55 \\
\hline \textbf{7} & 54.21 & 24.35 & 60.74 & 53.18 & 45.79 & 55.52 & 49.59 & 0.02 & 55.39 & 41.46 \\
\hline \textbf{8} & 52.09 & 24.09 & 59.91 & 52.25 & 47.30 & 51.70 & 49.06 & 42.69 & 0.02 & 44.11 \\
\hline \textbf{9} & 53.38 & 24.41 & 60.24 & 53.74 & 41.86 & 54.78 & 49.67 & 39.50 & 55.12 & 0.02 \\
\hline
\end{tabular}}
\end{table}

\begin{table}[h!]\centering

\caption{Multi-distribution ($\nclass=10$) results for FashionMNIST with INB.  }
       \label{tab:fmnist-k10-matrix-inb}
\resizebox{0.7\columnwidth}{!}{
\begin{tabular}{|r|r|r|r|r|r|r|r|r|r|r|}
\hline & \textbf{0} & \textbf{1} & \textbf{2} & \textbf{3} & \textbf{4} & \textbf{5} & \textbf{6} & \textbf{7} & \textbf{8} & \textbf{9} \\
\hline \textbf{0} & 0.21 & 18.78 & 34.80 & 26.11 & 31.48 & 37.13 & 34.44 & 21.70 & 47.89 & 31.94 \\
\hline \textbf{1} & 33.09 & 0.94 & 34.88 & 28.01 & 31.20 & 38.49 & 35.36 & 21.96 & 50.44 & 32.31 \\
\hline \textbf{2} & 32.20 & 18.42 & 0.28 & 26.77 & 28.10 & 37.65 & 33.55 & 21.74 & 47.43 & 32.03 \\
\hline \textbf{3} & 35.70 & 19.97 & 36.51 & 0.29 & 33.18 & 38.86 & 37.59 & 22.15 & 51.41 & 33.23 \\
\hline \textbf{4} & 32.12 & 18.20 & 33.05 & 26.37 & 0.28 & 37.31 & 34.38 & 21.25 & 46.77 & 31.26 \\
\hline \textbf{5} & 37.01 & 21.38 & 39.18 & 31.12 & 36.19 & 0.13 & 41.03 & 26.57 & 53.54 & 38.28 \\
\hline \textbf{6} & 31.79 & 18.76 & 34.56 & 26.29 & 30.19 & 38.17 & 0.15 & 21.89 & 50.18 & 32.18 \\
\hline \textbf{7} & 34.11 & 19.72 & 35.82 & 29.23 & 32.75 & 39.21 & 37.38 & 0.24 & 50.34 & 33.66 \\
\hline \textbf{8} & 34.69 & 20.41 & 36.50 & 28.81 & 32.87 & 37.06 & 36.92 & 21.34 & 0.04 & 33.13 \\
\hline \textbf{9} & 33.80 & 19.93 & 35.84 & 28.69 & 32.11 & 37.73 & 36.94 & 21.83 & 48.10 & 0.06 \\
\hline
\end{tabular}
}
\end{table}

\begin{table}[h!]\centering

\caption{Multi-distribution ($\nclass=10$) results for FashionMNIST with NB.  }       \label{tab:fmnist-k10-matrix-nb}
\resizebox{0.7\columnwidth}{!}{
\begin{tabular}{|r|r|r|r|r|r|r|r|r|r|r|}
\hline & \textbf{0} & \textbf{1} & \textbf{2} & \textbf{3} & \textbf{4} & \textbf{5} & \textbf{6} & \textbf{7} & \textbf{8} & \textbf{9} \\
\hline \textbf{0} & 0.01 & 31.54 & 47.74 & 34.91 & 43.80 & 64.02 & 44.57 & 37.95 & 77.73 & 59.11 \\
\hline \textbf{1} & 56.71 & 0.05 & 70.61 & 40.41 & 62.99 & 76.58 & 67.14 & 45.92 & 100.12 & 70.20 \\
\hline \textbf{2} & 41.04 & 38.60 & 0.00 & 37.02 & 29.53 & 58.41 & 37.37 & 38.19 & 69.86 & 58.64 \\
\hline \textbf{3} & 47.76 & 28.43 & 60.22 & 0.02 & 54.82 & 69.33 & 59.57 & 40.30 & 91.44 & 64.35 \\
\hline \textbf{4} & 45.42 & 36.13 & 36.38 & 39.10 & 0.01 & 62.69 & 39.69 & 39.06 & 73.08 & 60.68 \\
\hline \textbf{5} & 64.39 & 50.18 & 73.13 & 55.80 & 68.76 & 0.02 & 71.95 & 31.25 & 85.69 & 46.48 \\
\hline \textbf{6} & 35.17 & 31.87 & 37.16 & 34.27 & 33.03 & 58.02 & 0.00 & 37.19 & 71.92 & 56.84 \\
\hline \textbf{7} & 67.13 & 51.83 & 76.76 & 57.98 & 70.58 & 52.46 & 74.84 & 0.08 & 95.52 & 63.89 \\
\hline \textbf{8} & 50.87 & 43.78 & 53.48 & 46.20 & 46.52 & 51.57 & 53.05 & 29.71 & 0.00 & 47.45 \\
\hline \textbf{9} & 61.18 & 47.73 & 67.24 & 53.22 & 62.20 & 50.29 & 67.60 & 31.60 & 76.36 & 0.01 \\
\hline
\end{tabular}
}
\end{table}

\begin{table}[h!]\centering

\caption{Multi-distribution ($\nclass=10$) results for FashionMNIST with DD.  }
       \label{tab:fmnist-k10-matrix-dd}
\resizebox{0.7\columnwidth}{!}{
\begin{tabular}{|r|r|r|r|r|r|r|r|r|r|r|}
\hline & \textbf{0} & \textbf{1} & \textbf{2} & \textbf{3} & \textbf{4} & \textbf{5} & \textbf{6} & \textbf{7} & \textbf{8} & \textbf{9} \\
\hline \textbf{0} & 2.27 & 22.32 & 34.02 & 31.18 & 34.12 & 46.94 & 38.32 & 25.59 & 58.34 & 39.85 \\
\hline \textbf{1} & 33.26 & 0.55 & 35.24 & 31.82 & 36.16 & 47.57 & 39.30 & 25.58 & 61.59 & 41.87 \\
\hline \textbf{2} & 32.05 & 21.74 & 2.32 & 30.75 & 32.19 & 45.94 & 37.05 & 24.53 & 56.22 & 40.31 \\
\hline \textbf{3} & 33.77 & 24.07 & 33.27 & 1.06 & 34.72 & 47.59 & 40.46 & 26.30 & 60.18 & 39.87 \\
\hline \textbf{4} & 33.57 & 22.20 & 31.69 & 31.13 & 1.45 & 45.99 & 35.26 & 24.35 & 55.70 & 39.95 \\
\hline \textbf{5} & 34.94 & 23.12 & 36.14 & 34.61 & 35.80 & 0.01 & 40.96 & 25.04 & 61.08 & 38.95 \\
\hline \textbf{6} & 32.41 & 22.60 & 32.93 & 31.12 & 32.43 & 46.53 & 1.80 & 24.55 & 57.87 & 39.91 \\
\hline \textbf{7} & 34.12 & 24.21 & 35.72 & 34.49 & 36.91 & 45.84 & 40.66 & 0.15 & 61.25 & 40.71 \\
\hline \textbf{8} & 33.00 & 22.01 & 33.58 & 31.80 & 34.06 & 46.79 & 37.78 & 24.64 & 0.97 & 40.03 \\
\hline \textbf{9} & 33.95 & 23.10 & 35.69 & 33.70 & 36.26 & 45.88 & 39.43 & 24.41 & 58.87 & 0.29 \\
\hline
\end{tabular}
}
\end{table}

\section{EXPERIMENT DETAILS}
\label{app-sec:experimental-details}

\subsection{Histogram-based 1D Density Estimators for NB Method}
For high flexibility yet low computational cost, we choose to use a histogram-based density estimator for our independent component (nai\"ve) layers (NB) in our experiments.
While histograms are generally efficient and reasonable non-parametric estimators, one key drawback is that you must choose the interval for the histogram (e.g., using the minimum and maximum of the data).
This can yield odd edge conditions if the interval is not chosen properly.
Thus, to avoid this challenge, we first estimate a preprocessing transformation to squeeze the data to the interval $[0,1]$ and then estimate a histogram on this fixed interval.
In particular, we merely use a Gaussian CDF (where the mean and covariance are estimated from the data) to preprocess the data.
We then estimate a histogram on the transformed data.
This can be seen as an almost trivial 1D normalizing flow where the histogram is a learned base prior distribution and the Gaussian CDF is the flow.
We use the code from deep density destructors \citep{inouye2018deep} to implement this estimation procedure.
Note that this estimation procedure only requires estimating a 1D Gaussian and a 1D histogram---both of which have minimal computational cost.

\subsection[Details when the number of target directions is less than the dimensionality]{Details when the number of target directions is less than the dimensionality ($\ndir < \ndim$)}\label{app-sec:m<d}
For the INB layer, if the number of target directions $\ndir$ is less than the dimensionality $\ndim$, we can define a partial independent components layer that only acts on $\ndir$ directions.
From a theoretical viewpoint, we could adjust our estimators as follows:
\begin{enumerate}
    \item For estimating $ \projection$, the other $\ndim-\ndir$ directions of $ \projection$ can be filled in with an arbitrary orthonormal subspace.
    \item When estimating the independent class distributions, we could assume that the $\ndim-\ndir$ directions have the same distribution for \emph{all} classes.
\end{enumerate}
The first assumption allows us to preserve the full dimensionality of the data when projecting into the latent space.
The second assumption implies that the transform along the $\ndim-\ndir$ directions is the identity because all the class distributions are the same, which implies that their barycenter is equal to the class distributions, which implies that the symmetric Monge map is merely the identity function (see \autoref{thm:univariate-map}).
Thus, it can be seen that these assumptions roughly just ignore the $\ndim-\ndir$ directions.

In practice, we do not have to actually create a full orthogonal matrix $\projection$ or estimate the class distributions along the other $\ndim-\ndir$ directions.
We can instead use truncated orthogonal matrices (i.e., where the columns are orthogonal but it is not square) and truncated joint transformations.
More formally, we can create the following invertible but ``truncated'' transform to avoid unnecessary computation as is done in \citep{DBLP:conf/icml/DaiS21}:
\begin{align}
    T_\class^*(\xvec) = \projection t_\class^*(\projection^T \xvec) + \xvec^{\perp} =  \projection t_\class^*(\projection^T \xvec)  + (\xvec -  \projection \projection^T\xvec) \, ,
\end{align}
where $t_m^*=[t_{\class,1}^*, \dots, t_{\class,\ndir}^*]$ and $\xvec^{\perp}\triangleq \xvec -  \projection \projection^T\xvec$ contains the components that are perpendicular to $ \projection$.
Note that this transformation is invertible and equivalent to the non-truncated ``theoretical'' version described above but requires significantly less computation.

\subsection{Datasets}
In each run of our experiments, we use the same data even for our simulated data (i.e., we use the same random seed for generating the data for each run). 

\paragraph{2D distributions}
For 2D data, we use the fixed samples for each repetition of experiment (i.e., we produce simulated data for all runs rather than producing new simulated data for each run). 
\begin{itemize}
    \item $\nclass=2$: Datasets of Moon, Random Pattern, Circles are generated by \texttt{make\_moons}, \texttt{make\_classification} and \texttt{make\_circles} in \texttt{sklearn.datasets} respectively. 
    The original number of training samples is 2000 and the original number of test samples is 1000.
    \item $\nclass>2$: Dataset of Random Pattern ($\nclass=4$) is generated by \texttt{make\_classification} in \texttt{sklearn.datasets}. 
    The original number of training samples is 2666 and the original number of test samples is 1334.
    Dataset of Gaussian ($\nclass=3$) is generated by \texttt{MultivariateNormal} in \texttt{torch.distributions.multivariate\_normal} with different means and covariances.
    The original number of training samples is 4000 and the original number of test samples is 2000.
\end{itemize}

\paragraph{MNIST and FashionMNIST} We first take the full MNIST dataset (70k samples) and split into training and testing split.
The dimensionality of MNIST and FashionMNIST datasets is 784.
To ensure all classes have the same number of samples in the training and test split, we take the minimum number of samples over all classes and truncate the samples of all digits to that number.
The numbers vary slightly depending on the number of class distributions $\nclass$ and datasets but are approximately 4500 samples per digit for training and 2300 samples per digit for testing (Experiment for MNIST and FashionMNIST with $\nclass=10$ has approximately 1800 samples per digit for testing). 

For our models including DD, we preprocess the data by dequantizing the original data with uniform distribution and dividing by 256 to create a continuous distribution over the unit hypercube.
For AlignFlow, the data is further normalized to the range $[-1,1]$ to serve as the input to the Real-NVP and the GAN discriminator as in the original AlignFlow paper.

See below for the exact classes we use for our experiments.
\begin{itemize}
    \item MNIST with $\nclass=2$: We use digit 0 and 1.
    \item FashionMNIST with $\nclass=2$: We use T-shirt and trouser.
    \item MNIST with $\nclass=3$: We use digit 0, 1 and 9.
    \item FashionMNIST with $\nclass=3$: We use T-shirt, trouser and pullover.
    \item MNIST with $\nclass=10$: We use digit 0-9.
\end{itemize}

\subsection{Models for 2D Experiments}

\paragraph{Two class distributions ($\nclass=2$)}
\begin{itemize}
    \item Number of layers: All iterative models (including INB, NB-INB, Rand-NB, NB-Rand-NB, DD, SINF-Align) use 15 layers.
    \item Number of dimensions for orthogonal transformation: We apply orthogonal transformation in the full space with dimension 2, i.e., $\ndir=d=2$.
\item INB: We iteratively fit NB after orthogonal transformation based on max sliced Wasserstein distance. The maximum number of iterations for $\maxkSW$ $J_{max}$ is set to be 200 for all experiments.
\item Rand-NB: We iteratively fit NB after random orthogonal transformation found by QR decomposition of a matrix generated by \texttt{torch.randn}.
\item NB-INB and NB-Rand-NB: We first perform a full-dimensional NB layer and then follow this by 14 iterations of INB/Rand-NB.
    \item DD: For the univariate histogram density estimator, we use 40 bins and set $\alpha=1$, which corresponds to the pseudo-counts added to each bin.
    \item SINF-Align: We use the SIG code from original github repo for the SINF paper \citep{DBLP:conf/icml/DaiS21}. 
\end{itemize}

\paragraph{More than two class distributions ($\nclass>2$)}
\begin{itemize}
    \item  Basically the setup is very similar to the $\nclass=2$ case. The differences are listed as below.
    \item Number of layers: All iterative models (including INB, NB-INB, Rand-NB, NB-Rand-NB, DD) use 30 layers.
    \item Number of dimensions for orthogonal transformation: We apply orthogonal transformation in the space with dimension 2.
    \item DD: It is basically the same as the $\nclass=2$ case but with a different initial destructor.  Additionally, we add a normal distribution CDF and inverse CDF as pre and post processing transformations.
\end{itemize}

\subsection{Models for MNIST and FashionMNIST}

\paragraph{Two class distributions ($\nclass=2$)}
\begin{itemize}
    \item Number of dimensions for INB: We use orthogonal transformation with $\ndir=30$ directions which is much smaller than the ambient dimensions of $d=784$ similar to the the SINF paper \citep{DBLP:conf/icml/DaiS21}.
    \item Number of layers: We use 250 layers for INB and 10 layers for DD. 
    In this way, the product of the number of layers and the number of dimensions while fitting the NB/DD are approximately the same i.e. $250\times30 \approx 10\times784$.
    \item INB: We add a normal distribution inverse CDF and CDF at the start and the end of the entire INB model as pre and post processing transformations to project the unit data into the real space for transformation.
    \item DD: The setup of DD is basically the same as what we use for 2D experiments with $\nclass>2$ except that we remove the pre and post processing transformations with the normal distribution CDF and inverse CDF since the data is already on the unit hypercube.
    \item Alignflow: The AlignFlow implementation is done through the direct clone from the Github repository with some modifications on the code and parameters setup. We first follow the AlignFlow paper to have these general parameters getting set up: the batch size is 16, the learning rate is set to a fixed \texttt{2e-04}, maximum gradient norm is 10. We further set the \texttt{data\_constraint} value inside \texttt{RealNVP} model to be 0.999998.
    
    We train 200 epochs for choices of the lambda value \texttt{1e-05} and \texttt{1e-04}.
    For the Real-NVP model, the model is a four scale setup. 
    The first three scales contain three checkerboard coupling layers followed by three channelwise coupling layers. Then the data is squeezed and split so that half the data goes to the next scale. For the final scale, we only perform the checkerboard coupling layer four times.
    The squeeze operation is simply by turning each subvolume $4 \times 4 \times 1$ into the subvolume $1 \times 1 \times 4$. And the splitting operation tries to split the last dimension into two parts. Also within each coupling layer, we parameterize the scale and translate factors by using the ResNet structure with number of blocks equals 4. And the number of channels for the ResNet is set to 32 and gets doubled every time when we switch the coupling layer from checkerboard layer to channelwise layer. For the GAN setup, the discriminator is set to have 5 convolutional layers with kernel size 4 and stride 1. The number of channels is doubled each time when passing to the next layer with the initial value 32 for the generator and 64 for the discriminator.
\end{itemize}

\paragraph{More than two class distributions ($\nclass=3$)}
\begin{itemize}
    \item INB: The INB used for FashionMNSIT is set to be $\nlayer=100$ and $\ndir=10$.
    \item The setup of other models is exactly the same as the $\nclass=2$ case.
\end{itemize}

\paragraph{More than two class distributions ($\nclass=10$)}
\begin{itemize}
    \item Number of dimensions for INB: We use orthogonal transformation with $\ndir=10$ to transform the original distribution with dimension $d=784$.
    \item Number of layers: We use 100 layers for INB since the working dimension is only $\ndir=10$ for each layer while for DD we only use 10 layers because the working dimension is $d=784$.
    Thus, if we compare the total number of dimension-wise transformations INB has $100\times10=1000$ transformations while DD can have $784 \times 10=7840$ transformations. 
    Nevertheless, INB still performs better in general based on our quantitative and qualitative results in other sections. 
\end{itemize}

\subsection{Metrics for 2D Experiments}
\begin{itemize}
    \item \emph{Transportation cost} -
    We find the averaged squared distance for each class separately and use uniform weight to take the average over all class distributions, i.e.,
    \begin{align}
        \frac{1}{\nclass} \sum_\class \frac{1}{|X_\class|} \sum_{x\in X_\class} \|x - \hat{T}_\class(x)\|_2^2 \, ,
    \end{align}
    where $|X_\class|$ is the number of samples in the test set for the $\class$-th class distribution and $\hat{T}_\class$ are the estimated maps.
    \item \emph{Wasserstein distance} - 
    For the test samples, we form ``fake'' samples for each class distribution by using the estimated maps, i.e.,
    \begin{align}
        \widetilde{X}_{\class'\to \class} = \hat{T}_\class^{-1}(\hat{T}_{\class'}(X_{\class'})), \quad \forall \class'\neq \class,
    \end{align}
    where $\hat{T}_\class$ are the estimated maps.
    We then use the Sinkhorn algorithm (with $\epsilon=10^{-4}$ and maximum iterations set to 100) to estimate the WD between the real and fake samples over all possible real-fake pairs, i.e.,
    \begin{align}
        \frac{1}{\nclass^2 -\nclass} \sum_{\class=1}^\nclass \sum_{\class'\neq \class} \textnormal{SinkhornWD}(X_\class, \widetilde{X}_{\class' \to \class}) \, .
    \end{align}
    \item \emph{Repetitions} - We repeat the entire map estimation process and metric evaluation 5 times to average over random effects and calculate standard deviations for each method (except AlignFlow).
\end{itemize}

\subsection{Metrics for MNIST and FashionMNIST}
\paragraph{Transportation Cost} 
The setup for transportation cost is the same as 2D experiments except for the experiment with AlignFlow since the scale of the input and output are different in AlignFlow.
Specifically, iterative methods such as NB, INB, and DD, have images and latent spaces to be in the range $[0,1]$ for each dimension.
However, in AlignFlow, images are normalized into $[-1,1]$, and the latent space is a normal distribution. Therefore, for the purpose of comparison with all the iterative methods, we need some modifications on the transportation cost for the AlignFlow.
We can rescale the input domain from $[-1, 1]$ to $[-0.5,0.5]$ simply by dividing the input by $2$, which gives a unit domain as for the iterative methods.
We can do the same for the latent space which makes the Gaussian prior to have a standard deviation $0.5$ instead of 1.
By doing these pre and postprocessing steps, we can get approximately the same scale in both image space and latent space as the unit scale for the iterative methods. The transportation cost is then $c(\frac{1}{2}x,\frac{1}{2}z) = \|\frac{1}{2}(x-z)\|^2 = \frac{1}{4}c(x,z)$. Therefore, we manually divide the transportation cost computed in the unscaled space by a factor of $4$ for the AlignFlow paper for the purpose of fair comparison.
Note that this added scaling favors the baseline method AlignFlow---without it, the AlignFlow transportation cost would be worse.
Additionally, because AlignFlow is so computationally expensive, we do not repeat the estimation process five times and thus cannot compute standard deviation for AlignFlow transportation costs.

\paragraph{Wasserstein Distance} The  setup of Wasserstein Distance is basically the same as that in 2D experiments except that we partition the data.
In all experiments, the partition size is set to be $500$.
The final WD is computed as the weighted average of that of all partitions.

\section{ADDITIONAL AND EXPANDED FIGURES}
\label{app-sec:additional-and-expanded-figures}
This section includes the qualitative results for additional experiments and expanded figures from the main paper.

\clearpage
\begin{figure*}[!b]
     \centering
     \begin{subfigure}[b]{0.18\textwidth}
\includegraphics[width=\textwidth]{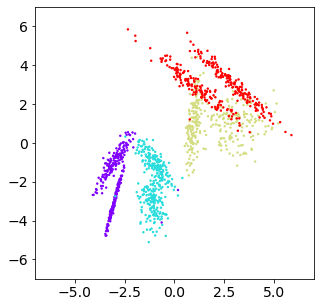}
         \caption{Original Data}
\end{subfigure}
     \begin{subfigure}[t]{0.18\textwidth}
         \centering
         \includegraphics[width=\textwidth]{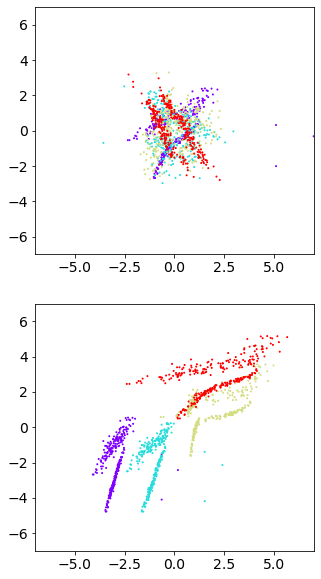}
         \caption{NB}
     \end{subfigure}
     \begin{subfigure}[t]{0.18\textwidth}
         \centering
         \includegraphics[width=\textwidth]{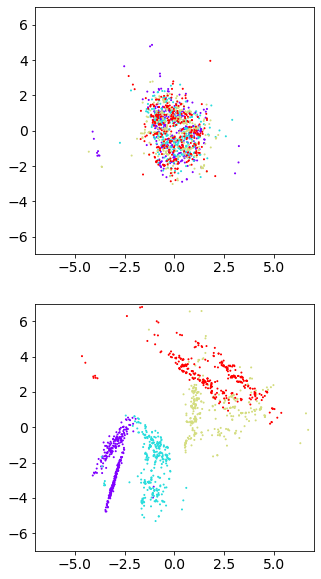}
         \caption{INB}
\end{subfigure}
     \begin{subfigure}[t]{0.18\textwidth}
         \centering
         \includegraphics[width=\textwidth]{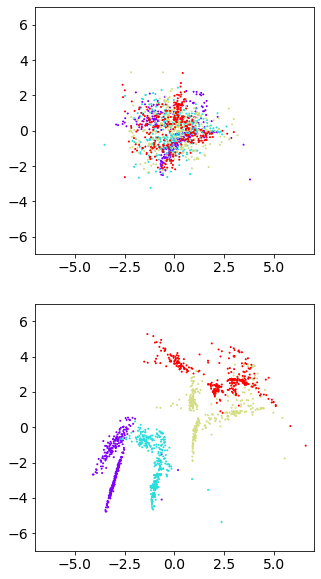}
         \caption{DD}
\end{subfigure}
        \caption{2D Random Pattern Data $(\nclass=4)$. The top row is the latent distribution found by class 1 data.
        The bottom row is the corresponding flipped distribution from it.}
        \label{fig:rand_k4}
\end{figure*}

\begin{figure*}[!t]
     \centering
     \newcommand{\mywidth}{0.47\textwidth}
     \begin{subfigure}{\mywidth}
         \centering
         \includegraphics[width=\textwidth]{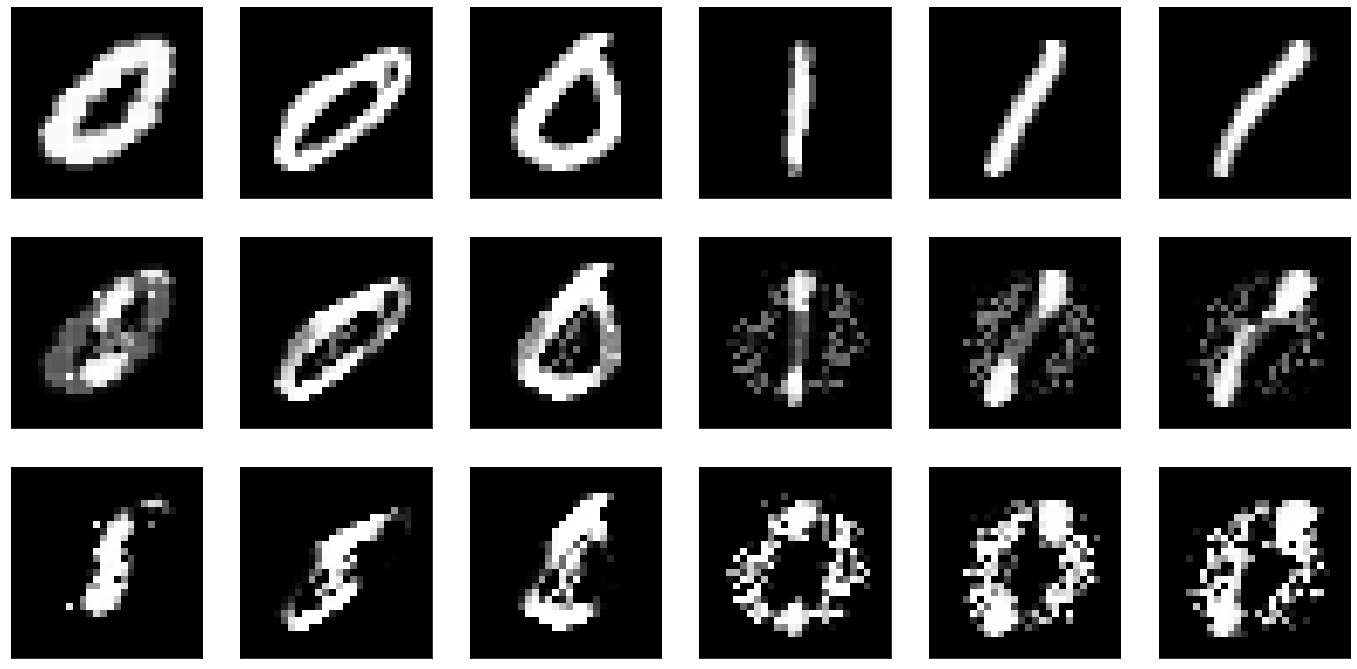}
         \caption{NB}
     \end{subfigure}
     \hspace{1em}
     \begin{subfigure}{\mywidth}
         \centering
         \includegraphics[width=\textwidth]{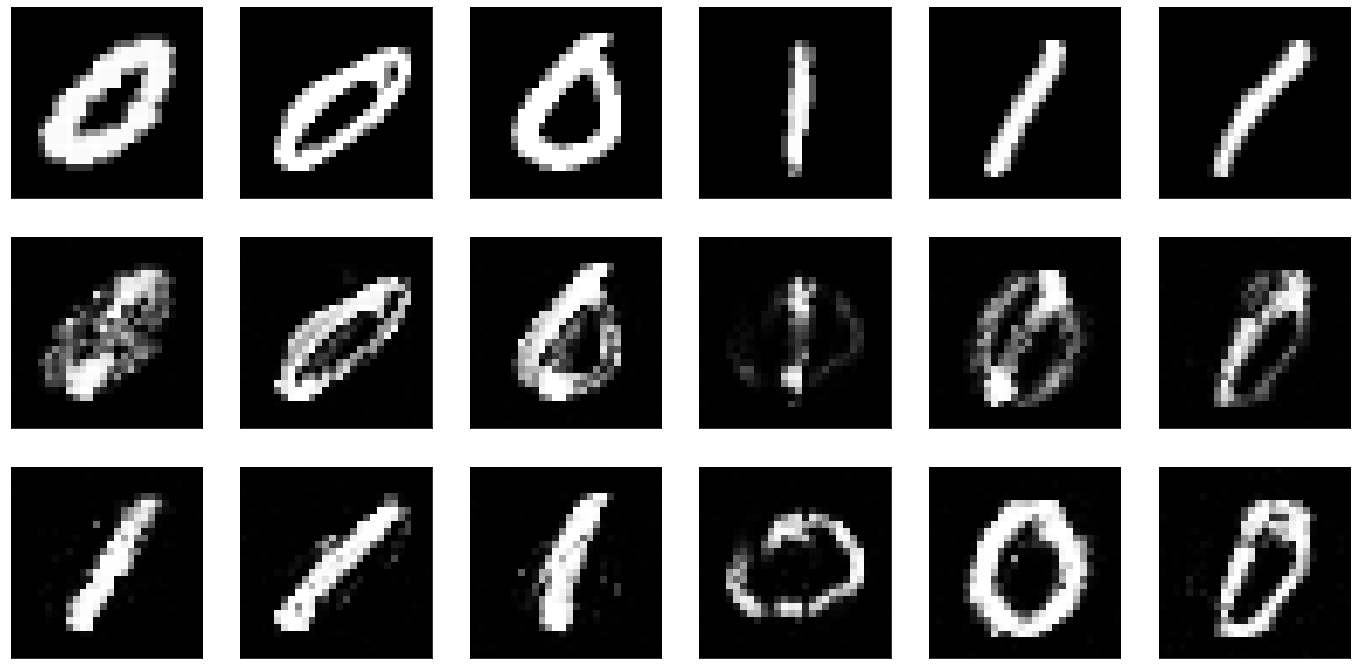}
         \caption{INB}
     \end{subfigure}
     \begin{subfigure}{\mywidth}
         \centering
         \includegraphics[width=\textwidth]{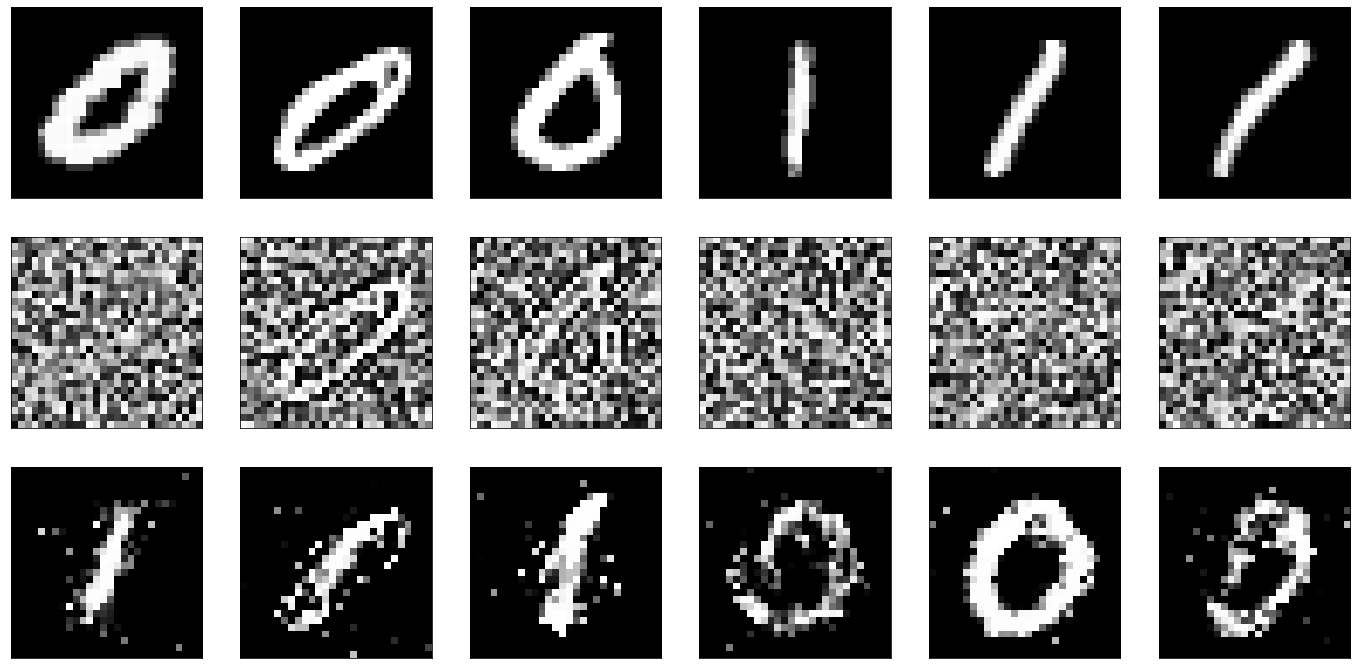}
         \caption{DD}
     \end{subfigure}
     \hspace{1em}
          \begin{subfigure}{\mywidth}
         \centering
         \includegraphics[width=\textwidth]{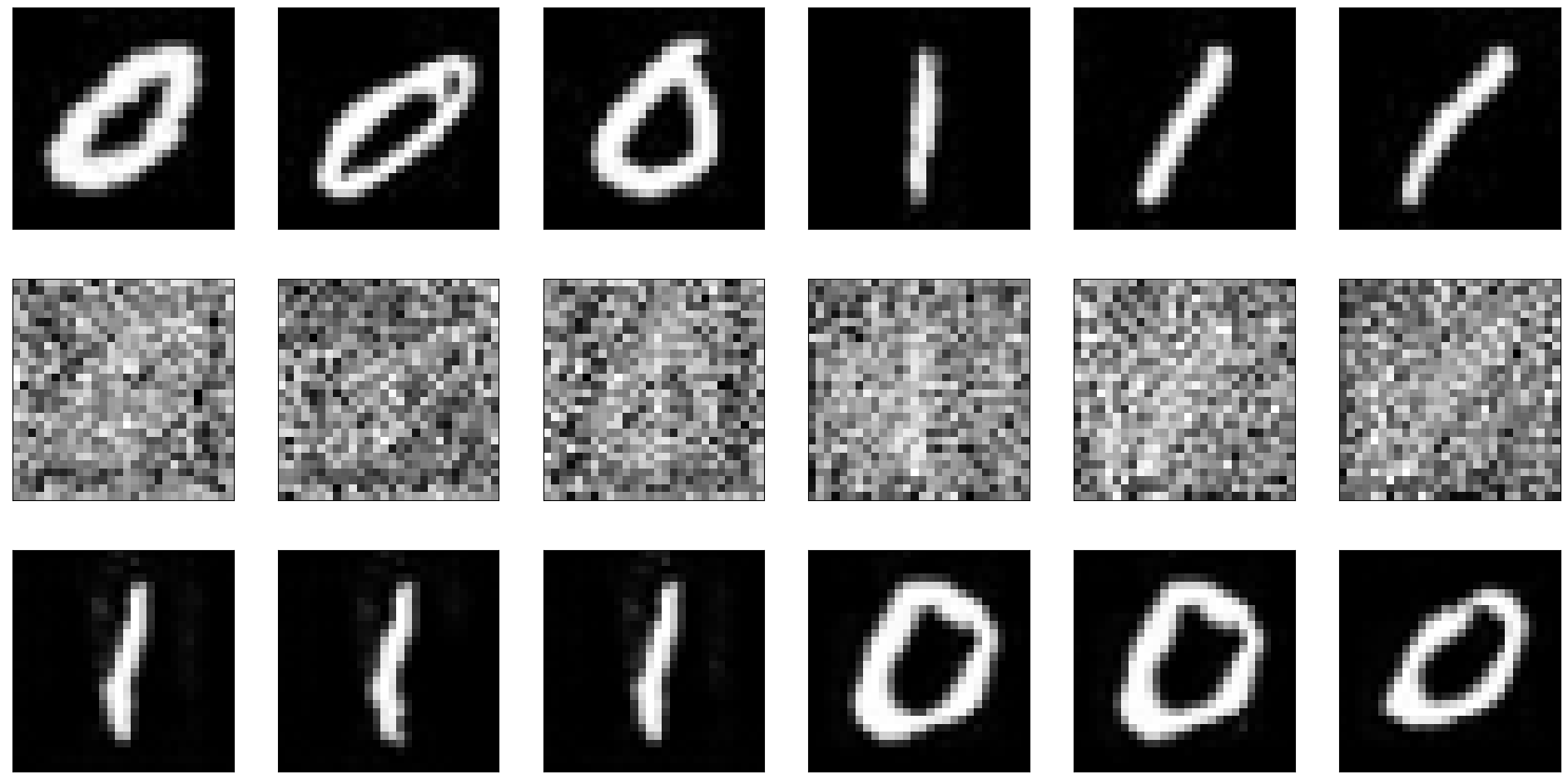}
         \caption{AlignFlow}
     \end{subfigure}
        \caption{Expanded figure of MNIST ($\nclass=2$). The first row represents the original samples. The second row represents the latent representation. The third row represents the flipped samples.}
        \label{fig:mnist_k2_full}
\end{figure*}

\begin{figure*}[!t]
 \centering
     \newcommand{\mywidth}{0.47\textwidth}
     \begin{subfigure}{\mywidth}
         \centering
         \includegraphics[width=\textwidth]{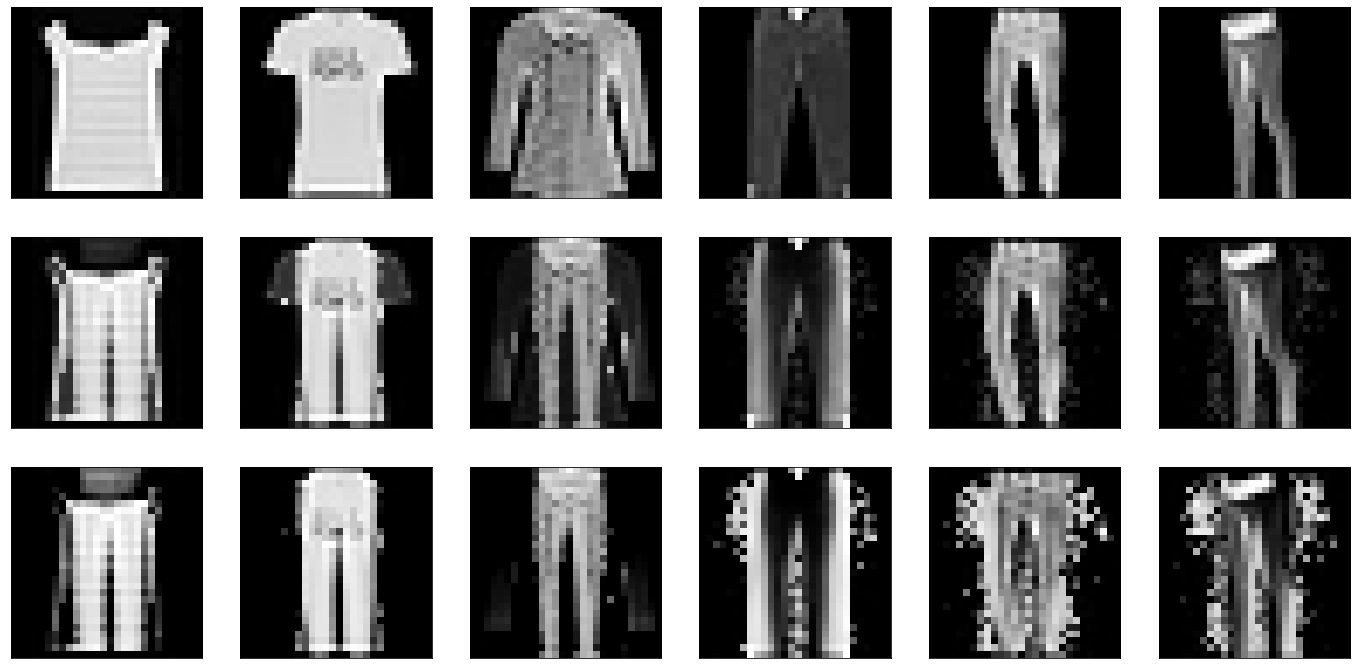}
         \caption{NB}
     \end{subfigure}
     \begin{subfigure}{\mywidth}
         \centering
         \includegraphics[width=\textwidth]{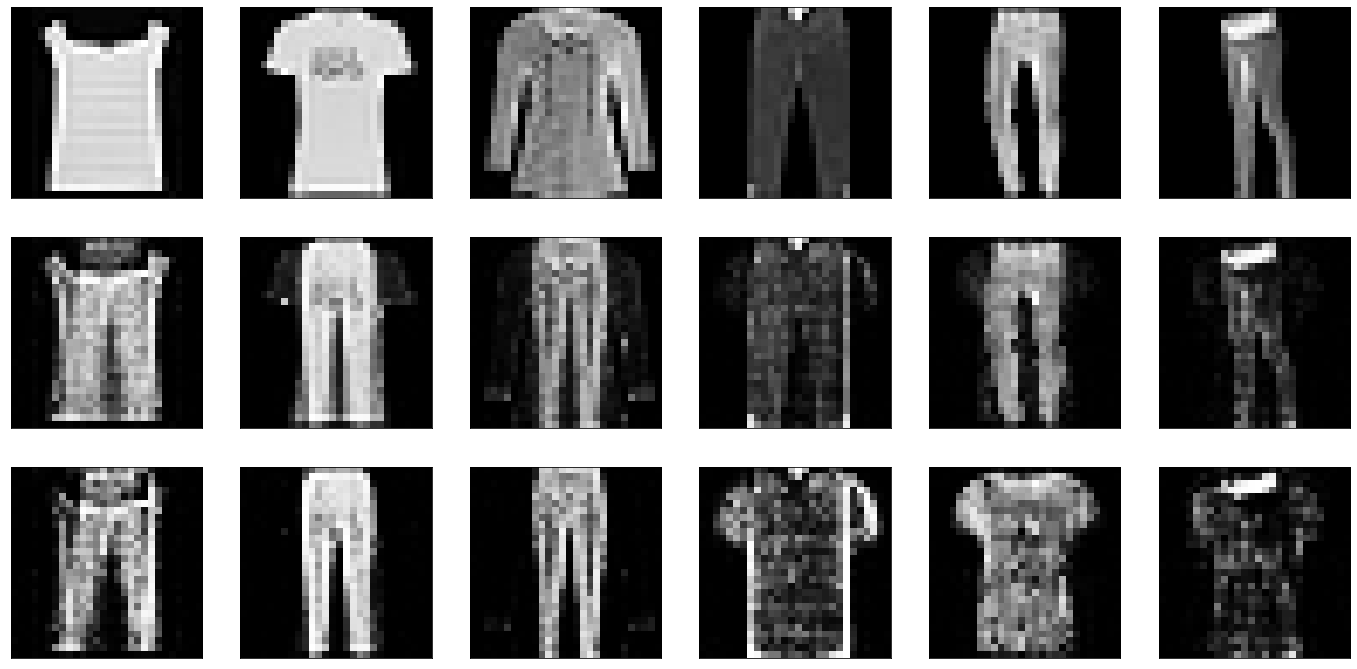}
         \caption{INB}
     \end{subfigure}
     \begin{subfigure}{\mywidth}
         \centering
         \includegraphics[width=\textwidth]{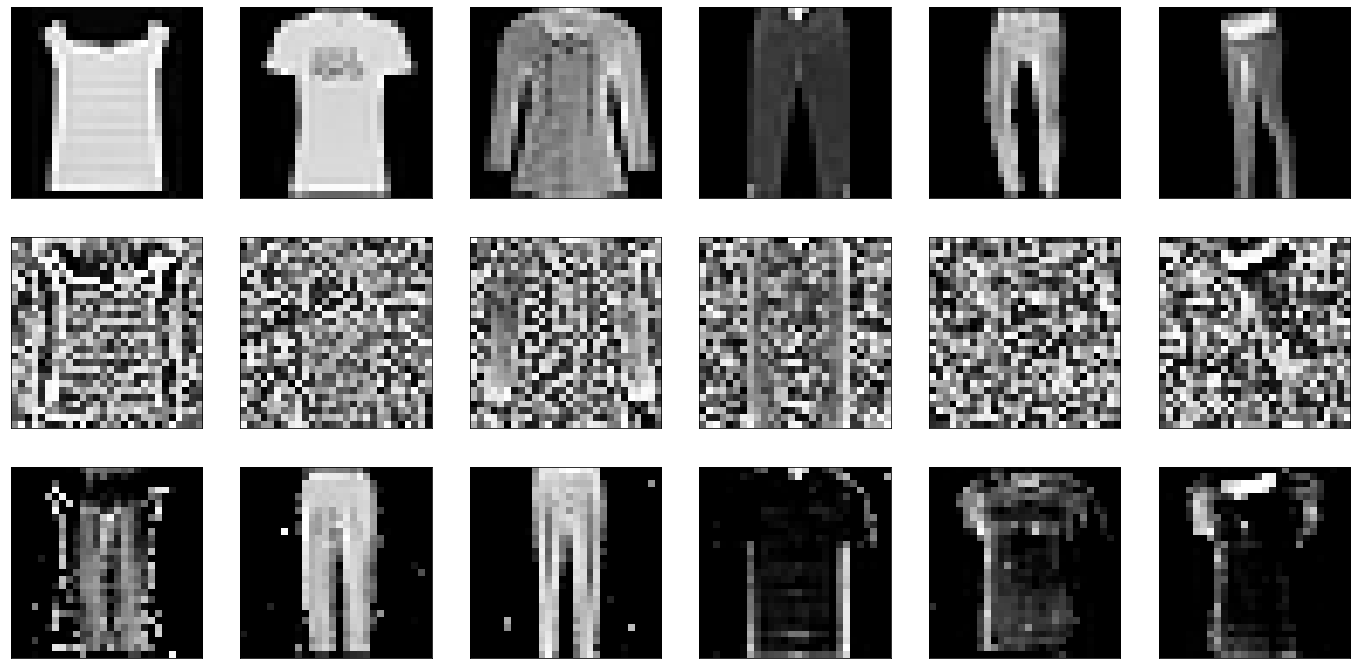}
         \caption{DD}
     \end{subfigure}
          \begin{subfigure}{\mywidth}
         \centering
         \includegraphics[width=\textwidth]{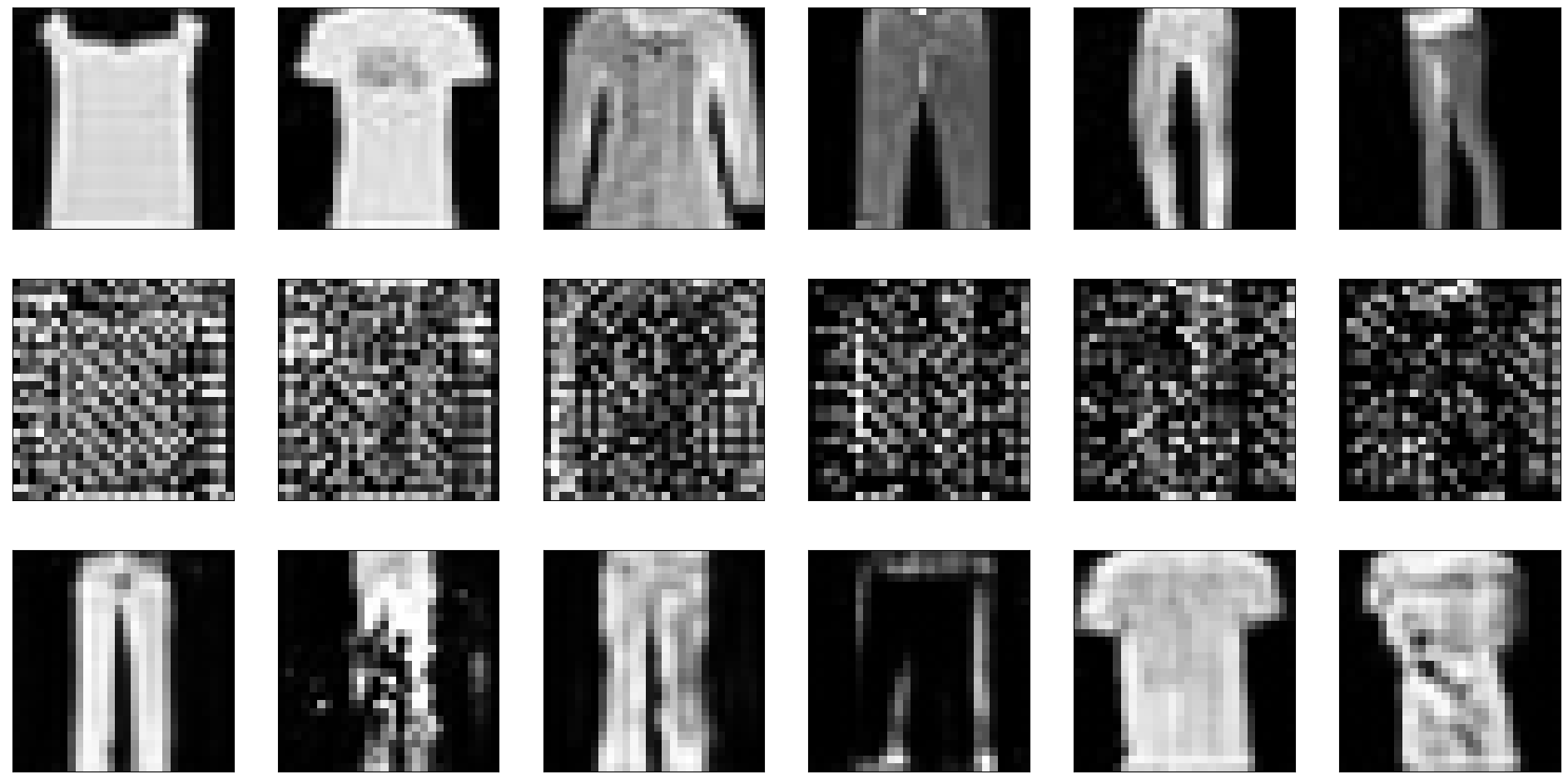}
         \caption{AlignFlow}
     \end{subfigure}
        \caption{Expanded figure of FashionMNIST ($\nclass=2$). The first row represents the original samples. The second row represents the latent representation. The third row represents the flipped samples.}
        \label{fig:fmnist_k2_full}
\end{figure*}

\begin{figure*}[!b]
     \centering
         \begin{subfigure}[t]{0.5\textwidth}
         \centering
         \includegraphics[width=\textwidth]{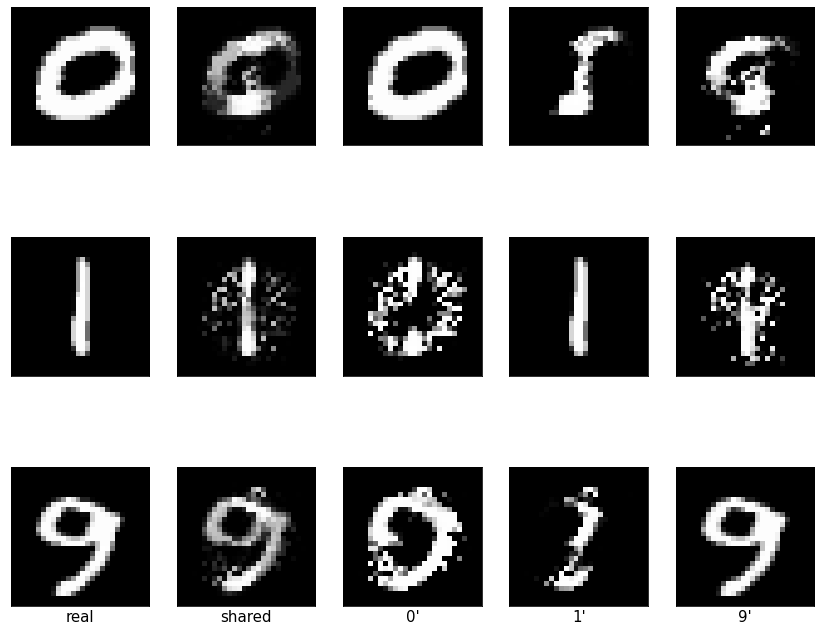}
         \caption{NB}
     \end{subfigure}
    \begin{subfigure}[t]{0.5\textwidth}
         \centering
         \includegraphics[width=\textwidth]{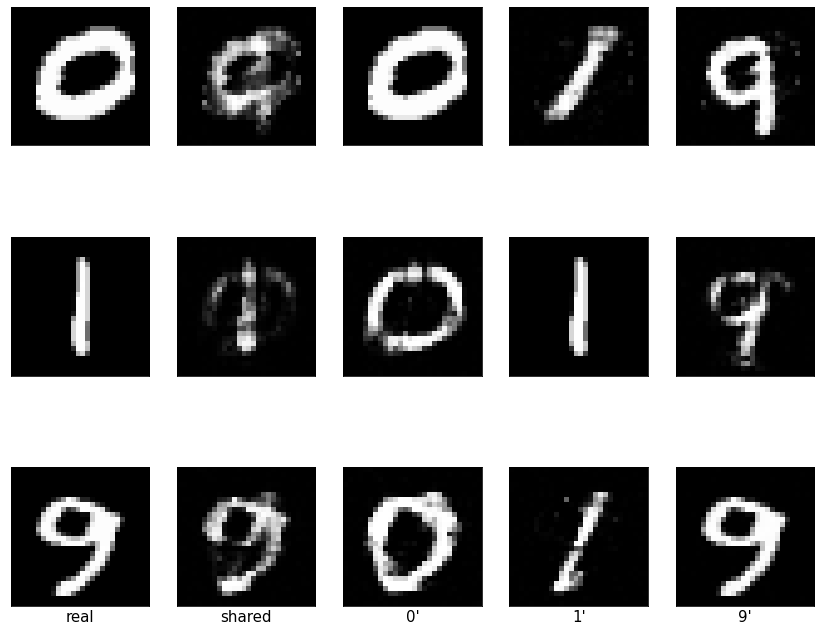}
         \caption{INB}
     \end{subfigure}
        \begin{subfigure}[t]{0.5\textwidth}
         \centering
         \includegraphics[width=\textwidth]{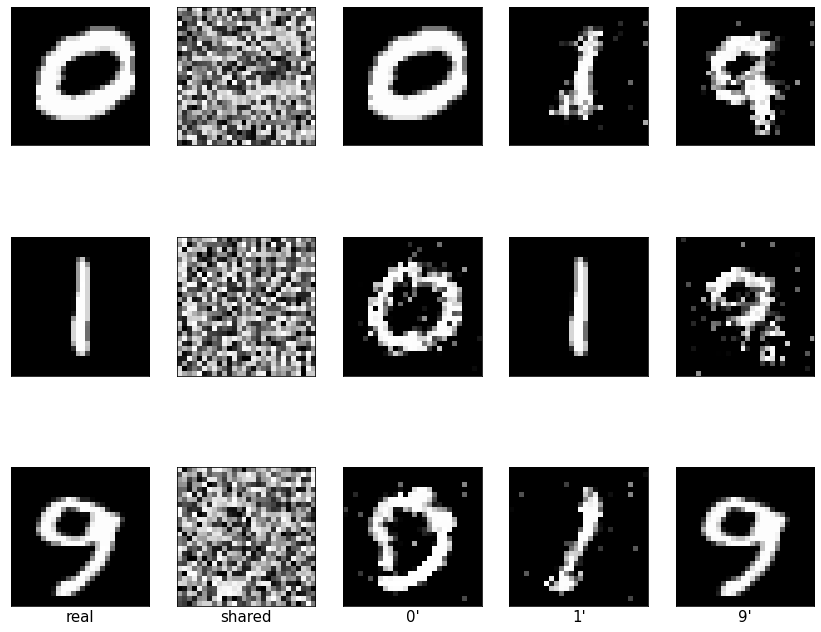}
         \caption{DD}
     \end{subfigure}
        \caption{Samples of MNIST ($\nclass=3$). 
        The first column shows the real samples and the second column shows their shared latent representations. The following columns show the mappings of the real samples to the distribution of the other digits e.g. all flipped samples in the first row are flipped from the real 0 in the first column.}
    \label{fig:mnist_k3}
\end{figure*}

\begin{figure*}[!b]
     \centering
         \begin{subfigure}[t]{0.5\textwidth}
         \centering
         \includegraphics[width=\textwidth]{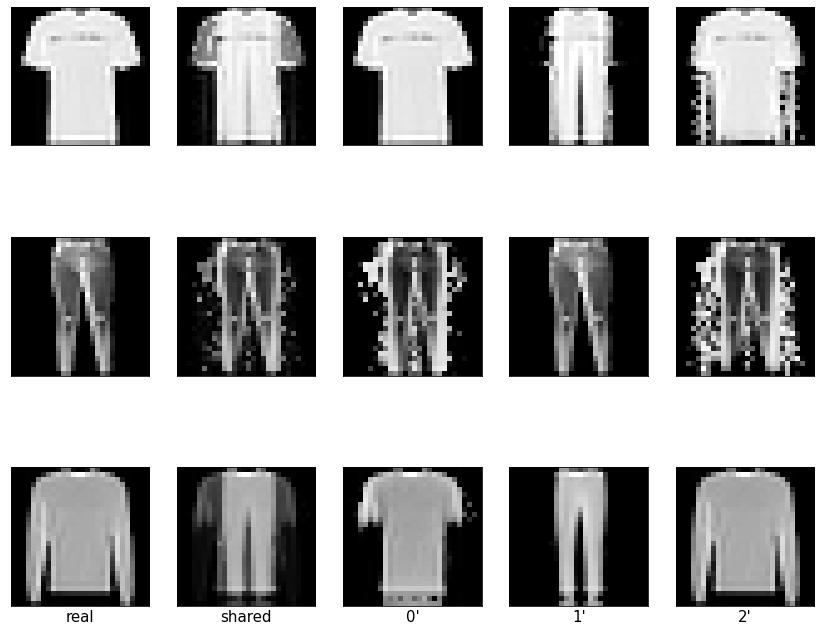}
         \caption{NB}
     \end{subfigure}
         \begin{subfigure}[t]{0.5\textwidth}
         \centering
         \includegraphics[width=\textwidth]{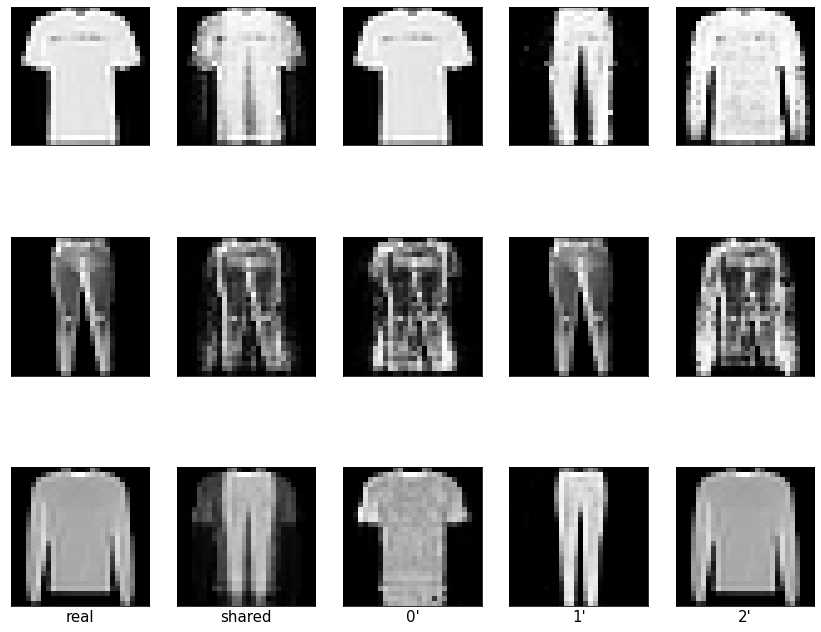}
         \caption{INB}
     \end{subfigure}
        \begin{subfigure}[t]{0.5\textwidth}
         \centering
     \includegraphics[width=\textwidth]{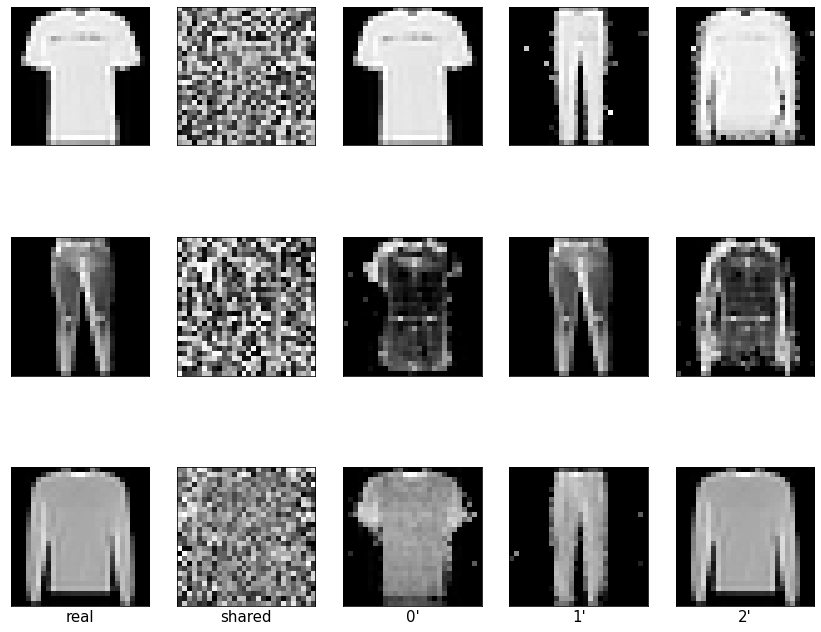}
         \caption{DD}
     \end{subfigure}
        \caption{Samples of FashionMNIST ($\nclass=3$).}
    \label{fig:fmnist_k3}
\end{figure*}
\clearpage

\begin{figure}[h!]\centering
\includegraphics[width=0.99\linewidth]{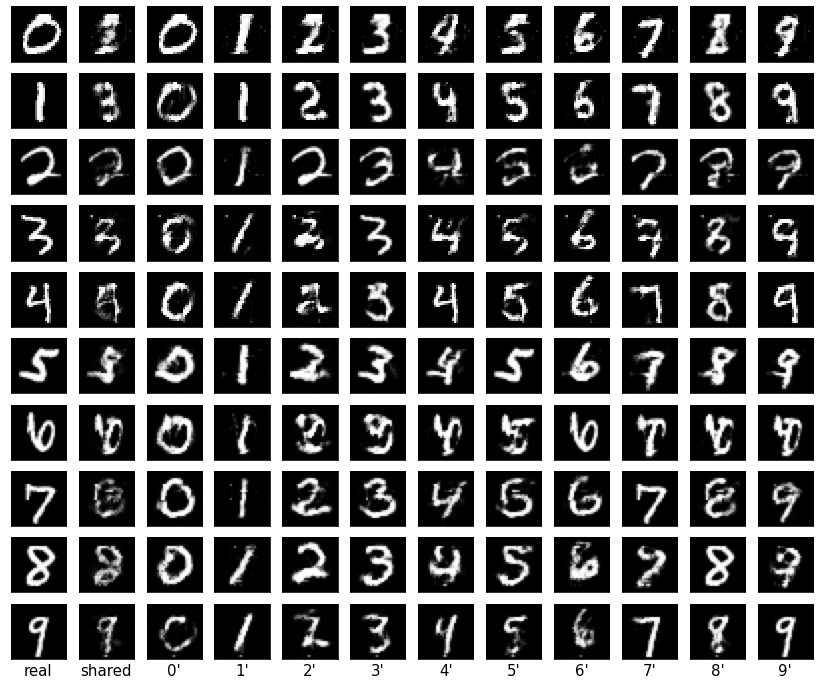}
        \caption{Multi-distribution ($\nclass=10$) results for MNIST with INB.
        The first column shows the real samples and the second column shows their shared latent representations. The following columns show the mappings of the real samples to the distribution of the other digits e.g. all flipped samples in the first row are flipped from the real 0 in the first column.}
        \label{fig:mnist-10-samples-inb}
\end{figure}

\begin{figure}[h!]\centering
\includegraphics[width=0.99\linewidth]{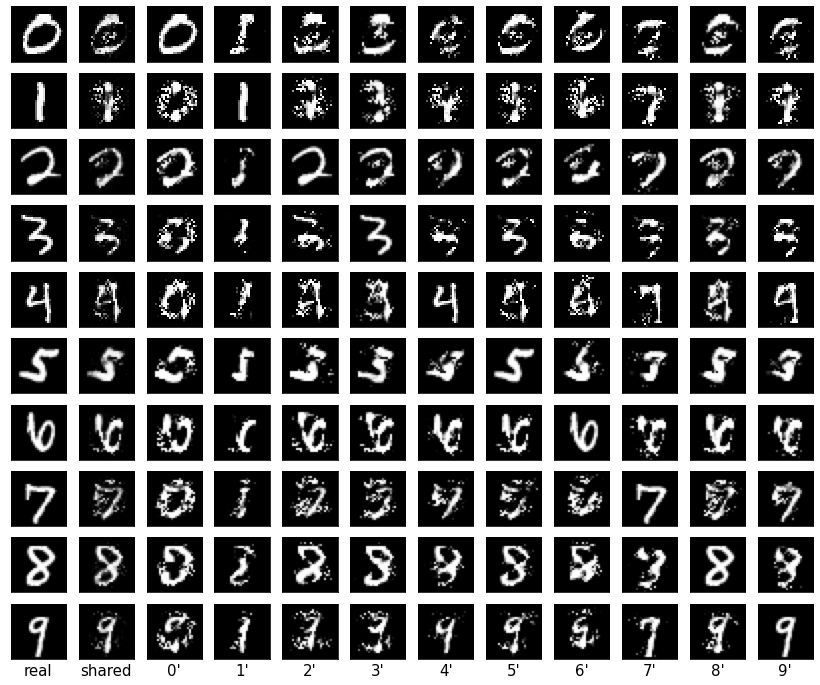}
        \caption{Multi-distribution ($\nclass=10$) results for MNIST with NB.}
        \label{fig:mnist-10-samples-nb}
\end{figure}

\begin{figure}[h!]\centering
\includegraphics[width=0.99\linewidth]{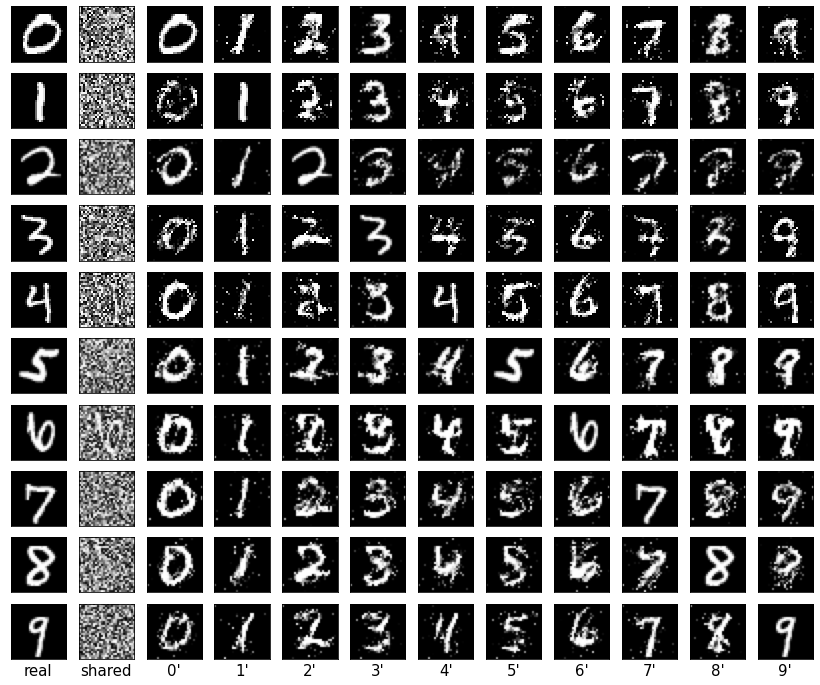}
        \caption{Multi-distribution ($\nclass=10$) results for MNIST with DD.}
        \label{fig:mnist-10-samples-dd}
\end{figure}

\begin{figure}[h!]\centering
\includegraphics[width=0.99\linewidth]{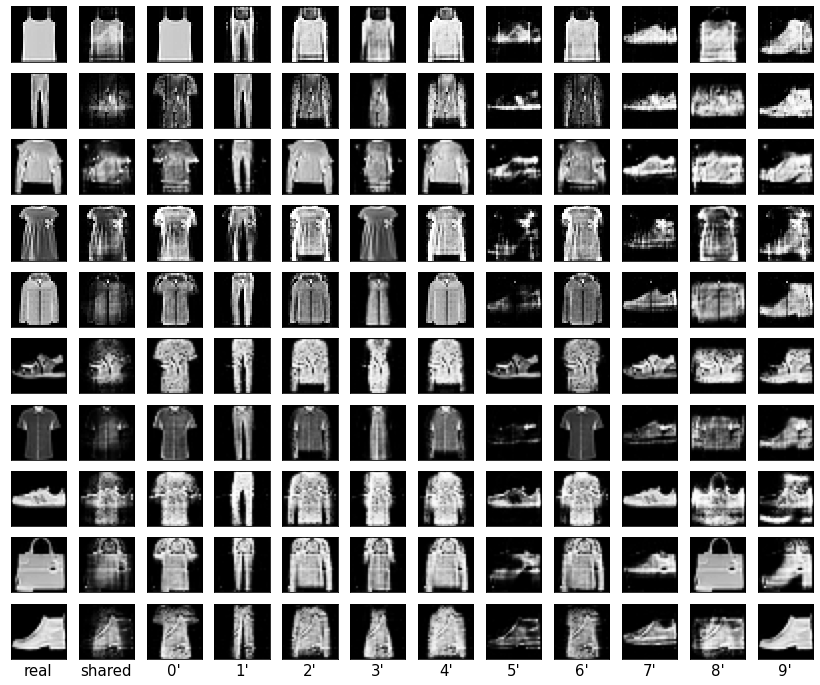}
        \caption{Multi-distribution ($\nclass=10$) results for FashionMNIST with INB.}
        \label{fig:fmnist-10-samples-inb}
\end{figure}

\begin{figure}[h!]\centering
\includegraphics[width=0.99\linewidth]{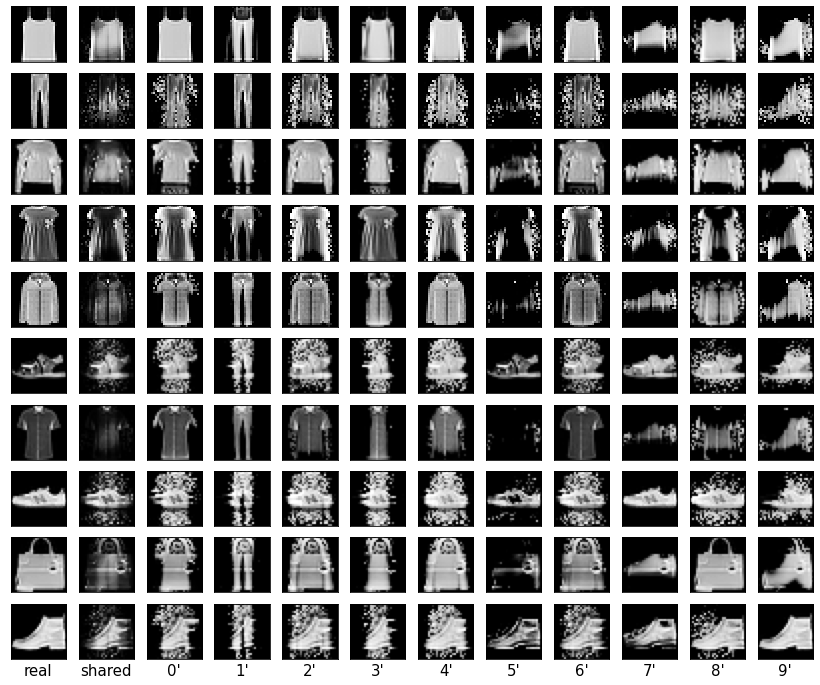}
        \caption{Multi-distribution ($\nclass=10$) results for FashionMNIST with NB.}
        \label{fig:fmnist-10-samples-nb}
\end{figure}

\begin{figure}[h!]\centering
\includegraphics[width=0.99\linewidth]{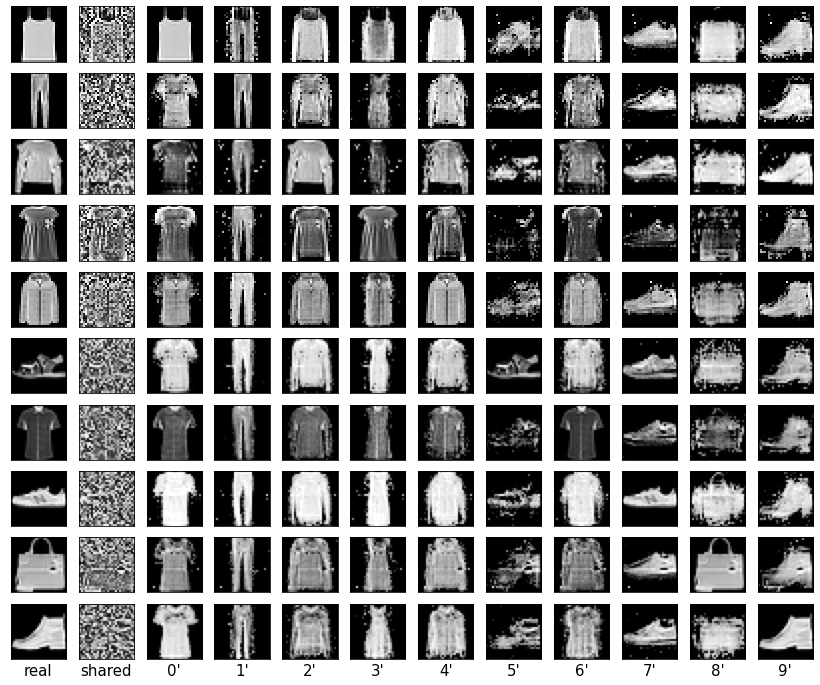}
        \caption{Multi-distribution ($\nclass=10$) results for FashionMNIST with DD.}
        \label{fig:fmnist-10-samples-dd}
\end{figure}

\clearpage

\begin{figure*}[h!]
     \centering
     \begin{subfigure}{0.85\textwidth}
     \centering
         \includegraphics[width=0.85\textwidth]{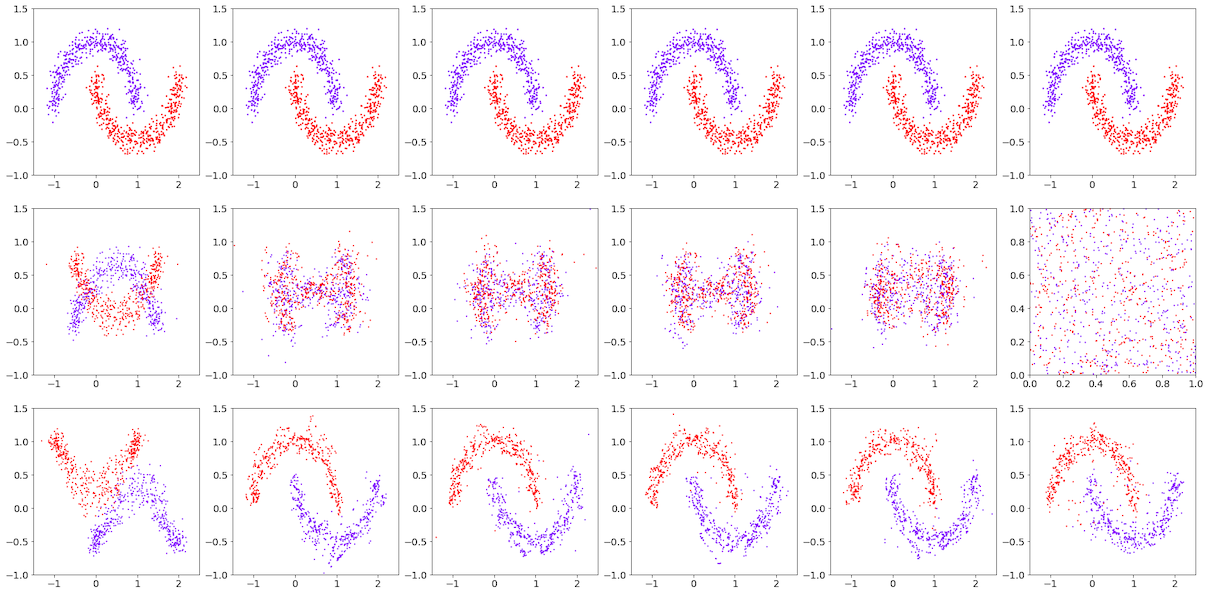}
\end{subfigure}
     \begin{subfigure}{0.85\textwidth}
         \centering
         \includegraphics[width=0.85\textwidth]{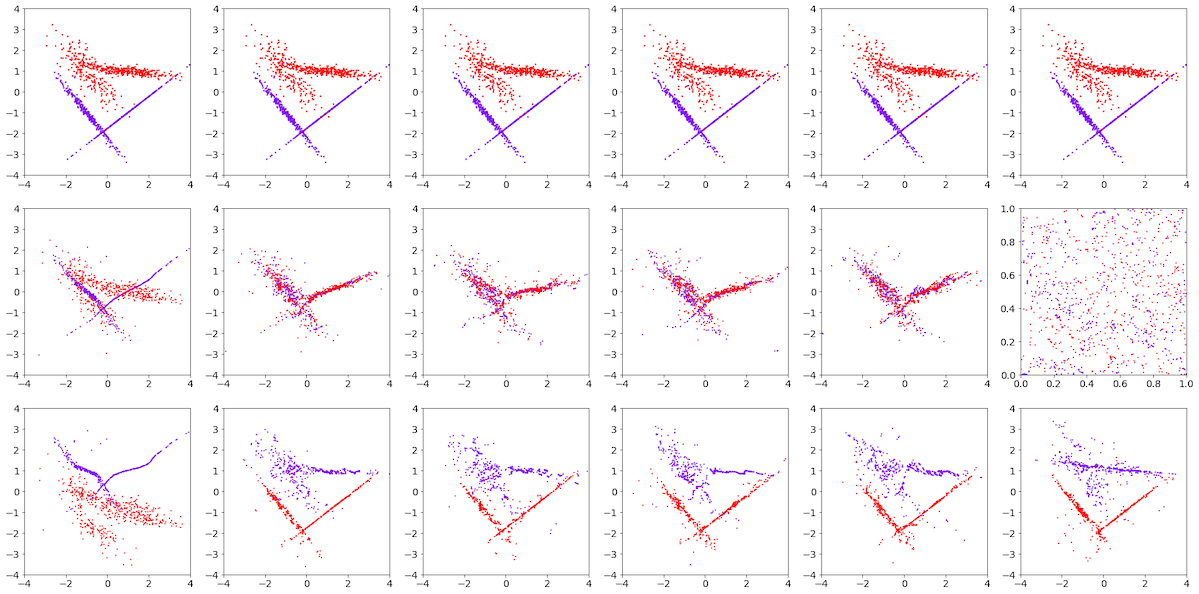}
\end{subfigure}
     \begin{subfigure}{0.85\textwidth}
         \centering
         \includegraphics[width=0.85\textwidth]{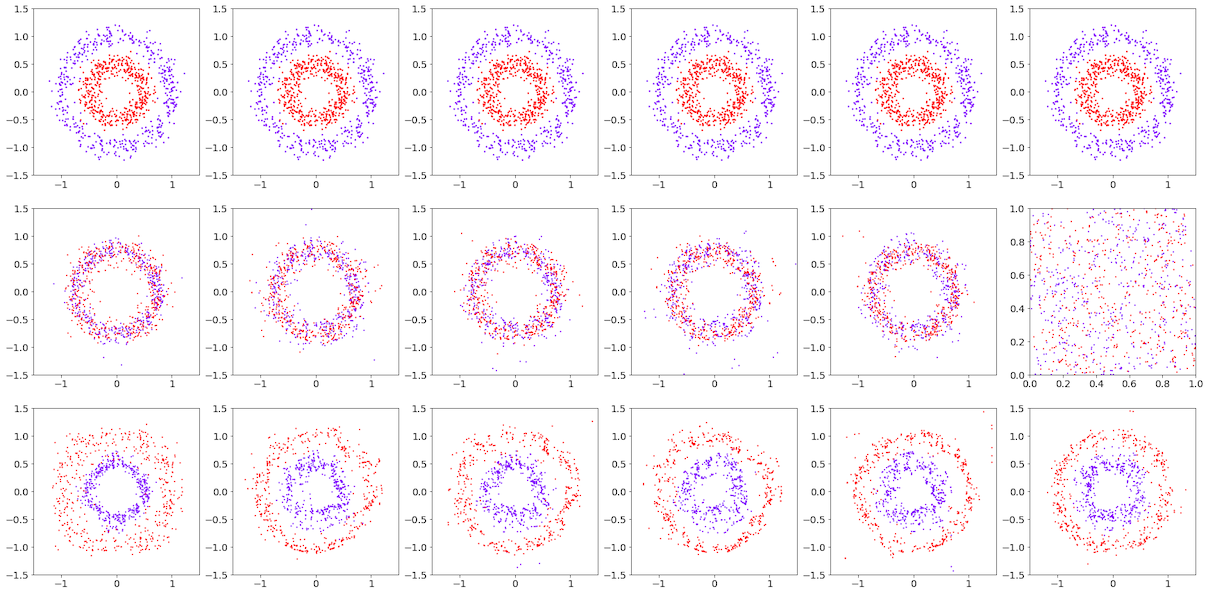}
\end{subfigure}
        \caption{Expanded figure of 2D Data ($\nclass=2$). In each sub figure, the first row represents the original data. The second row represents the latent distribution. The third row represents the flipped distribution. The columns from left to right represent the model: NB, INB, NB-INB, Rand-NB, NB-Rand-NB, DD.}
        \label{fig:2d_k2_full}
\end{figure*}
\clearpage

\begin{figure*}[!t]
     \centering
     
     \begin{subfigure}{0.15\textwidth}
\makebox[\textwidth][c]{\includegraphics[width=1\textwidth]{figures/multi_toy/random_ori.png}}
\label{fig:toy_gb}
     \end{subfigure}
     \begin{subfigure}{0.85\textwidth}
         \centering
         \includegraphics[width=0.85\textwidth]{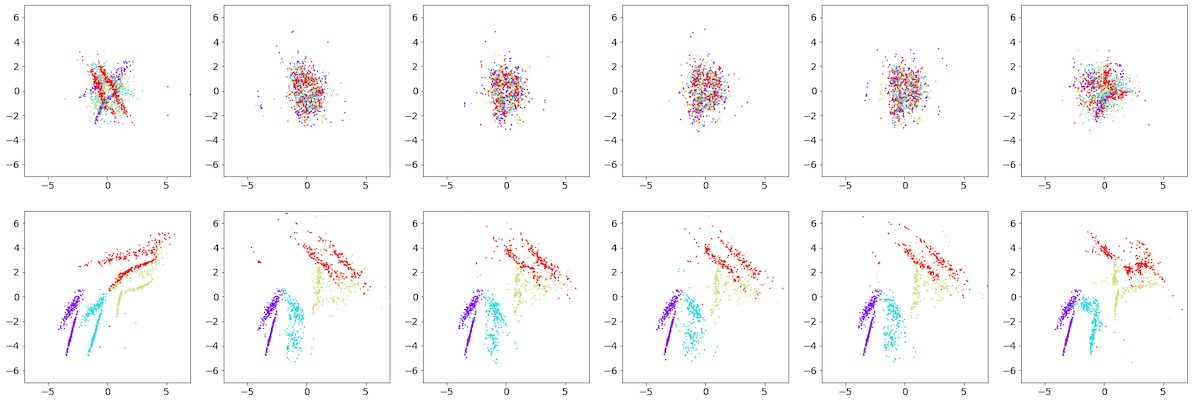}
\end{subfigure}
     \begin{subfigure}{0.85\textwidth}
         \centering
         \includegraphics[width=0.85\textwidth]{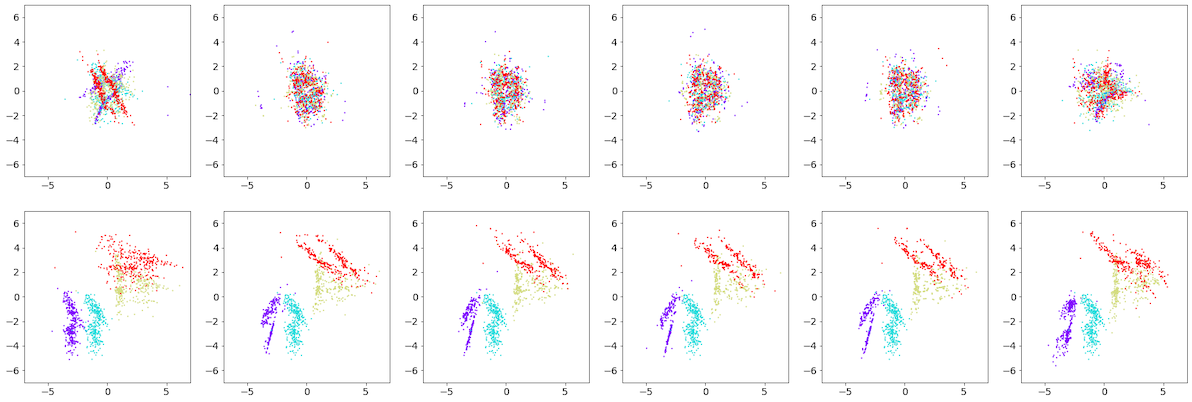}
\end{subfigure}
          \begin{subfigure}{0.85\textwidth}
         \centering
         \includegraphics[width=0.85\textwidth]{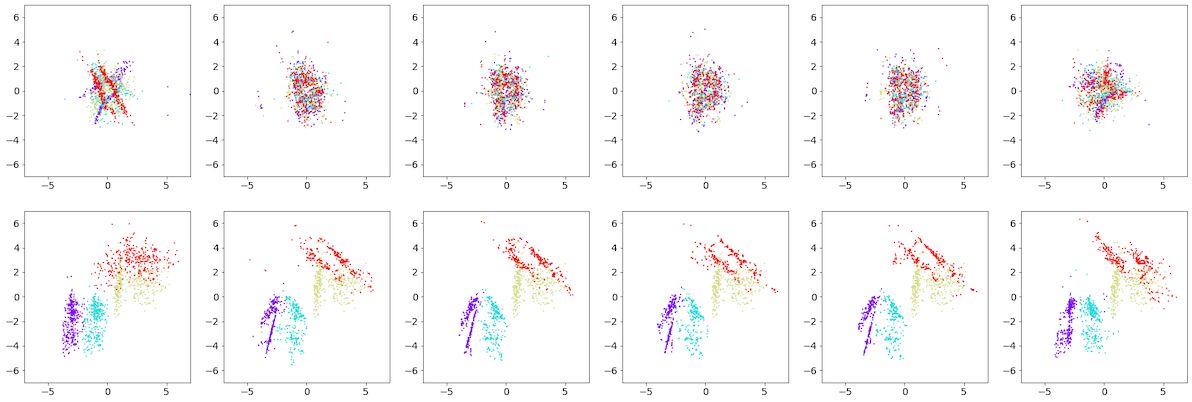}
\end{subfigure}
          \begin{subfigure}{0.85\textwidth}
         \centering
         \includegraphics[width=0.85\textwidth]{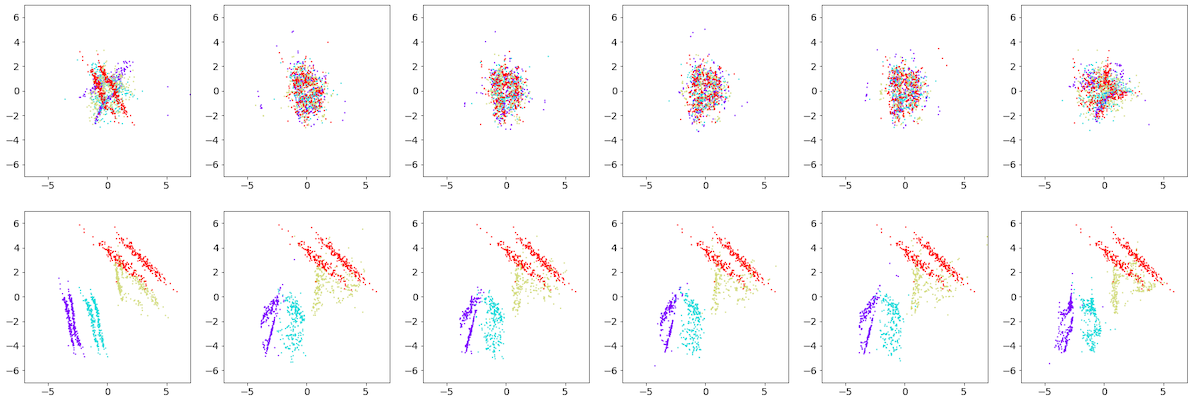}
\end{subfigure}
        \caption{Expanded figure of 2D Random Pattern ($\nclass=4$). 
        The columns from left to right represent the model: NB, INB, NB-INB, Rand-NB, NB-Rand-NB, DD.
        The top image is the original distribution.
        Each pair of rows represents the translation of samples from one class distribution to all other class distributions.
        We can translate every class distribution to every other class distribution since all functions are invertible.
        The pairs of rows are the results of translating from different source distributions, i.e., class 1 (purple), class 2 (turqoise), class 3 (yellow), and class 4 (red) distributions respectively.
        The top of each pair is the shared latent representation (the same across all rows) whereas the bottom row shows the generated data.
        Note that if the source and target distribution are the same, e.g., from class 1 to class 1, the output distribution will be exactly as in the original since our transformations are invertible.
        }
        \label{fig:random_k4_full}
\end{figure*}
\clearpage

\begin{figure*}[!t]
     \centering
     \begin{subfigure}{0.2\textwidth}
\includegraphics[width=1\textwidth]{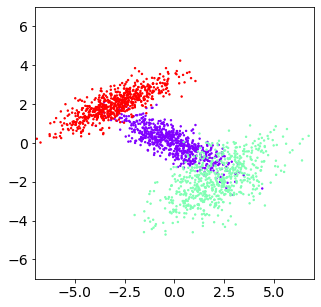}
\end{subfigure}
     \begin{subfigure}{\textwidth}
         \centering
         \includegraphics[width=0.95\textwidth]{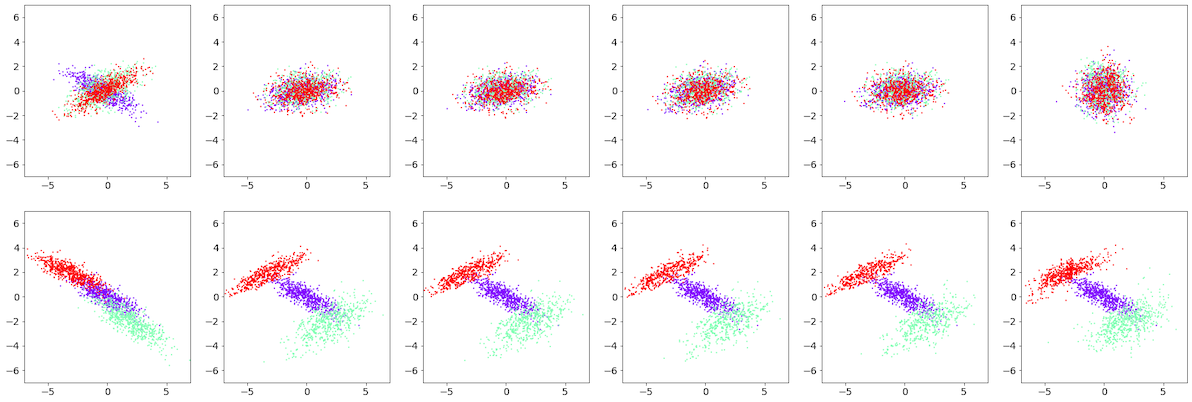}
\end{subfigure}
     \begin{subfigure}{\textwidth}
         \centering
         \includegraphics[width=0.95\textwidth]{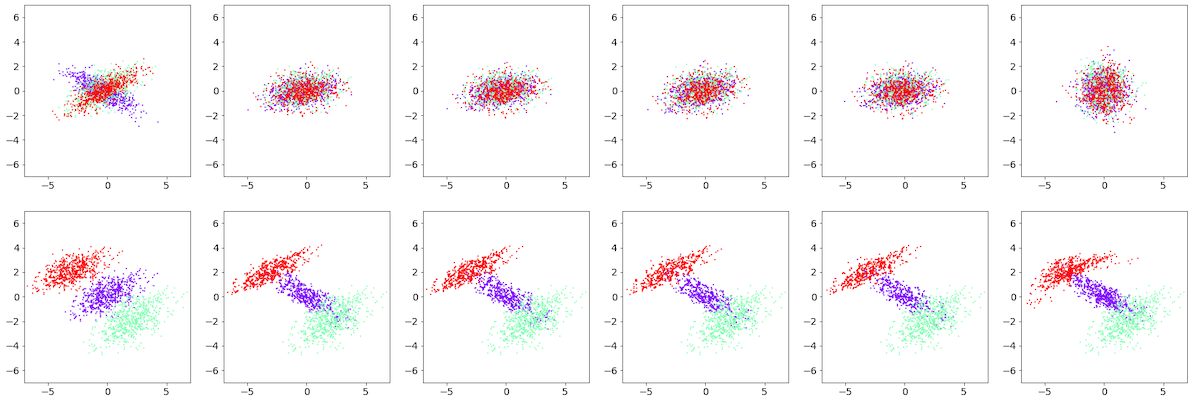}
\end{subfigure}
          \begin{subfigure}{\textwidth}
         \centering
         \includegraphics[width=0.95\textwidth]{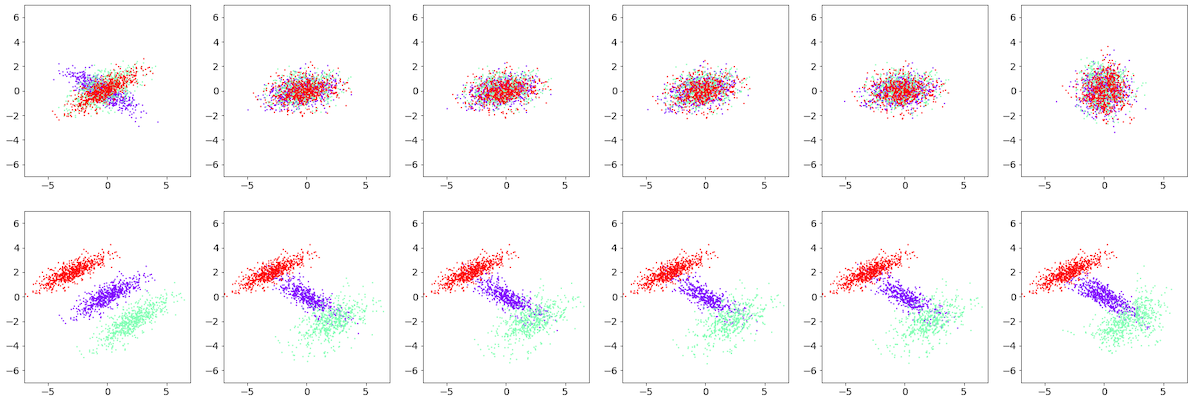}
\end{subfigure}
        \caption{Expanded figure of 2D Gaussian ($\nclass=3$).
        The columns from left to right represent the model: NB, INB, NB-INB, Rand-NB, NB-Rand-NB, DD.
        The top image is the original distribution. 
        See caption of \autoref{fig:random_k4_full} for explanation of each pair of rows.}
        \label{fig:gaussian_k3_full}
\end{figure*}

\end{document}